\documentclass{article} 

\usepackage[T1]{fontenc}
\usepackage{iclr2027_conference,times}

\usepackage{amsmath,amsfonts,bm}

\def\eqref#1{equation~\ref{#1}}
\def\Eqref#1{Equation~\ref{#1}}

\def\1{\bm{1}}

\DeclareMathAlphabet{\mathsfit}{\encodingdefault}{\sfdefault}{m}{sl}
\SetMathAlphabet{\mathsfit}{bold}{\encodingdefault}{\sfdefault}{bx}{n}

\usepackage{hyperref}
\usepackage{url}

\usepackage{booktabs}   
\usepackage{multirow}   
\usepackage{amsmath}    
\usepackage{amssymb}
\usepackage{amsthm}
\usepackage{mathtools}
\usepackage{bm}         
\usepackage{float}
\usepackage{placeins}
\usepackage{wrapfig} 
\usepackage{graphicx}
\usepackage{xcolor}

\usepackage{needspace}
\usepackage{caption}
\usepackage{threeparttable}
\usepackage{tabularx} 
\usepackage{ragged2e}

\definecolor{paperlink}{RGB}{0,78,140}
\hypersetup{
  colorlinks=true,
  linkcolor=paperlink,
  citecolor=paperlink,
  urlcolor=paperlink
}

\newtheorem{theorem}{Theorem}
\newtheorem{proposition}{Proposition}

\title{Transolver-$\sigma$: Joint Spectral-Physical\\ Subspace Modeling for Neural PDE Solving}

\author{{\fontsize{8.7}{10.5}\selectfont\bfseries
  Haonan Shangguan$^{1,*}$, Hang Zhou$^{1,*}$, Haixu Wu$^{2}$,
  Yuezhou Ma$^{1}$, Jianmin Wang$^{1}$, Mingsheng Long$^{1,\dagger}$}\\
  \normalfont $^{1}$School of Software, BNRist, Tsinghua University\qquad $^{2}$MIT CSAIL\\
  \texttt{\{sghn25,zhou-h23\}@mails.tsinghua.edu.cn}\\
  \texttt{\{jimwang,mingsheng\}@tsinghua.edu.cn}\\
  $^{*}$Equal contribution.\quad $^{\dagger}$Corresponding author.
}

\newcommand{\meanstd}[2]{%
  \begin{tabular}[c]{@{}c@{}}
    $\mathbf{#1}$\\[-1.5pt]
    {\scriptsize$\pm #2$}
  \end{tabular}%
}

\iclrfinalcopy
\begin{document}

\maketitle
\lhead{}

\begin{abstract}
Neural solvers offer efficient surrogates for numerical simulation of partial differential equations (PDEs). For time-dependent problems, strong one-step accuracy does not necessarily translate into reliable autoregressive rollout. We observe that a solver based only on physical-state modeling can achieve lower one-step error, whereas its spectral-only counterpart can become more accurate at later rollout steps. Motivated by this observation, we present Transolver-$\sigma$, a neural PDE solver based on joint spectral--physical subspace modeling. Within each block, adaptive physical-state interactions and spectral transformations are modeled in dedicated latent subspaces, whose responses are recomposed to enable information exchange between the two representations. Within the physical subspace, we introduce Slice-Residual Physics-Attention~(SRPA), which preserves an explicit slice-space identity path while retaining learnable cross-slice interaction. In parallel, an axis-factorized Fourier operator captures global spectral structure. Across five well-established PDE benchmarks spanning steady-state prediction and time-dependent dynamics, Transolver-$\sigma$ achieves state-of-the-art with a benchmark-averaged relative error reduction of 33.4\% over the strongest baseline for each metric, while consistently improving autoregressive rollout over single-operator counterparts. Transolver-$\sigma$ further delivers strong gains on coupled multiphysics systems and real-world fluid and combustion measurements from \mbox{RealPDEBench}, demonstrating its effectiveness beyond standard simulation benchmarks.
\end{abstract}
\setlength{\parskip}{4pt}

\vspace{-3pt}
\section{Introduction}
\vspace{-3pt}

Partial differential equations (PDEs) govern physical processes ranging from fluid dynamics to
material deformation. High-fidelity numerical solvers remain indispensable, but repeatedly solves
for parameter sweeps or optimization can be computationally prohibitive. Neural operators offer a
promising alternative by learning mappings between function spaces and providing fast, data-driven
surrogates~\citep{kovachki2023neuraloperator,li2021fno}. For time-dependent problems, however,
single-step precision alone does not guarantee long-term reliability. During autoregressive rollout,
predictions are recursively fed back as inputs, causing error accumulation over time. We refer to a
solver's capacity to maintain predictive fidelity under such recursive deployment as its
\textbf{rollout robustness}.

Autoregressive prediction can be fundamentally harder than single-step prediction. Rollout error
contains not only the error introduced at the current prediction step, but also errors already
present in the model-generated history. As a result, a model that performs well under teacher
forcing may behave very differently once it has to continue from its own predictions. Existing work
tackles this problem mainly through training strategies that expose models to their own outputs or
through iterative prediction refinement~\citep{brandstetter2022message,lippe2023pderefiner}. 
Here we study a complementary question from the perspective of model architecture:
\emph{how should neural operator architectures be structured to better reconcile local single-step precision with long-term rollout robustness?}

Figure~\ref{fig:accuracy-stability-discrepancy}(a) provides a simple diagnostic. We construct two
counterparts with comparable parameter budgets. \textbf{Physolver} follows Transolver and models the
field through input-adaptive physical states~\citep{wu2024Transolver}, whereas \textbf{Specsolver}
follows F-FNO and performs axis-factorized spectral
transformations~\citep{li2021fno,tran2023factorized}. Under teacher forcing, Physolver gives the
lower single-step error. Under autoregressive rollout, however, Specsolver becomes more accurate
at later prediction steps. In other words, the operator that is better for the next-step prediction
is not necessarily the one that remains better after repeated application. This contrast suggests
that physical-state and spectral modeling offer different advantages across the rollout horizon,
motivating their integration within a single architecture to exploit these complementary behaviors.

\begin{figure*}[t]
  \centering
  \includegraphics[width=\textwidth]{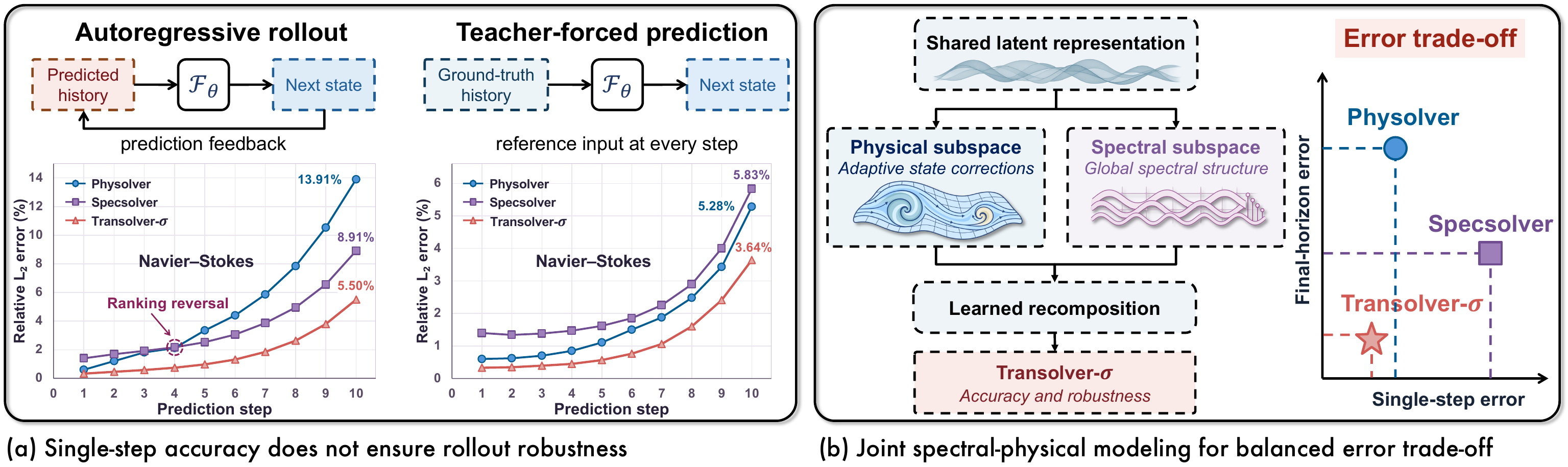}

  \caption{\textbf{Single-step accuracy does not guarantee rollout robustness.}
  (a) Error dynamics under teacher-forced prediction (ground-truth inputs) versus autoregressive rollout (recursive inputs) on 2D Navier--Stokes. Physolver achieves lower single-step error, whereas Specsolver achieves lower error at later rollout steps.
  (b) Conceptual overview of Transolver-$\sigma$, which leverages joint spectral-physical modeling via learned recomposition to achieve a balanced error trade-off.}
  \label{fig:accuracy-stability-discrepancy}
  \vspace{-18pt}
\end{figure*}

The remaining question is how to unify these two mechanisms effectively. The key design question
is how to preserve the distinct roles of physical-state and spectral updates while allowing them to
exchange information throughout the network. We therefore introduce \textbf{Transolver-$\sigma$},
which performs \emph{joint spectral--physical subspace modeling} within every block
(Figure~\ref{fig:accuracy-stability-discrepancy}(b)). Inspired by channel splitting in ShuffleNet
V2~\citep{ma2018shufflenetv2} and spatial--spectral channel allocation in Fast Fourier
Convolution~\citep{chi2020ffc}, each block first builds a shared latent representation and then
splits its channels into dedicated physical and spectral subspaces. The physical operator models
adaptive interactions among learned physical states, while the spectral operator applies
axis-factorized Fourier transformations. Their responses are then concatenated and recomposed
through a full-channel SwiGLU network~\citep{shazeer2020glu}. In this way, the two mechanisms remain distinct during their
own updates but can exchange information before the next block, through repeated recomposition.

We further revisit the physical-state update itself. Vanilla Physics-Attention updates the slice
states using attention alone, without a residual connection in slice space. LinearNO removes the
slice-to-slice attention through a linear-attention reformulation~\citep{hu2026linearno}, but consequently dispenses with a dedicated inter-slice interaction stage.
We therefore introduce
\emph{Slice-Residual Physics-Attention} (SRPA), which adds an explicit residual connection to the
slice-space update.
SRPA carries the incoming slice states to deslicing through an explicit
identity path and adds a learnably scaled attention correction, starting
from an exact slice-space identity at initialization.
We show when attention harms through replacement but helps through residual correction (Appendix~\ref{app:srpa-utility}).

We evaluate Transolver-$\sigma$ across steady-state and time-dependent PDEs, coupled multiphysics
simulations, and real-world measurements. It consistently outperforms single-operator counterparts
across four dynamical systems and achieves the lowest error on all eight metrics over five canonical
PDE benchmarks, with a benchmark-averaged relative reduction of 33.4\%. Further gains are observed
on coupled reacting-flow problems and RealPDEBench measurements~\citep{hu2026realpdebench}.

Our key contributions are summarized as follows:
\begin{itemize}
  \item \textbf{Joint Spectral--Physical Modeling}: We introduce Transolver-$\sigma$, which assigns
  physical-state and spectral operators to dedicated latent subspaces and recomposes their responses
  within each block, enabling the two representations to interact throughout the network.

  \item \textbf{Slice-Residual Physics-Attention}: We introduce SRPA, a residual reformulation of
  Physics-Attention that preserves an explicit slice-space identity path while retaining learnable
  cross-slice interaction. This design improves over both vanilla Physics-Attention and removing
  slice-to-slice attention in our controlled comparisons.

  \item \textbf{Broad Empirical Validation}: Transolver-$\sigma$ consistently achieves
  state-of-the-art performance across canonical PDE benchmarks, while improving autoregressive rollout
  over four dynamical systems and demonstrating strong generalization to coupled multiphysics
  simulations and real-world experimental measurements.
\end{itemize}

\vspace{-3pt}
\section{Related Work}
\vspace{-3pt}

\vspace{-3pt}
\subsection{Neural PDE Solvers}
\vspace{-3pt}
Neural PDE solvers based on operator learning approximate mappings between
function spaces to represent families of PDE solutions~\citep{kovachki2023neuraloperator}.
Spectral approaches parameterize these mappings through basis expansions,
with FNO implementing kernel integrals in the Fourier domain~\citep{li2021fno}.
F-FNO combines axis-factorized spectral layers with a residual architecture and
training refinements to support efficient, deeper operator
networks~\citep{tran2023factorized}. LSM applies spectral methods in learned
latent spaces, while GINO couples geometry-aware graph operators with Fourier
processing on regular latent grids~\citep{wu2023lsm,li2023gino}.

Attention-based operators model global interactions through input-dependent
weights. Galerkin Transformer connects efficient attention with operator
projections~\citep{cao2021galerkin}, GNOT extends attention to heterogeneous
inputs and irregular geometries~\citep{hao2023gnot}, and FactFormer factorizes
multidimensional interactions along spatial axes for improved computational
efficiency~\citep{li2023factformer}. 
Unisolver improves
cross-PDE generalization by conditioning a Transformer on explicit PDE
components~\citep{unisolver}. 
Beyond efficiency, these methods also
differ in the representation space over which attention is performed.
Transolver aggregates mesh features into learned physical states and applies
attention to capture global physical interactions
~\citep{wu2024Transolver}. Transolver++ improves state distinctiveness and
parallelism, while Transolver-3 scales training and inference to complex
industrial-scale geometries~\citep{luo2025transolverpp,zhou2026transolver3}.
LinearNO reformulates Physics-Attention as linear attention and removes the
explicit slice-to-slice attention stage~\citep{hu2026linearno}. In contrast,
SRPA retains cross-slice interaction while introducing an explicit residual
path, making attention a residual update in slice space rather than a direct replacement.

Several recent neural operators combine spectral processing with attention
to exploit both physical and frequency-domain representations
~\citep{rahman2024codano,yue2025holistic}. CoDA-NO uses Fourier neural
operators to construct function-valued attention for multiphysics
learning~\citep{rahman2024codano}, while HPM learns a unified
spectral--physical space that directly couples physical states with spectral
basis functions~\citep{yue2025holistic}. Transolver-$\sigma$ takes a different
approach: it assigns physical-state attention and spectral processing to
dedicated latent subspaces, allowing each to model the input independently
before their features are recombined within each block. This repeated
recomposition enables information from both representations to propagate
through the network.

\vspace{-3pt}
\subsection{Mitigating Error Accumulation in PDE Rollouts}
\vspace{-3pt}

Autoregressive neural PDE solvers reuse their predictions as future inputs,
allowing errors to propagate over the rollout horizon. Existing methods address
this problem through training, prediction refinement, or hybrid numerical
integration. MP-PDE uses pushforward training to expose the model to
self-generated states and temporal bundling to reduce feedback steps
~\citep{brandstetter2022message}. PDE-Refiner instead applies
diffusion-inspired iterative refinement to improve long-horizon prediction
~\citep{lippe2023pderefiner}.
DRIFT-Net couples a low-frequency spectral branch with an image-space branch
to preserve global structure and reduce drift during closed-loop rollout
~\citep{li2026drift}. Our work instead organizes physical-state and spectral operators into dedicated latent subspaces and repeatedly recomposes their responses across blocks to improve autoregressive prediction.

\FloatBarrier

\begingroup
\setlength{\textfloatsep}{4pt plus 1pt minus 1pt}
\setlength{\intextsep}{4pt plus 1pt minus 1pt}
\setlength{\abovecaptionskip}{2pt}
\setlength{\belowcaptionskip}{0pt}
\setlength{\abovedisplayskip}{3pt plus 1pt minus 1pt}
\setlength{\belowdisplayskip}{3pt plus 1pt minus 1pt}
\setlength{\abovedisplayshortskip}{0pt}
\setlength{\belowdisplayshortskip}{3pt plus 1pt minus 1pt}
\setlength{\columnsep}{12pt}
\vspace{-3pt}
\section{Transolver-$\sigma$}
\vspace{-3pt}

To address the discrepancy between one-step prediction and autoregressive
rollout performance, we present \textbf{Transolver-$\sigma$}, a novel network architecture built on joint spectral--physical subspace modeling. Within each block,
physical-state interactions and spectral transformations are modeled
in dedicated latent subspaces, whose features are recomposed to
enable effective information exchange.

\begin{figure}[t]
  \centering
  \setlength{\parskip}{0pt}
  \includegraphics[width=\textwidth]{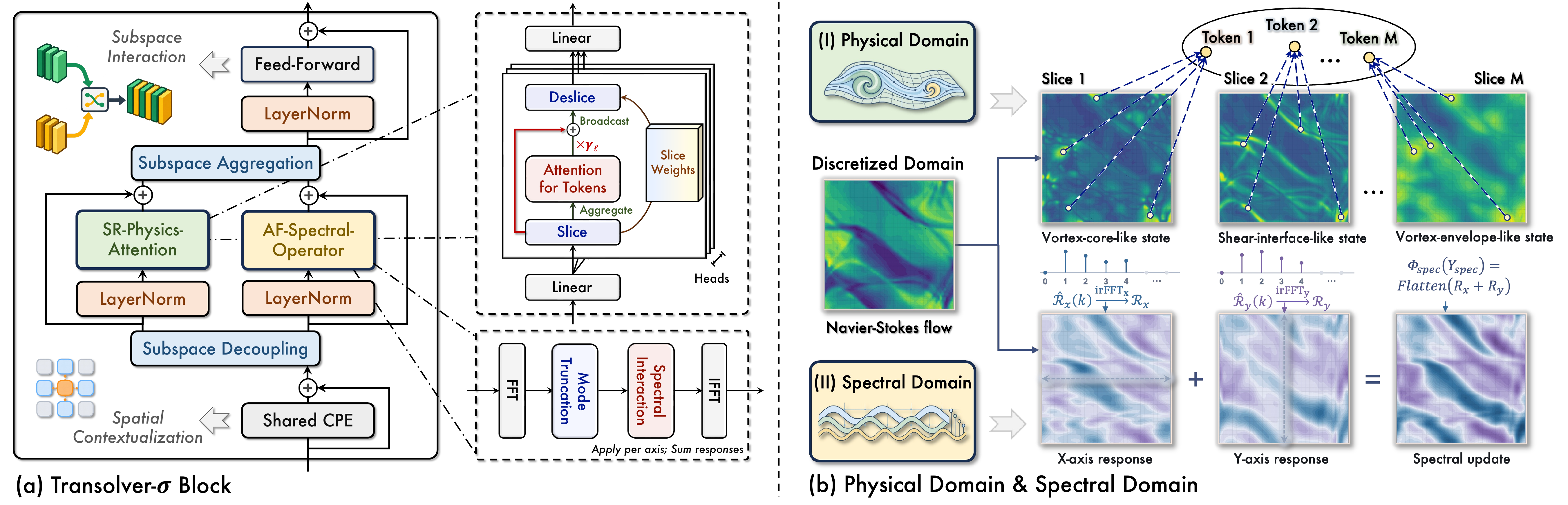}
  \vspace{-7pt}
  \caption{\textbf{Transolver-$\sigma$ architecture and spectral--physical modeling.}
  \textbf{(a)} Transolver-$\sigma$ performs specialized physical and spectral modeling in decoupled latent
  subspaces, followed by learned recomposition.
  \textbf{(b)} The physical branch captures intrinsic state interactions through slice-based modeling,
  while the spectral branch captures global spatial variations along individual axes.}
  \label{fig:transolver-sigma-architecture}
  \vspace{8pt}
\end{figure}

\paragraph{Problem Setup.}
\label{sec:method-setup}
We consider PDE prediction on a spatial domain
$\Omega \subset \mathbb{R}^{d}$ sampled at $N$ grid points
$\mathbf{g}\in\mathbb{R}^{N\times d}$. Given input fields
$\mathbf{X}\in\mathbb{R}^{N\times C_{\mathrm{in}}}$, we learn a neural
operator $\mathcal{F}_{\theta}$ to predict
\[
\widehat{\mathbf{U}}
=
\mathcal{F}_{\theta}(\mathbf{X};\mathbf{g})
\in\mathbb{R}^{N\times C_{\mathrm{out}}}.
\]
For steady-state problems, the model directly predicts the corresponding
solution fields. For time-dependent problems, $\mathbf{X}$ contains
$T_{\mathrm{in}}$ historical frames and the model predicts the next $T_{\mathrm{out}}$
frames, with $C_{\mathrm{out}}=T_{\mathrm{out}}q$ for $q$ output field channels per frame.
During autoregressive rollout, each predicted block is appended to the history
and the most recent $T_{\mathrm{in}}$ frames are used for the next model call.
See Appendix~\ref{app:experiment-protocols} for task-specific settings and rollout configurations.

\vspace{-3pt}
\subsection{Joint Spectral--Physical Subspace Modeling}
\vspace{3pt}
\label{sec:method-decomposition}

Physical-state and spectral operators represent spatial interactions through
different intermediate representations. Physics-Attention organizes spatial
features into input-adaptive physical states and models interactions among them~\citep{wu2024Transolver}. Fourier operators instead model global structure through frequency modes in a prescribed spectral basis~\citep{li2021fno}. These two representations therefore
provide distinct inductive biases: physical-state modeling adapts its interaction
structure specific to the current field, whereas spectral modeling employs a fixed global basis.

\vspace{8pt}
Rather than forcing these two forms of representation into a shared operator,
Transolver-$\sigma$ assigns physical-state and spectral transformations to
dedicated channel subspaces. Each subspace can specialize under its own
operator, while their responses are recomposed within every block to exchange
information across successive layers. As illustrated in Figure~\ref{fig:transolver-sigma-architecture}(a), Transolver-$\sigma$ first lifts the
input fields $\mathbf{X}$ and coordinates $\mathbf{g}$ into latent features $\mathbf{H}^{(0)} \in
\mathbb{R}^{N \times C}$, followed by $L$ dual-subspace blocks. In each block,
conditional positional encoding ($\operatorname{CPE}$)~\citep{chu2023conditional} injects spatial context before the features
are evenly split into spectral and physical channel subspaces:

\begingroup
\setlength{\parskip}{0pt}
\setlength{\abovedisplayskip}{1pt plus 1pt minus 1pt}
\setlength{\belowdisplayskip}{1pt plus 1pt minus 1pt}
\setlength{\abovedisplayshortskip}{0pt}
\setlength{\belowdisplayshortskip}{1pt plus 1pt minus 1pt}
\begin{subequations}
  \label{eq:transolver-sigma-block}
  \begin{align}
    &\begin{aligned}
      \overline{\mathbf{H}}^{(\ell)}
      &= \mathbf{H}^{(\ell)} + \operatorname{CPE}_{\ell}(\mathbf{H}^{(\ell)}),
      \qquad
      [\mathbf{H}^{(\ell)}_{\mathrm{spec}},\mathbf{H}^{(\ell)}_{\mathrm{phy}}]
      = \operatorname{Split}_{C/2}(\overline{\mathbf{H}}^{(\ell)}),
    \end{aligned}
    \label{eq:block-context}
    \\
    &\widetilde{\mathbf{H}}^{(\ell)}_{\ast}
    = \mathbf{H}^{(\ell)}_{\ast}
    + \Phi^{(\ell)}_{\ast}\!\left(
    \operatorname{LN}^{(\ell)}_{\ast}(\mathbf{H}^{(\ell)}_{\ast})
    \right),
    \qquad \ast \in \{\mathrm{spec},\mathrm{phy}\},
    \label{eq:block-private}
    \\
    &\begin{aligned}
      \mathbf{H}^{(\ell)}_{\mathrm{cat}}
      &= \operatorname{Concat}\!\left(
      \widetilde{\mathbf{H}}^{(\ell)}_{\mathrm{spec}},
      \widetilde{\mathbf{H}}^{(\ell)}_{\mathrm{phy}}
      \right),
      \qquad
      \mathbf{H}^{(\ell+1)}
      = \mathbf{H}^{(\ell)}_{\mathrm{cat}}
      + \operatorname{SwiGLU}_{\ell}\!\left(
      \operatorname{LN}^{(\ell)}(\mathbf{H}^{(\ell)}_{\mathrm{cat}})
      \right).
    \end{aligned}
    \label{eq:block-recompose}
  \end{align}
\end{subequations}
\endgroup

\vspace{8pt}
Here, the two operators capture complementary spatial dynamics: $\Phi_{\mathrm{phy}}$ denotes Slice-Residual
Physics-Attention, which models adaptive interactions among intrinsic physical states, while
$\Phi_{\mathrm{spec}}$ denotes the axis-factorized Fourier operator for global spectral modeling. The two
subspaces are updated independently and then recomposed within each block through a full-channel
SwiGLU network, enabling information exchange between physical and spectral representations. After
$L$ blocks, layer normalization and a linear readout produce the target fields
$\widehat{\mathbf{U}} \in \mathbb{R}^{N \times C_{\mathrm{out}}}$.

\vspace{1pt}
\subsection{Slice-Residual Physics-Attention}
\vspace{3pt}
\label{sec:method-srpa}

To overcome the quadratic complexity of point-to-point self-attention in discretized domains,
Transolver introduced Physics-Attention, which groups spatial points into $M$ intrinsic physical
tokens via soft slicing and models interactions strictly in the token
space~\citep{wu2024Transolver}. However, a key limitation of vanilla Physics-Attention lies in its
state update rule: at each layer, the aggregated slice tokens $\mathbf{Z}$ are completely
overwritten by the attention-mixed values $\Delta \mathbf{Z}$. Across deeper layers, this repeated replacement increasingly mixes slice representations, weakening distinctions among physical states and reducing token diversity (Figure~\ref{fig:efficiency-representations}c) (Appendix~\ref{app:srpa-theory}). A recent alternative, LinearNO, reformulates Physics-Attention as linear attention and removes the separate slice-to-slice attention stage while preserving input-adaptive global aggregation~\citep{hu2026linearno}.

\begingroup
\setlength{\parskip}{5pt}
\setlength{\abovedisplayskip}{2pt}
\setlength{\belowdisplayskip}{2pt}
\setlength{\abovedisplayshortskip}{0pt}
\setlength{\belowdisplayshortskip}{2pt}

\Needspace{22\baselineskip}
\begingroup
\newcommand{\SRPAFigureWidth}{0.49\textwidth} 
\newcommand{\SRPAFigureGap}{0.025\textwidth} 
\setlength{\columnsep}{\SRPAFigureGap}
\setlength{\intextsep}{2pt}
\setlength{\parskip}{5pt}
\begin{wrapfigure}{r}{\SRPAFigureWidth}
  \centering
  \setlength{\parskip}{0pt}
  \captionsetup{width=\linewidth,skip=4pt,justification=justified,singlelinecheck=false}
  \includegraphics[width=\linewidth]{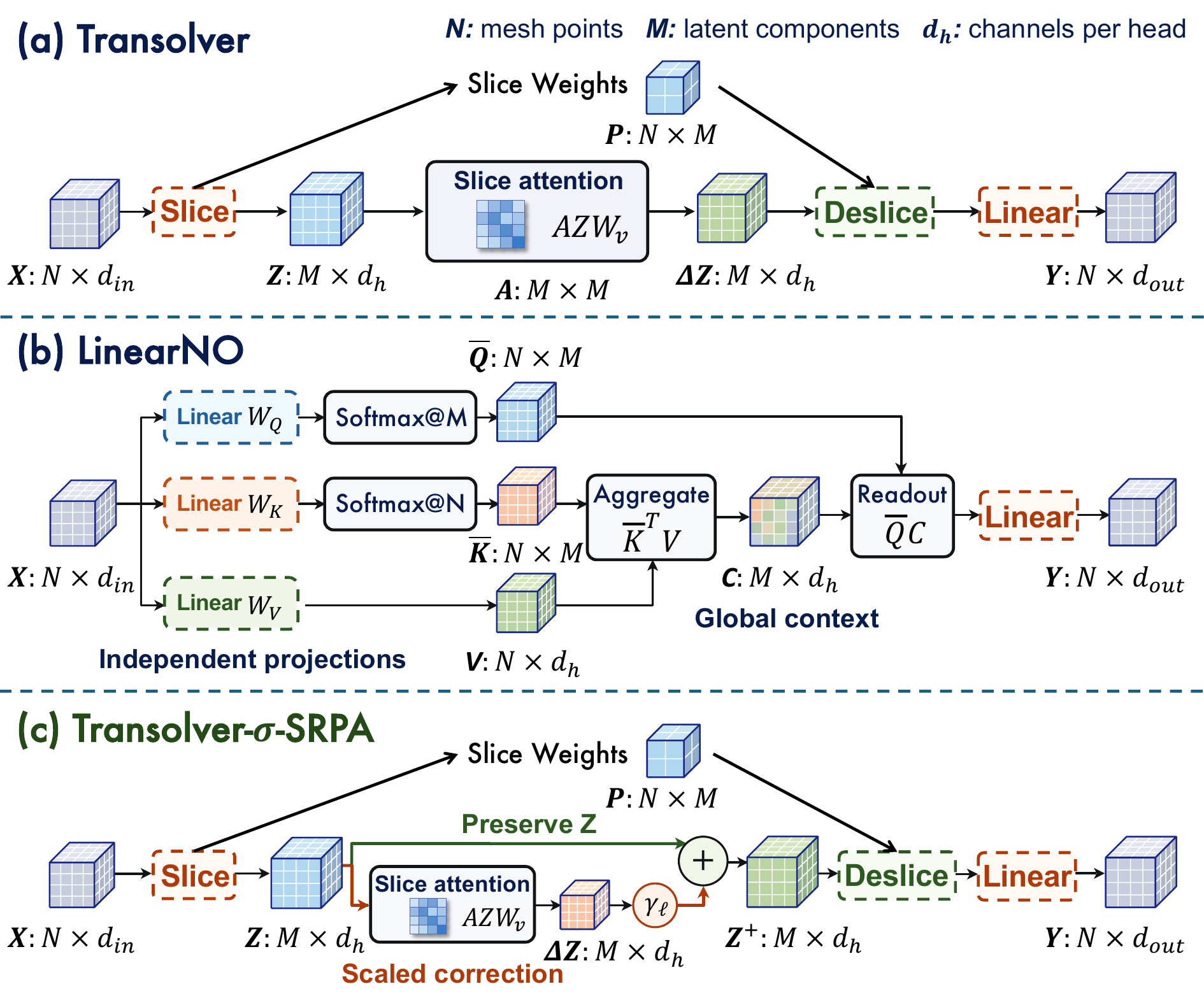}
  \caption{\textbf{Comparison of slice-based operator updates.}
  \textbf{(a)} Transolver updates physical states with solely attention.
  \textbf{(b)} LinearNO uses linear global-context aggregation.
  \textbf{(c)} SRPA retains slice states with scaled residual attention.}
  \label{fig:srpa-comparison}
\end{wrapfigure}
\noindent
We argue that inter-slice interactions can help capture non-local physical correlations, but should refine the existing slice states rather than replace them. To
this end, we propose \textbf{Slice-Residual Physics-Attention (SRPA)}. By introducing an explicit
residual path in the physical token space, SRPA retains the primary slice identities $\mathbf{Z}$
and incorporates the attention response $\Delta \mathbf{Z}$ as a learnably scaled residual
correction ($\mathbf{Z}^+ = \mathbf{Z} + \gamma_\ell \Delta \mathbf{Z}$). As illustrated in
Figure~\ref{fig:srpa-comparison}, SRPA preserves slice identities while enabling effective
inter-slice refinement.

\noindent
Formally, given the normalized physical subspace features \(\mathbf{Y}_{\mathrm{phy}}=\operatorname{LN}_{\mathrm{phy}}(\mathbf{H}_{\mathrm{phy}})\in\mathbb{R}^{N\times C_{\mathrm{phy}}}\), we first apply separate projections to obtain routing and feature representations, then partition outputs into \(n_h\) heads with head dimension \(d_h=C_{\mathrm{phy}}/n_h\). For each head, we denote these representations by \(\mathbf{R},\mathbf{F}\in\mathbb{R}^{N\times d_h}\). The point-to-slice routing matrix $\mathbf{P} \in \mathbb{R}^{N
\times M}$ softly assigns $N$ spatial points to $M$ physical slice tokens ($M \ll N$), forming slice states $\mathbf{Z} \in \mathbb{R}^{M \times d_h}$ as defined in
~\Eqref{eq:physical-slicing}, where $\mathbf{D}$ normalizes each physical
token by its total spatial mass, and temperature $\tau$ controls assignment sharpness across the
domain. \(\mathbf W_s\in\mathbb R^{d_h\times M}\) and \(\mathbf b_s\in\mathbb R^M\) are the learnable routing weight and bias. \(\mathbf1_N,\mathbf1_M\) are all-ones vectors, and \(\varepsilon>0\) stabilizes slice normalization.

Next, rather than replacing $\mathbf{Z}$ directly, SRPA performs multi-head self-attention across
the $M$ physical tokens to compute an interaction response $\Delta\mathbf{Z}$, which is added back to $\mathbf{Z}$ as a scaled residual correction (~\Eqref{eq:srpa-update}), where $\mathbf{W}_q, \mathbf{W}_k,
\mathbf{W}_v \in \mathbb{R}^{d_h \times d_h}$ are projection matrices. \(\mathbf A\) is the slice-attention matrix, and \(\operatorname{softmax}_M\) normalizes each row over \(M\) entries. The scalar
$\gamma_\ell$ is initialized to zero following ReZero~\citep{bachlechner2021rezero}. This ensures an
exact identity mapping in slice space at initialization, with subsequent inter-slice corrections controlled by the learned scale. 
\par
\endgroup

\vspace{-2pt}
\noindent
\begin{minipage}[t]{0.48\linewidth}
  \setlength{\abovedisplayskip}{1pt}
  \setlength{\belowdisplayskip}{1pt}
  \setlength{\abovedisplayshortskip}{0pt}
  \setlength{\belowdisplayshortskip}{1pt}
  \begin{equation}
    \begin{aligned}
      &\mathbf{P}
      =\operatorname{softmax}_{M}\!\left(
      \frac{\mathbf{R}\mathbf{W}_s
      +\mathbf{1}_N\mathbf{b}_s^{\top}}{\tau}
      \right),\\
      &\mathbf{D}
      =\operatorname{diag}\!\left(
      \mathbf{P}^{\top}\mathbf{1}_N
      +\varepsilon\mathbf{1}_M
      \right),\\
      &\mathbf{Z}
      =\mathbf{D}^{-1}\mathbf{P}^{\top}\mathbf{F}.
    \end{aligned}
    \label{eq:physical-slicing}
  \end{equation}
\end{minipage}
\hfill
\begin{minipage}[t]{0.48\linewidth}
  \setlength{\abovedisplayskip}{1pt}
  \setlength{\belowdisplayskip}{1pt}
  \setlength{\abovedisplayshortskip}{0pt}
  \setlength{\belowdisplayshortskip}{1pt}
  \begin{equation}
    \begin{aligned}
      &\mathbf{A}
      =\operatorname{softmax}_{M}\!\left(
      \frac{
      (\mathbf{Z}\mathbf{W}_q)
      (\mathbf{Z}\mathbf{W}_k)^{\top}}
      {\sqrt{d_h}}
      \right),\\
      &\Delta\mathbf{Z}
      =\mathbf{A}\mathbf{Z}\mathbf{W}_v,\\
      &\mathbf{Z}^{+}
      =\mathbf{Z}+\gamma_{\ell}\Delta\mathbf{Z}.
    \end{aligned}
    \label{eq:srpa-update}
  \end{equation}
\end{minipage}

Finally, the refined slice tokens $\mathbf{Z}^+$ are mapped back to spatial domain points via
deslicing using the routing weights $\mathbf{P}$, concatenated across heads, and projected to yield
the physical operator output:

\vspace{-4pt}
\begin{equation}
  \begin{aligned}
    &\mathbf{O}^{(r)}
    =\mathbf{P}^{(r)}\mathbf{Z}^{+,(r)},
    \qquad r=1,\ldots,n_h,\\
    &\Phi_{\mathrm{phy}}(\mathbf{Y}_{\mathrm{phy}})
    =\operatorname{Proj}_o\!\left(
    \operatorname{Concat}_{r=1}^{n_h}\mathbf{O}^{(r)}
    \right)
    \in\mathbb{R}^{N\times C_{\mathrm{phy}}}.
  \end{aligned}
  \label{eq:srpa-deslicing}
\end{equation}

Here, \(r\) indexes heads, \(\mathbf{O}^{(r)}\) is the desliced head output, and
\(\operatorname{Proj}_o\) is the output projection. Crucially, the inner residual path
carries the unmixed slice state $\mathbf{Z}$ into the deslicing stage
($\mathbf{P}\mathbf{Z}$), whereas the outer block residual in
Equation~\ref{eq:block-private} bypasses the physical operator
$\Phi_{\mathrm{phy}}$ through $\mathbf{H}_{\mathrm{phy}}$.
Appendix~\ref{app:srpa-theory} provides conditional pairwise and centered-variation
bounds for slice-space updates.
\par
\endgroup 

\vspace{-3pt}
\subsection{Axis-Factorized Spectral Operator}
\vspace{-3pt}
\label{sec:method-spectral}

While Slice-Residual Physics-Attention captures content-adaptive spatial interactions in the physical subspace, the spectral subspace plays a complementary role by modeling global spatial structure through retained Fourier modes.
Rather than using a multi-dimensional Fourier operator with parameter complexity $\mathcal{O}(C_{\mathrm{spec}}^2K^d)$, we adopt an axis-factorized spectral operator inspired by F-FNO~\citep{tran2023factorized}. \(K\) denotes the common per-axis mode budget. This factorization enables efficient spectral modeling in deep backbones with reduced parameter growth while providing constraints that help mitigate trajectory drift during autoregressive rollout over prediction horizons.

Specifically, for a $d$-dimensional field, 1D real Fast Fourier Transforms ($\operatorname{rFFT}$)
are applied independently along each spatial axis $a \in \{1, \ldots, d\}$. The lowest $K_a$ Fourier
modes are filtered via learnable complex channel-mixing matrices, transformed back via inverse
$\operatorname{rFFT}$, and summed across axes to form the updated representation in spectral space
(Figure~\ref{fig:transolver-sigma-architecture}(b)). The full formulation, mode truncation details,
and implementation conventions are provided in Appendix~\ref{app:spectral-operator}.

\endgroup 

\begingroup
\setlength{\textfloatsep}{4pt plus 1pt minus 1pt}
\setlength{\intextsep}{4pt plus 1pt minus 1pt}
\setlength{\abovecaptionskip}{2pt}
\setlength{\belowcaptionskip}{0pt}
\setlength{\abovedisplayskip}{3pt plus 1pt minus 1pt}
\setlength{\belowdisplayskip}{3pt plus 1pt minus 1pt}
\setlength{\abovedisplayshortskip}{0pt}
\setlength{\belowdisplayshortskip}{3pt plus 1pt minus 1pt}
\setlength{\columnsep}{12pt}

\vspace{-3pt}
\section{Experiments}
\vspace{-3pt}
\label{sec:experiments}

We conduct extensive experiments to evaluate Transolver-$\sigma$, extending from five canonical PDE
benchmarks to two coupled multiphysics problems and four real-world tasks. Covering
both steady-state and time-dependent systems, these evaluations assess field prediction accuracy and
autoregressive rollout robustness across numerical simulations and experimental
measurements.

\begin{table*}[!t]
  \centering
  \caption {\textbf{Results on Standard Benchmarks.} Relative $L_2$ errors (in \%) of pressure $p$,
  vorticity $\omega$, velocity $\mathbf{u}$, and smoke density $d$ across standard benchmarks, where
  ``avg.'' and ``final'' denote horizon-aggregated and final-frame errors. Our results are reported as
  mean $\pm$ std over three runs.}
  \label{tab:pde-main-results}
  \small
  \setlength{\tabcolsep}{2.8pt}
  \renewcommand{\arraystretch}{1.08}
  \begin{tabular*}{\linewidth}{@{\extracolsep{\fill}}l|cccccccc@{}}
    \toprule
    \multirow{2}{*}{\textsc{Models}} & \multicolumn{1}{c}{\textsc{Darcy}} & \multicolumn{1}{c}{\shortstack{\textsc{Navier--}\\\textsc{Stokes}}} & \multicolumn{2}{c}{\textsc{2D Kolmogorov}} & \multicolumn{2}{c}{\textsc{3D isotropic}} & \multicolumn{2}{c}{\textsc{3D smoke}} \\
    \cmidrule(lr){2-2}\cmidrule(lr){3-3}\cmidrule(lr){4-5}\cmidrule(lr){6-7}\cmidrule(lr){8-9}
    & $p$ & $\omega$ & $\omega_{\mathrm{avg.}}$ & $\omega_{\mathrm{final}}$ & $\mathbf{u}$ & $p$ & $\mathbf{u}$ & $d$ \\
    \midrule
    U-Net~\citeyearpar{ronneberger2015unet} & 0.80 & 19.82 & 32.81 & 47.24 & 35.76 & 48.49 & 44.10 & 14.95 \\
    ViT~\citeyearpar{dosovitskiy2021vit} & 0.47 & 4.64 & 17.19 & 27.19 & 31.66 & 44.54 & 38.41 & 12.86 \\
    \midrule
    FNO~\citeyearpar{li2021fno} & 1.08 & 15.56 & 29.78 & 45.67 & 33.82 & 46.34 & 42.55 & 13.44 \\
    F-FNO~\citeyearpar{tran2023factorized} & 0.77 & 23.22 & 24.53 & 38.61 & 23.03 & 32.64 & 37.13 & 12.36 \\
    RNO~\citeyearpar{lu2025radon} & 0.54 & 8.94 & 18.93 & 29.14 & 22.37 & 30.72 & 28.51 & 11.53 \\
    HPM~\citeyearpar{yue2025holistic} & 0.46 & 7.34 & 34.19 & 58.81 & 24.85 & 39.11 & 55.23 & 14.94 \\
    DRIFT-Net~\citeyearpar{li2026drift} & \underline{0.45} & 4.32 & 13.83 & 23.97 & 14.56 & 22.64 & 28.99 & 10.82 \\
    \midrule
    FactFormer~\citeyearpar{li2023factformer} & 0.47 & \underline{4.23} & 14.29 & 25.11 & 13.87 & 21.74 & \underline{23.37} & \underline{9.19} \\
    Transolver~\citeyearpar{wu2024Transolver} & 0.58 & 9.00 & 43.98 & 75.03 & 26.42 & 39.47 & 61.08 & 16.33 \\
    Transolver++~\citeyearpar{luo2025transolverpp} & 0.49 & 7.19 & 47.51 & 79.07 & 26.23 & 39.67 & 61.14 & 16.25 \\
    EddyFormer~\citeyearpar{du2025eddyformer} & 0.66 & 4.25 & \underline{12.69} & \underline{21.99} & \underline{12.70} & \underline{19.43} & 25.87 & 9.35 \\
    LinearNO~\citeyearpar{hu2026linearno} & 0.50 & 6.99 & 33.37 & 56.14 & 22.86 & 37.14 & 50.04 & 14.18 \\
    \midrule
    \textbf{Transolver-$\sigma$ (Ours)}
    & \meanstd{0.40}{0.01}
    & \meanstd{2.79}{0.05}
    & \meanstd{2.52}{0.05}
    & \meanstd{4.14}{0.07}
    & \meanstd{10.80}{0.16}
    & \meanstd{15.65}{0.23}
    & \meanstd{17.43}{0.25}
    & \meanstd{7.07}{0.09} \\
    \midrule
    \textsc{Relative promotion} & $11.1\%$ & $34.0\%$ & $80.1\%$ & $81.2\%$ & $15.0\%$ & $19.4\%$ & $25.4\%$ & $23.1\%$ \\
    \bottomrule
  \end{tabular*}
\end{table*}

\vspace{-3pt}
\subsection{Experimental Setup}
\vspace{-3pt}
\label{sec:experimental-setup}

\paragraph{Benchmarks}
Our evaluation spans a broad range of steady-state and time-dependent PDEs in both 2D and 3D, across diverse problem settings, including Darcy flow and Navier--Stokes from FNO~\citep{li2021fno}, and Kolmogorov flow, isotropic
turbulence, and smoke buoyancy from FactFormer~\citep{li2023factformer}. Beyond simulation-generated
benchmarks, we further evaluate on RealPDEBench~\citep{hu2026realpdebench}, which provides
real-world experimental measurements of fluid and combustion dynamics, offering a more realistic
test of neural PDE solvers beyond numerical simulation. We additionally include the coupled
reacting-flow benchmarks IgnitHIT and EvolveJet from REALM~\citep{mao2025benchmarking}. Full
benchmark details are provided in Appendix~\ref{app:experiment-protocols}.

\paragraph{Baselines}
We comprehensively compare Transolver-$\sigma$ with $12$ baselines, covering representative
transform-domain and hybrid neural operators, such as FNO~\citep{li2021fno},
DRIFT-Net~\citep{li2026drift}, etc.; Transformer-based PDE solvers, including
Transolver~\citep{wu2024Transolver}, LinearNO~\citep{hu2026linearno}, etc.; and general-purpose
architectures such as U-Net~\citep{ronneberger2015unet} and ViT~\citep{dosovitskiy2021vit}, etc. For
RealPDEBench and REALM, we additionally compare against the baselines reported in their original
studies~\citep{hu2026realpdebench,mao2025benchmarking}.

\paragraph{Implementations}
For fair comparison, we set the hidden dimension $C$ to $128$ or $256$ and the number of blocks \(L\) to a task-specific value in \(\{4,6,8\}\), keeping the model size comparable to
Transformer-based PDE solvers such as Transolver~\citep{wu2024Transolver} and
LinearNO~\citep{hu2026linearno}. All main experiments are conducted on a single NVIDIA A100 GPU. We
primarily report relative $L^2$ errors for field prediction and autoregressive rollout, while
following the original evaluation protocols of RealPDEBench and REALM for their task-specific
metrics. See Appendix~\ref{app:experiments} for comprehensive implementation details and metric
definitions, and training configurations.

\vspace{-3pt}
\subsection{Main Results}
\vspace{-3pt}

\paragraph{Standard Benchmarks}
As shown in Table~\ref{tab:pde-main-results}, Transolver-$\sigma$ achieves state-of-the-art performance
on five widely used PDE benchmarks spanning steady-state prediction and time-dependent dynamics in
both 2D and 3D. It obtains the lowest errors on all eight metrics, with a benchmark-averaged
relative error reduction of $33.4\%$ over the strongest per-metric baselines. 
The gains further extend to 3D dynamics, where Transolver-$\sigma$ reduces velocity and density errors on
smoke buoyancy by $25.4\%$ and $23.1\%$. Overall, these results demonstrate the consistent
effectiveness of joint spectral--physical modeling across diverse PDE systems and forecasting
settings.
Moreover, scaling experiments further show performance gains from more data, finer grids, and larger models (Appendix~\ref{app:model-scalability}).

\paragraph{RealPDEBench: Real-World Experimental Measurements}
To further assess performance beyond simulation-based benchmarks, we evaluate Transolver-$\sigma$ on
RealPDEBench~\citep{hu2026realpdebench}, the first scientific ML benchmark built on real-world
physical measurements paired with numerical simulations across diverse operating conditions. As shown in
Table~\ref{tab:realpdebench-results}, Transolver-$\sigma$ achieves superior performance across all four
tasks and metrics, reducing Foil RMSE by $53.0\%$ and Combustion fRMSE by $25.0\%$ over the
strongest respective baselines. Moreover, it surpasses U-Net, the strongest model reported in
RealPDEBench, with only $13.6\%$ of its parameters ($3.13$M vs.\ $23.08$M), achieving better
accuracy with substantially lower model complexity and computational overhead.

\begin{table*}[!t]
  \centering
  \caption{\textbf{Results on RealPDEBench.}
  RMSE, relative $L^2$, and Fourier-space RMSE (fRMSE) across four real-world measurement tasks.
  Params denotes the mean parameter count (M) across tasks.}

  \label{tab:realpdebench-results}
  \fontsize{8.0pt}{9pt}\selectfont
  \setlength{\tabcolsep}{0.6pt}
  \renewcommand{\arraystretch}{1.24}
  \begin{tabular*}{\linewidth}{@{\extracolsep{\fill}}l|c|ccc|ccc|ccc|ccc@{}}
    \toprule
    \multirow{2}{*}{\textsc{Models}} & \multirow{2}{*}{\shortstack{\textsc{Params}\\(M)}} & \multicolumn{3}{c}{\textsc{Controlled Cylinder}} & \multicolumn{3}{c}{\textsc{FSI}} & \multicolumn{3}{c}{\textsc{Foil}} & \multicolumn{3}{c}{\textsc{Combustion}} \\
    \cmidrule(lr){3-5}\cmidrule(lr){6-8}\cmidrule(lr){9-11}\cmidrule(lr){12-14}
    & & RMSE & Rel $L^2$ & fRMSE & RMSE & Rel $L^2$ & fRMSE & RMSE & Rel $L^2$ & fRMSE & RMSE & Rel $L^2$ & fRMSE \\
    \midrule
    U-Net & 23.08 & \underline{0.80} & \underline{5.55} & 0.10 & \underline{0.85} & \underline{5.83} & \underline{0.07} & \underline{1.00} & \underline{1.59} & 0.11 & 2.16 & 54.87 & 0.26 \\
    CNO & 8.00 & 0.81 & 5.83 & \underline{0.09} & 1.05 & 7.41 & 0.09 & 1.36 & 2.53 & 0.18 & 2.48 & 60.30 & 0.32 \\
    DeepONet & 3.53 & 3.09 & 23.99 & 0.58 & 3.50 & 25.02 & 0.51 & 2.22 & 3.63 & 0.28 & 2.29 & 57.51 & 0.28 \\
    FNO & 109.10 & 0.97 & 7.23 & 0.12 & 1.29 & 8.92 & 0.12 & 1.30 & 2.28 & 0.15 & 2.26 & 56.64 & 0.27 \\
    WDNO & 233.15 & 1.15 & 9.27 & 0.14 & 1.17 & 8.17 & 0.11 & 1.62 & 4.44 & 0.18 & 3.80 & 86.05 & 0.55 \\
    MWT & 2.90 & 1.02 & 7.47 & 0.11 & 1.28 & 9.10 & 0.10 & 1.33 & 2.27 & 0.15 & 2.21 & 55.60 & 0.27 \\
    GK-Transformer & 59.20 & 1.03 & 7.67 & 0.12 & 1.27 & 9.16 & 0.11 & 1.42 & 2.45 & 0.18 & 2.47 & 60.83 & 0.31 \\
    Transolver & 4.30 & 1.71 & 12.00 & 0.24 & 2.36 & 15.67 & 0.27 & 2.20 & 3.70 & 0.27 & 3.75 & 78.25 & 0.50 \\
    DPOT-S-FT & 38.73 & 0.84 & 5.98 & 0.10 & 1.05 & 7.46 & 0.08 & 1.06 & 1.66 & \underline{0.10} & 2.09 & 53.49 & \underline{0.24} \\
    DPOT-L-FT & 667.75 & 0.84 & 6.03 & 0.10 & 0.99 & 6.87 & 0.08 & 1.05 & \underline{1.59} & 0.11 & \underline{2.08} & \underline{53.31} & \underline{0.24} \\
    \midrule
    \textbf{Transolver-$\sigma$} & 3.13 & \textbf{0.65} & \textbf{5.20} & \textbf{0.06} & \textbf{0.74} & \textbf{5.55} & \textbf{0.06} & \textbf{0.47} & \textbf{1.40} & \textbf{0.08} & \textbf{1.86} & \textbf{52.86} & \textbf{0.18}\\
    \bottomrule
  \end{tabular*}
\end{table*}

\textbf{REALM: Coupled Multiphysics Dynamics} To push the evaluation beyond standalone fields and
empirical measurements, we further challenge Transolver-$\sigma$ on complex multi-field coupled reacting
flows from REALM (IgnitHIT and EvolveJet). Coupled multiphysics systems present severe challenges
for neural operators, as non-linear cross-field interactions often trigger explosive error
accumulation during long-horizon dynamic rollouts. As reported in Appendix~\ref{app:realm-training}
(Table~\ref{tab:realm-results}), Transolver-$\sigma$ outperforms all baselines, further validating our
joint spectral--physical modeling.

\vspace{-3pt}
\subsection{Model Analysis}
\vspace{-3pt}
\label{sec:experiment-analysis}

\paragraph{Single-Step Accuracy vs. Rollout Robustness}
Evaluating autoregressive forecasts across four dynamical systems reveals a fundamental trade-off in
single-domain modeling: short-term precision does not guarantee long-horizon stability. As shown in
Figure~\ref{fig:rollout-comparison}(a--d), the relative advantage of physical-only (Physolver) and
spectral-only (Specsolver) models shifts as prediction leads grow. In contrast, Transolver-$\sigma$
achieves lower errors at every prediction lead across all four systems, cutting rollout-averaged
error by \textbf{35.6\%--68.7\%} relative to the stronger single-domain counterpart
(Figure~\ref{fig:rollout-comparison}(e)). These findings confirm that joint physical-state and
spectral modeling is important for suppressing error accumulation in autoregressive rollouts
(details in Appendix~\ref{app:rollout-comparison} and Appendix~\ref{app:exact-error-decomposition}).

\vspace{-3pt}
\paragraph{Efficiency}
We compare training efficiency with representative PDE solvers selected for
their predictive performance and architectural relevance
(Figure~\ref{fig:efficiency-representations}(a,b)).
Transolver-$\sigma$ requires the least parameter storage across both tasks while
maintaining competitive training times.
Specifically, on NS2D, Transolver-$\sigma$ reduces relative $L_2$ error by $34.0\%$
compared with FactFormer, the strongest baseline on this task
($2.79\%$ vs.\ $4.23\%$; Table~\ref{tab:pde-main-results}).
In the efficiency comparison, it has a $3.4\times$ smaller parameter-storage
footprint and $1.05\times$ the mean training speed of FactFormer
(Appendix~\ref{app:efficiency-measurements}).

\begin{figure}[!t]
  \centering
  \includegraphics[width=\textwidth]{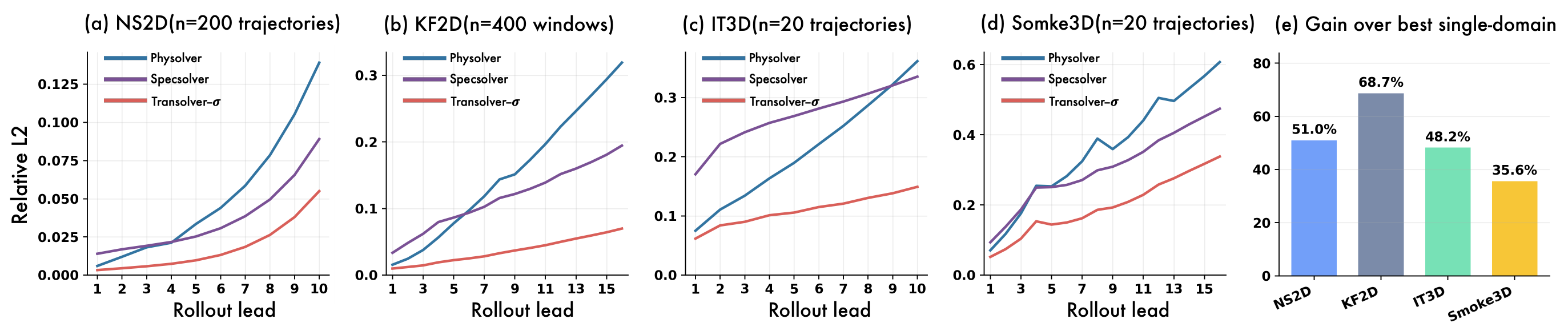}
  \vspace{-10pt}
  \caption{\textbf{Autoregressive forecasting errors across four dynamical systems.}
  \textbf{(a--d)} Relative $L^2$ errors over prediction leads on Standard Benchmarks, with $n$ the
  evaluation sample count.
  \textbf{(e)} Percentage reduction in rollout-averaged error over the stronger single-domain
  baseline.}
  \label{fig:rollout-comparison}
\end{figure}

\begin{table*}[!h]
  \centering
  \newcommand{\AblationLeftWidth}{0.43\textwidth}
  \newcommand{\AblationColumnGap}{0.015\textwidth}

  \begin{minipage}[t]{\AblationLeftWidth}
    \vspace{2pt}
    \justifying

    \textbf{Ablations.}
    We conduct ablation studies on Darcy, NS2D, and Smoke3D to evaluate our
    core design choices (Table~\ref{tab:ablation-components}).
    First, bypassing either subspace branch or replacing our parallel
    architecture with sequential hybrid configurations degrades performance
    on dynamic tasks, with the full model reducing Smoke3D velocity error by
    \textbf{16.0\%} over the stronger sequential baseline.
    Second, removing shared CPE, restricting cross-subspace mixing via grouped
    SwiGLU (NS2D error increases from \textbf{2.79\%} to \textbf{4.32\%}),
    or replacing SRPA with vanilla Physics-Attention~\citep{wu2024Transolver}
    consistently increases errors across all metrics.
    These comparisons collectively support the effectiveness of our joint subspace organization,
    full-channel recomposition, and slice-residual updates.

  \end{minipage}%
  \hspace{\AblationColumnGap}%
  \begin{minipage}[t]{\dimexpr\textwidth-\AblationLeftWidth-\AblationColumnGap\relax}
    \vspace{2pt}
    \raggedright

    \begin{threeparttable}
      \setlength{\parskip}{0pt}
      \captionsetup{position=top, skip=2pt}

      \caption{\textbf{Ablations of Transolver-$\sigma$.}
      Relative $L_2$ errors (\%) for pressure $p$, vorticity $\omega$,
      velocity $\mathbf{u}$, and density $d$.
      Branch ablations use identity mappings; ``w/o Cross-Branch FFN''
      removes cross-branch mixing while retaining grouped SwiGLU.
      SRPA $\rightarrow$ PA replaces SRPA with Physics-Attention.}
      \label{tab:ablation-components}

      \small
      \setlength{\tabcolsep}{3.5pt}
      \renewcommand{\arraystretch}{1.08}
      \begin{tabular*}{\linewidth}{@{\extracolsep{\fill}}lcccc@{}}
        \toprule
        \multirow{2}{*}{\textsc{Variant}}
        & \textsc{Darcy}
        & \textsc{NS-2D}
        & \multicolumn{2}{c}{\textsc{Smoke-3D}} \\
        \cmidrule(lr){2-2}
        \cmidrule(lr){3-3}
        \cmidrule(lr){4-5}
        & $p$ & $\omega$ & $\mathbf{u}$ & $d$ \\
        \midrule
        w/o Physical Branch
        & 0.65 & 4.96 & 26.29 & 8.97 \\
        w/o Spectral Branch
        & \textbf{0.40} & 7.94 & 29.02 & 10.15 \\
        \midrule
        Physical $\rightarrow$ Spectral
        & \underline{0.42} & \underline{3.04} & 20.75 & 7.78 \\
        Spectral $\rightarrow$ Physical
        & \underline{0.42} & 3.15 & 21.88 & 8.06 \\
        \midrule
        w/o Shared CPE
        & \underline{0.42} & 3.69 & 23.56 & 8.59 \\
        w/o Cross-Branch FFN
        & 0.43 & 4.32 & 25.53 & 9.04 \\
        SRPA $\rightarrow$ PA
        & \underline{0.42} & 3.07 & \underline{19.53} & \underline{7.50} \\
        \midrule
        \textbf{Transolver-$\sigma$}
        & \textbf{0.40}
        & \textbf{2.79}
        & \textbf{17.43}
        & \textbf{7.07} \\
        \bottomrule
      \end{tabular*}

    \end{threeparttable}
  \end{minipage}

\end{table*}

\vspace{-3pt}
\paragraph{Slice-Attention Analysis}
As presented in Table~\ref{tab:slice-attention-analysis}, vanilla Physics-Attention can
paradoxically degrade performance relative to its no-attention counterpart---a failure mode also
reflected in LinearNO~\citep{hu2026linearno}, which removes separate slice-to-slice attention
through a linear-attention reformulation. Crucially, this limitation stems not from unnecessary
inter-slice interaction, but from the update mechanism itself: vanilla attention replaces existing
slice states with attention-mixed values without an explicit slice-space identity path, eroding
feature distinctions as queries repeatedly draw from concentrated keys. SRPA addresses this by
retaining incoming slice states and appending a learnably scaled residual correction, allowing
cross-slice interaction to refine rather than replace representations. Consequently, SRPA converts
slice attention from a detrimental update into a beneficial correction, reducing errors on Darcy and
NS2D by \textbf{9.5\%} and \textbf{13.5\%}, respectively, over the no-attention baseline.
Complementary visualizations in Figure~\ref{fig:efficiency-representations}(c,d) confirm that
Transolver-$\sigma$ exhibits spatially structured routing across more slices with less key concentration,
supporting our design goal of preserving distinct states during interaction
(Appendix~\ref{app:slice-attention-analysis}).

\setlength{\floatsep}{4pt plus 1pt minus 1pt}

\begin{figure}[!h]
  \centering
  \includegraphics[width=\textwidth]{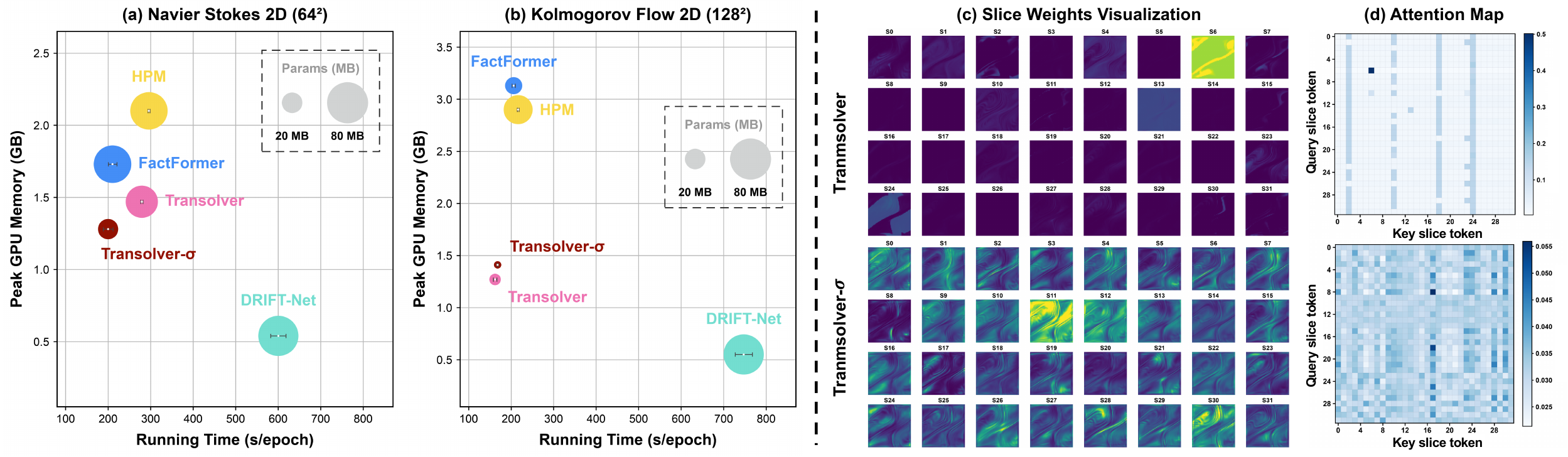}
  \captionsetup{skip=6pt}
  \caption{\textbf{Training efficiency and learned representations.}
  \textbf{(a,b)} Epoch time and peak GPU memory on NS2D ($64^2$) and Kolmogorov flow ($128^2$), with
  bubble area indicating parameter storage and error bars denoting $\pm1$ standard deviation over
  three runs.
  \textbf{(c, d)} Spatial routing maps and slice-attention matrices extracted from the final layer on
  NS2D.}
  \label{fig:efficiency-representations}
\end{figure}

\begin{figure}[!t]
  \centering
  \includegraphics[width=\textwidth]{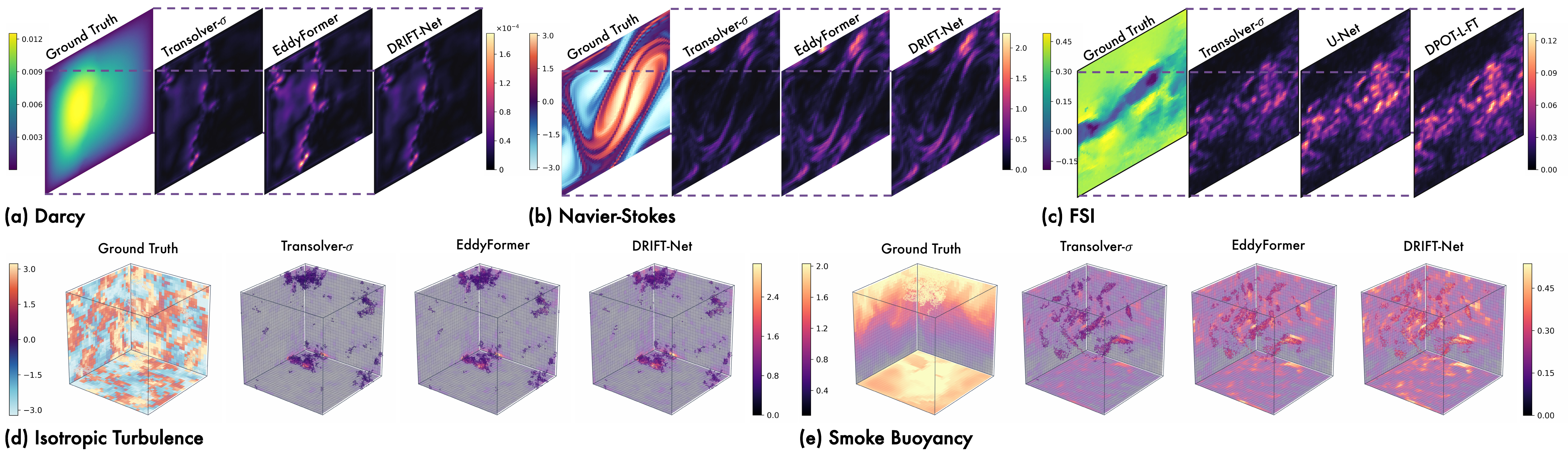}
  \captionsetup{skip=6pt}
  \caption{\textbf{Case study on error maps of different models.} Within each task, models share the
  same test case, prediction time, spatial view, and error scale. See
  Appendix~\ref{app:rollout-case-studies} for more details.}
  \label{fig:case-study-baselines}
\end{figure}

\paragraph{Case study}
Figure~\ref{fig:case-study-baselines} compares spatial error maps of
Transolver-$\sigma$, EddyFormer, and DRIFT-Net on representative simulation
and experimental cases.
In Navier--Stokes, the baselines exhibit pronounced band-shaped errors,
while Transolver-$\sigma$ shows weaker errors along the same structures.
The FSI case likewise shows weaker localized errors for Transolver-$\sigma$ than the baselines.

We further compare Physolver, Specsolver, and Transolver-$\sigma$ across
successive autoregressive leads (Appendix~\ref{app:rollout-case-studies}).
In NS2D, Physolver's initially small errors develop into pronounced bands
at later leads, while Specsolver exhibits weaker late-stage errors.
On Kolmogorov flow, both single-domain models show increasingly prominent
error patterns, whereas Transolver-$\sigma$ maintains lower error intensity
throughout the displayed sequence.
Together with the three-dimensional cases, these visualizations provide
a spatial view of the rollout differences quantified in
Figure~\ref{fig:rollout-comparison}.

\noindent\textbf{Layer-wise SRPA Corrections}\quad
To examine how SRPA uses its residual path after training, we inspect learned
scales and realized correction ratios across five PDE tasks.
Each ratio measures the joint norm of the scaled slice-attention correction
relative to the incoming states; we evaluate three
{\parfillskip=0pt\par}

\begingroup
\setlength{\columnsep}{0.015\textwidth}
\setlength{\intextsep}{4pt}
\Needspace{8\baselineskip} 
\begin{wraptable}{r}{0.60\textwidth}
  \vspace{-5.4pt} 
  \centering
  \setlength{\parskip}{0pt}
  \captionsetup{position=top,skip=2pt}
  \caption{\textbf{Effects of slice-attention updates.}
  }
  \label{tab:slice-attention-analysis}
  \small
  \setlength{\tabcolsep}{1.75pt}
  \renewcommand{\arraystretch}{1.08}
  \setlength{\aboverulesep}{\dimexpr 2\aboverulesep/3-\lightrulewidth/6\relax}
  \setlength{\belowrulesep}{\dimexpr 2\belowrulesep/3-\lightrulewidth/6\relax}
  \begin{tabular*}{\linewidth}{@{\extracolsep{\fill}}llcc@{}}
    \toprule
    Model & Slice-space update & Darcy & NS2D \\
    \midrule
    LinearNO & -- & 0.5033 & \underline{6.9871} \\
    \midrule
    Transolver & $\mathbf{A}\mathbf{Z}\mathbf{W}_v$ & 0.5807 & 9.0315 \\
    \textendash\,w/o Slice Attention & $\mathbf{Z}$ & \underline{0.4724} & 7.7124 \\
    \textendash\,with SRPA & $\mathbf{Z}+\gamma_\ell\mathbf{A}\mathbf{Z}\mathbf{W}_v$ & \textbf{0.4275} & \textbf{6.6735} \\
    \bottomrule
  \end{tabular*}
\end{wraptable}
\noindent
fixed held-out cases per task.
Corrections concentrate primarily in Darcy's first layer but remain measurable
throughout all eight Smoke3D layers.
Learned scale magnitudes alone do not determine the realized corrections.
These task- and layer-dependent profiles complement the preceding update
comparison, illustrating nonuniform use of residual interactions
(Appendix~\ref{app:srpa-layerwise-corrections}).
\par
\endgroup

\vspace{-3pt}
\section{Conclusion}
\vspace{-3pt}

We introduce Transolver-$\sigma$, a neural PDE solver that improves both
field prediction and autoregressive rollout through joint spectral--physical
subspace modeling. The key design is to preserve adaptive physical-state
interactions and spectral transformations in dedicated latent subspaces, while
repeatedly recomposing their responses across the network. Within the physical
subspace, Slice-Residual Physics-Attention further improves state interaction
by refining, rather than replacing, the incoming slice representations.
Transolver-$\sigma$ achieves a benchmark-averaged relative error reduction of
$33.4\%$ across five canonical PDE benchmarks, with consistent gains extending
to coupled multiphysics simulations and real-world measurements. Rollout comparisons and ablations further demonstrate the
complementary behaviors of physical-state and spectral modeling, while
highlighting the importance of dedicated subspace organization and
cross-subspace recomposition. These findings suggest that neural PDE solvers can benefit from preserving
specialized physical and spectral representations while explicitly learning
their interactions.

\setlength{\parskip}{6pt}
\subsection*{AI Use Statement}
We used generative AI tools for language polishing, grammar checking,
and translation, and to assist in developing and refining the
theoretical analyses and mathematical proofs in
Appendix~\ref{app:theoretical-analysis}.
All AI-assisted content was carefully reviewed and checked by the
authors. In particular, the authors reviewed the stated assumptions
and independently checked the derivations and logical steps of all
AI-assisted proofs, verifying that the mathematical claims follow
under the stated conditions. The authors take full responsibility
for the correctness and integrity of the final manuscript, including
all AI-assisted content.

\bibliography{iclr2027_conference}
\bibliographystyle{iclr2027_conference}

\clearpage
\appendix
\raggedbottom
\setlength{\parskip}{2pt}
\setlength{\floatsep}{8pt}
\setlength{\textfloatsep}{6pt}
\setlength{\intextsep}{6pt}
\setlength{\abovedisplayskip}{3pt}
\setlength{\belowdisplayskip}{3pt}
\setlength{\abovedisplayshortskip}{0pt}
\setlength{\belowdisplayshortskip}{3pt}
\captionsetup{skip=2pt}
\makeatletter
\renewcommand{\paragraph}{\@startsection{paragraph}{4}{\z@}{4pt}{-1em}{\normalsize\bfseries}}
\makeatother
\setcounter{topnumber}{5}
\setcounter{bottomnumber}{5}
\setcounter{totalnumber}{10}
\renewcommand{\topfraction}{0.95}
\renewcommand{\bottomfraction}{0.9}
\renewcommand{\textfraction}{0.05}
\renewcommand{\floatpagefraction}{0.85}

\paragraph{Appendix roadmap}
Appendix~\ref{app:theoretical-analysis} presents theoretical analyses and proofs.
Appendix~\ref{app:experiments} collects benchmarks, metrics, and implementation details.
Appendix~\ref{app:full-ablations} contains ablation analyses;
Appendix~\ref{app:supplementary-results} presents learned representations and prediction showcases.
Appendices~\ref{app:additional-experiments} and~\ref{app:efficiency-measurements}
contain additional experiments and the full efficiency analysis, respectively.

\section{Theoretical Analysis and Proofs}
\label{app:theoretical-analysis}

We develop a conditional explanation of why an attention response can be
harmful as a replacement yet useful as a residual correction.
We first characterize the slice--deslice path and changes in pairwise
differences and total centered variation. We then connect these changes to
prediction loss under explicit task conditions. A separate finite-horizon
analysis explains how prediction errors propagate through feedback.

\subsection{Geometry of Slice-Space Updates}
\label{app:srpa-theory}

We analyze the slice-space updates in
Section~\ref{sec:method-srpa}, separating aggregation from inter-slice interactions.
We first characterize the no-attention slice--deslice kernel, then establish
conditional bounds on pairwise differences and centered slice variation.
Both results concern a single head before output projection, with attention
dropout disabled.
The attention matrices are evaluated at a given input and may depend on it.

\paragraph{No-Attention Slice--Deslice Kernel}

\begin{proposition}[Conditional slice--deslice kernel]
  \label{prop:noattention-kernel}
  Fix a nonnegative routing matrix $\mathbf{P}\in\mathbb{R}^{N\times M}$ whose
  rows sum to one, as in Equation~\ref{eq:physical-slicing}.
  Let $\mathbf{D}=\operatorname{diag}(\mathbf{P}^{\top}\mathbf{1}_N
  +\varepsilon\mathbf{1}_M)$ with $\varepsilon>0$; removing cross-slice
  attention gives the following slice--deslice output:
  \begin{equation}
    \mathbf{O}_0=\mathbf{K}_{\mathrm{sd}}\mathbf{F},
    \qquad
    \mathbf{K}_{\mathrm{sd}}
    =\mathbf{P}\mathbf{D}^{-1}\mathbf{P}^{\top}.
    \label{eq:appendix-noattention-kernel}
  \end{equation}
  The matrix $\mathbf{K}_{\mathrm{sd}}$ is symmetric positive semidefinite and
  $\operatorname{rank}(\mathbf{K}_{\mathrm{sd}})\leq M$.
\end{proposition}

\begin{proof}
  Equation~\ref{eq:physical-slicing} gives
  $\mathbf{Z}=\mathbf{D}^{-1}\mathbf{P}^{\top}\mathbf{F}$. The no-attention
  slice-space update is the identity, so deslicing yields
  $\mathbf{O}_0=\mathbf{P}\mathbf{Z}=\mathbf{K}_{\mathrm{sd}}\mathbf{F}$.
  Because $\mathbf{D}$ is positive diagonal,
  $\mathbf{K}_{\mathrm{sd}}^{\top}=\mathbf{K}_{\mathrm{sd}}$. For any
  $\mathbf{x}\in\mathbb{R}^{N}$,
  \begin{equation}
    \mathbf{x}^{\top}\mathbf{K}_{\mathrm{sd}}\mathbf{x}
    =\left\lVert
    \mathbf{D}^{-1/2}\mathbf{P}^{\top}\mathbf{x}
    \right\rVert_2^2\geq 0,
    \label{eq:appendix-noattention-psd}
  \end{equation}
  which proves positive semidefiniteness. Finally,
  $\operatorname{rank}(\mathbf{K}_{\mathrm{sd}})
  \leq\operatorname{rank}(\mathbf{P})\leq M$.
\end{proof}

For fixed routing, the kernel combines mesh features through shared slice
memberships without interactions between distinct slice tokens.
The routing matrix $\mathbf{P}$ is itself learned from the input.
The complete no-attention operator therefore remains input-conditioned,
rather than reducing to a globally fixed linear kernel.
This result concerns the shared-routing factorization of Transolver, not
LinearNO's more general parameterization with independent queries and keys.
The kernel need not be an identity or an orthogonal projection.
Writing $s_m=\sum_n p_{nm}$ also gives
$\mathbf K_{\mathrm{sd}}\mathbf1_N
=\mathbf P(s_m/(s_m+\varepsilon))_{m=1}^M$:
the stabilizer prevents exact constant preservation in general.

\paragraph{Conditional Mixing Pressure and Residual Control}

For comparison, the three slice-space updates have the following forms:
\begin{equation}
  \begin{alignedat}{1}
    &\mathcal{T}_{\mathrm{van}}(\mathbf{Z})
    =\mathbf{A}\mathbf{Z}\mathbf{W}_v,\\
    &\mathcal{T}_{\mathrm{noattn}}(\mathbf{Z})
    =\mathbf{Z},\\
    &\mathcal{T}_{\mathrm{SRPA}}(\mathbf{Z})
    =\mathbf{Z}+\gamma\mathbf{A}\mathbf{Z}\mathbf{W}_v.
    \label{eq:appendix-identity-interaction}
  \end{alignedat}
\end{equation}
The vanilla update provides interaction without an explicit identity path;
the no-attention variant retains states without separate cross-slice refinement.
SRPA combines both paths, with the residual scalar controlling the added
interaction response $\Delta\mathbf{Z}=\mathbf{A}\mathbf{Z}\mathbf{W}_v$
defined in Equation~\ref{eq:srpa-update}.
At $\gamma=0$, deslicing still produces the slice--deslice output
$\mathbf{P}\mathbf{Z}$ characterized above, before the output projection.
The slice-space identity therefore preserves a different representation
from the outer residual path carrying $\mathbf{H}_{\mathrm{phy}}$.
In particular, $\gamma=1$ does not recover vanilla attention, since the
$\mathbf Z$ term remains. We next quantify how these updates affect
differences between slices.

\begin{proposition}[Pairwise refinement and diameter control]
  \label{prop:srpa-mixing}
  For rows $\mathbf{z}_i$ of $\mathbf{Z}$, define
  $\operatorname{diam}(\mathbf{Z})=\max_{i,j}
  \lVert\mathbf{z}_i-\mathbf{z}_j\rVert_2$.
  Let $\mathbf{A}$ be a row-stochastic matrix and define
  $\eta(\mathbf{A})=\frac{1}{2}\max_{i,j}\sum_k|a_{ik}-a_{jk}|$.
  Set $c=\eta(\mathbf{A})\lVert\mathbf{W}_v\rVert_2$,
  $[q]_+=\max(q,0)$, and
  $\mathcal{T}_{\gamma}(\mathbf{Z})
  =\mathbf{Z}+\gamma\mathbf{A}\mathbf{Z}\mathbf{W}_v$.
  For $\eta_{ij}=\tfrac12\|\mathbf a_i-\mathbf a_j\|_1$ and
  $\Delta\mathbf Z=\mathbf A\mathbf Z\mathbf W_v$, each pair satisfies
  \begin{equation}
    \begin{alignedat}{1}
    &\|\Delta\mathbf z_i-\Delta\mathbf z_j\|_2
    \le\eta_{ij}\|\mathbf W_v\|_2\operatorname{diam}(\mathbf Z),\\
    &\|(\mathbf z_i^+-\mathbf z_j^+)-(\mathbf z_i-\mathbf z_j)\|_2
    \le|\gamma|\eta_{ij}\|\mathbf W_v\|_2\operatorname{diam}(\mathbf Z).
    \end{alignedat}
    \label{eq:srpa-pairwise-control}
  \end{equation}
  Consequently, the output diameters obey
  \begin{equation}
    \begin{alignedat}{1}
      &\operatorname{diam}(\mathbf{A}\mathbf{Z}\mathbf{W}_v)
      \leq c\operatorname{diam}(\mathbf{Z}),\\
      &\operatorname{diam}(\mathcal{T}_{\gamma}(\mathbf{Z}))
      \geq[1-|\gamma|c]_+\operatorname{diam}(\mathbf{Z}),\\
      &\operatorname{diam}(\mathcal{T}_{\gamma}(\mathbf{Z}))
      \leq(1+|\gamma|c)\operatorname{diam}(\mathbf{Z}).
      \label{eq:srpa-diameter-bounds}
    \end{alignedat}
  \end{equation}
  For each fixed attention matrix with finite unmasked softmax logits,
  $\eta(\mathbf{A})<1$.
\end{proposition}

\begin{proof}[Proof of Proposition~\ref{prop:srpa-mixing}]
  Let $\mathbf{Y}=\mathbf{Z}\mathbf{W}_v$. For two rows $i$ and $j$ of
  $\mathbf{A}$, set $\mathbf{r}=\mathbf{a}_i-\mathbf{a}_j$. Since
  $\mathbf{A}$ is row-stochastic, $\sum_k r_k=0$. Define
  $\alpha=\frac{1}{2}\lVert\mathbf{r}\rVert_1$. If $\alpha=0$, the two output
  rows are equal. Otherwise, the normalized positive and negative parts of
  $\mathbf{r}$ define probability vectors $\mathbf{p}$ and $\mathbf{q}$, and
  \begin{equation}
    \begin{aligned}
      (\mathbf{A}\mathbf{Y})_i-(\mathbf{A}\mathbf{Y})_j
      &=\alpha\!\left(\sum_k p_k\mathbf{y}_k-
      \sum_m q_m\mathbf{y}_m\right)\\
      &=\alpha\sum_{k,m}p_kq_m(\mathbf{y}_k-\mathbf{y}_m).
      \label{eq:appendix-dobrushin-decomposition}
    \end{aligned}
  \end{equation}
  The triangle inequality gives
  \begin{equation}
    \lVert(\mathbf{A}\mathbf{Y})_i-(\mathbf{A}\mathbf{Y})_j\rVert_2
    \leq\alpha\operatorname{diam}(\mathbf{Y})
    \leq\alpha\lVert\mathbf{W}_v\rVert_2
    \operatorname{diam}(\mathbf{Z}).
    \label{eq:appendix-dobrushin-bound}
  \end{equation}
  Since $\alpha=\eta_{ij}$, this proves the first pairwise bound; the
  residual difference changes by $\gamma(\Delta\mathbf z_i-\Delta\mathbf z_j)$,
  proving the second. Maximizing over $i,j$ proves the first inequality in
  Equation~\ref{eq:srpa-diameter-bounds}.

  For the upper residual bound, apply the triangle inequality to every pair of
  rows:
  \begin{equation}
    \begin{aligned}
      \lVert(\mathbf{z}_i-\mathbf{z}_j)
      +\gamma[(\mathbf{A}\mathbf{Y})_i-(\mathbf{A}\mathbf{Y})_j]\rVert_2
      \leq(1+|\gamma|c)\operatorname{diam}(\mathbf{Z}).
      \label{eq:appendix-residual-upper}
    \end{aligned}
  \end{equation}
  For the lower bound, choose a pair $(i^*,j^*)$ that attains
  $\operatorname{diam}(\mathbf{Z})$ and apply the reverse triangle inequality.
  This yields $(1-|\gamma|c)\operatorname{diam}(\mathbf{Z})$ when the factor is
  nonnegative; the diameter is always nonnegative otherwise. This proves the
  lower residual bound in Equation~\ref{eq:srpa-diameter-bounds}.

  Finally, finite softmax logits give $a_{ik}>0$ for all $i,k$. The equivalent
  identity
  \begin{equation}
    \eta(\mathbf{A})
    =1-\min_{i,j}\sum_k\min(a_{ik},a_{jk})
    \label{eq:appendix-dobrushin-overlap}
  \end{equation}
  then implies $\eta(\mathbf{A})<1$ for the fixed matrix because every pair of
  rows has positive overlap. A uniform contraction factor over all inputs would
  require a uniform lower bound on that overlap. Softmax positivity alone does
  not provide such a bound.
\end{proof}

\paragraph{Identical attention rows}
If every row of $\mathbf A$ equals $\mathbf a^\top$, then
$\Delta\mathbf Z=\mathbf1_M\mathbf c^\top$ for
$\mathbf c^\top=\mathbf a^\top\mathbf Z\mathbf W_v$.
Vanilla attention makes all output slices identical, whereas
$\mathbf z_i^+-\mathbf z_j^+=\mathbf z_i-\mathbf z_j$ for every pair.
Thus the same interaction that removes all slice differences when used
as a replacement preserves them exactly when added as a residual.

\paragraph{Centered slice variation}
Diameter concerns the most separated pair. To measure all pairs together,
let $\mathbf\Pi=\mathbf I_M-M^{-1}\mathbf1_M\mathbf1_M^\top$ and define
$V(\mathbf Z)=\|\mathbf\Pi\mathbf Z\|_F$. Then
\begin{equation}
  V(\mathbf Z)^2=\frac{1}{2M}\sum_{i,j=1}^M
  \|\mathbf z_i-\mathbf z_j\|_2^2.
  \label{eq:centered-pairwise-energy}
\end{equation}
This measures total pairwise variation, rather than rank or a lower bound
on each individual distance.
\begin{proposition}[Centered-variation control]
\label{prop:srpa-centered}
For row-stochastic $\mathbf A$, put
$\kappa=\|\mathbf\Pi\mathbf A\mathbf\Pi\|_2\|\mathbf W_v\|_2$.
Then
\begin{equation}
\begin{aligned}
 V(\Delta\mathbf Z)&\le\kappa V(\mathbf Z),\\
 (1-|\gamma|\kappa)_+V(\mathbf Z)
 &\le V(\mathbf Z+\gamma\Delta\mathbf Z)
 \le(1+|\gamma|\kappa)V(\mathbf Z),
\end{aligned}
\label{eq:srpa-centered-bounds}
\end{equation}
where $(x)_+=\max\{x,0\}$.
\end{proposition}
\begin{proof}
Row stochasticity gives $\mathbf A\mathbf1_M=\mathbf1_M$ and hence
$\mathbf\Pi\mathbf A\mathbf Z
=\mathbf\Pi\mathbf A\mathbf\Pi\mathbf Z$.
Submultiplicativity proves the first bound; the triangle and reverse
triangle inequalities prove the other two.
\end{proof}
Row stochasticity alone does not imply $\kappa<1$.
Similarly, strictly positive softmax weights imply $\eta(\mathbf A)<1$,
but not $\eta(\mathbf A)\|\mathbf W_v\|_2<1$.
The coefficient $\eta$ is the finite-state Dobrushin coefficient;
its diameter interpretation is consistent with the general contraction
framework of Gaubert and Qu.\footnote{S.~Gaubert and Z.~Qu,
\emph{Dobrushin's ergodicity coefficient for Markov operators on cones},
Integral Equations and Operator Theory, 81, 2015,
\url{https://doi.org/10.1007/s00020-014-2193-2}.}
For a given input with $V(\mathbf Z)>0$, the realized centered ratio
$|\gamma|\|\mathbf\Pi\Delta\mathbf Z\|_F/V(\mathbf Z)$ directly bounds
the relative change in $V$. This differs from the uncentered correction
ratios reported in Appendix~\ref{app:srpa-layerwise-corrections}.
The latter quantify total update magnitude, including common-mode changes.

\paragraph{Successive updates}
For an idealized recursion containing only slice-space updates, define
separate coefficients on each trajectory $r\in\{\mathrm{van},\mathrm{res}\}$:
\[
\begin{aligned}
 c_\ell^{r}
 &=\eta(\mathbf A_\ell^{r})\|\mathbf W_{v,\ell}^{r}\|_2,\\
 \kappa_\ell^{r}
 &=\|\mathbf\Pi\mathbf A_\ell^{r}\mathbf\Pi\|_2
   \|\mathbf W_{v,\ell}^{r}\|_2.
\end{aligned}
\]
Starting from the same $\mathbf Z_0$, repeated application of
Propositions~\ref{prop:srpa-mixing} and~\ref{prop:srpa-centered} gives
\begin{equation}
\begin{aligned}
 \operatorname{diam}(\mathbf Z_L^{\mathrm{van}})
 &\le\left(\prod_{\ell=0}^{L-1}c_\ell^{\mathrm{van}}\right)
       \operatorname{diam}(\mathbf Z_0),\\
 \operatorname{diam}(\mathbf Z_L^{\mathrm{res}})
 &\ge\left(\prod_{\ell=0}^{L-1}
       (1-|\gamma_\ell|c_\ell^{\mathrm{res}})_+\right)
       \operatorname{diam}(\mathbf Z_0),
\end{aligned}
\label{eq:appendix-depth-bound}
\end{equation}
and, for total centered variation,
\begin{equation}
\begin{aligned}
 V(\mathbf Z_L^{\mathrm{van}})
 &\le\left(\prod_{\ell=0}^{L-1}\kappa_\ell^{\mathrm{van}}\right)
       V(\mathbf Z_0),\\
 V(\mathbf Z_L^{\mathrm{res}})
 &\ge\left(\prod_{\ell=0}^{L-1}
       (1-|\gamma_\ell|\kappa_\ell^{\mathrm{res}})_+\right)
       V(\mathbf Z_0).
\end{aligned}
\label{eq:appendix-centered-depth-bound}
\end{equation}
These bounds follow by induction and do not identify the coefficients
across the two trajectories.
Uniformly $c_\ell^{\mathrm{van}}\le c<1$ gives geometric diameter
contraction; uniformly $\kappa_\ell^{\mathrm{van}}\le\kappa<1$
gives geometric contraction of $V$.
For residual updates, assume $\operatorname{diam}(\mathbf Z_0)>0$
(equivalently $V(\mathbf Z_0)>0$).
The respective finite-depth lower bound is strictly positive when
$|\gamma_\ell|c_\ell^{\mathrm{res}}<1$, or
$|\gamma_\ell|\kappa_\ell^{\mathrm{res}}<1$, at every layer.
In each case, if these nonnegative step coefficients are also summable,
the infinite product is positive, giving a depth-independent lower bound.
Small steps alone do not imply this summability condition.

In the complete architecture, routing, deslicing, projections, and
feature mixing intervene between slice updates.
For either measure $\mathcal V=\operatorname{diam}$ or $\mathcal V=V$,
an extension along the actual residual trajectory requires
\[
 \mathcal V(\mathbf Z_{\ell+1})
 \ge b_\ell\mathcal V(\mathbf Z_\ell^+),\qquad b_\ell\ge0,
\]
where $\mathbf Z_\ell^+$ is the updated state and
$\mathbf Z_{\ell+1}$ the next incoming slice representation.
The corresponding residual product then acquires factors $b_\ell$;
a positive finite-depth lower bound requires a nonzero initial measure
and strictly positive residual and transition factors.
This is a condition on the realized representations, allowing routing
and other branches to contribute to the next state.
It does not assume that $\mathbf Z_\ell^+$ alone determines
$\mathbf Z_{\ell+1}$.
Zero initialization preserves the incoming state inside SRPA at
initialization; the learned scales control subsequent corrections.

\subsection{When Residual Interaction Improves Prediction}
\label{app:srpa-utility}
\begingroup
\setlength{\abovedisplayskip}{5pt plus 1pt minus 1pt}
\setlength{\belowdisplayskip}{5pt plus 1pt minus 1pt}
\setlength{\abovedisplayshortskip}{3pt plus 1pt minus 1pt}
\setlength{\belowdisplayshortskip}{4pt plus 1pt minus 1pt}
\setlength{\jot}{3pt}
The preceding bounds explain preservation of slice differences.
To connect preservation to accuracy, we now state conditions under which
an interaction hurts as a replacement but helps as a residual.
These are sufficient conditions, not assumptions on every possible task.

\paragraph{Shared-gate comparison}
Fix $S$ inputs and all upstream parameters at one block. Collect the
slice states of all $n_h$ heads in $\mathbf Z=(\mathbf Z_{s,r})$ and their
responses in $\Delta\mathbf Z=(\mathbf A_{s,r}\mathbf Z_{s,r}\mathbf W_{v,r})$.
Use the cohort inner product
\[
 \langle\mathbf X,\mathbf Y\rangle_{\mathcal D}
 =S^{-1}\sum_{s=1}^S\sum_{r=1}^{n_h}
 \langle\mathbf X_{s,r},\mathbf Y_{s,r}\rangle_F
\]
and its induced norm. The projection $\mathbf\Pi$ acts on each head.
The three alternatives are $\mathbf Z$ (no slice attention),
$\Delta\mathbf Z$ (replacement), and
$\mathbf Z+\gamma\Delta\mathbf Z$ (SRPA), with a single real $\gamma$
shared across all inputs and heads, as in one SRPA block.
The learnable block gate $\gamma_\ell\in\mathbb R$ is unconstrained
and initialized to zero; the analysis permits either sign.
Suppose
\begin{equation}
 s_Z=\|\mathbf\Pi\mathbf Z\|_{\mathcal D}>0,\qquad
 \|\mathbf\Pi\Delta\mathbf Z\|_{\mathcal D}\le q s_Z,\qquad 0\le q<1.
 \label{eq:utility-contraction}
\end{equation}
For example, Proposition~\ref{prop:srpa-centered} supplies this condition
if every head has $\kappa_{s,r}\le q$.

\begin{theorem}[Useful refinement of an informative state]
\label{thm:latent-refinement}
Let an ideal slice target be $\mathbf T=\mathbf Z+\mathbf E$, with
\[
\begin{aligned}
 &0<e=\|\mathbf E\|_{\mathcal D}<\tfrac12(1-q)s_Z,
   \qquad \Delta\mathbf Z\ne0,\\
 &|\langle\mathbf E,\Delta\mathbf Z\rangle_{\mathcal D}|
   \ge\alpha e\|\Delta\mathbf Z\|_{\mathcal D},
   \qquad 0<\alpha\le1.
\end{aligned}
\]
Under \eqref{eq:utility-contraction}, define
$\gamma_*=\langle\mathbf E,\Delta\mathbf Z\rangle_{\mathcal D}/
\|\Delta\mathbf Z\|_{\mathcal D}^2$ and choose
$\gamma=\theta\gamma_*$, $0<\theta\le1$.
For squared slice-target loss, the three alternatives satisfy
\begin{equation}
\begin{aligned}
 &\mathcal L_{\mathrm{SRPA}}
 \le[1-\theta(2-\theta)\alpha^2]e^2
 <\mathcal L_{\mathrm{no}}=e^2<\mathcal L_{\mathrm{van}},\\
 &\|\gamma\Delta\mathbf Z\|_{\mathcal D}\le\theta e,\\
 &\|\mathbf\Pi(\mathbf Z+\gamma\Delta\mathbf Z)\|_{\mathcal D}
 \ge s_Z-\theta e>q s_Z
 \ge\|\mathbf\Pi\Delta\mathbf Z\|_{\mathcal D}.
\end{aligned}
\label{eq:latent-loss-order}
\end{equation}
\end{theorem}
\begin{proof}
Orthogonal projection and the reverse triangle inequality give
\[
 \|\Delta\mathbf Z-\mathbf Z\|_{\mathcal D}
 \ge\|\mathbf\Pi(\Delta\mathbf Z-\mathbf Z)\|_{\mathcal D}
 \ge(1-q)s_Z.
\]
Thus
\[
 \sqrt{\mathcal L_{\mathrm{van}}}
 =\|\Delta\mathbf Z-\mathbf Z-\mathbf E\|_{\mathcal D}
 \ge(1-q)s_Z-e>e.
\]
Writing $c=\langle\mathbf E,\Delta\mathbf Z\rangle_{\mathcal D}$ and
$d=\|\Delta\mathbf Z\|_{\mathcal D}>0$ gives
\[
 \mathcal L_{\mathrm{SRPA}}
 =\|\gamma\Delta\mathbf Z-\mathbf E\|_{\mathcal D}^2
 =e^2-\theta(2-\theta)c^2/d^2.
\]
The alignment assumption proves the strict loss improvement.
Cauchy--Schwarz gives $|\gamma|d=\theta|c|/d\le\theta e$.
Projecting this bound and applying the reverse triangle inequality
proves the centered-variation statement.
\end{proof}
Here the no-attention state already contains task-relevant differences,
and the interaction has a component along its remaining target error.
Replacement discards too much of that state; the shared residual scale
extracts a useful correction whose norm is bounded by the remaining error.
The ideal latent target is a modeling assumption. The next proposition states
an analogous result directly for observed output targets.

\paragraph{A nonempty example}
Consider one input and one head with two channels:
\[
 \mathbf Z=
 \begin{pmatrix}-1&0\\1&0\end{pmatrix},\qquad
 \mathbf A=
 \begin{pmatrix}0.75&0.25\\0.25&0.75\end{pmatrix},\qquad
 \mathbf W_v=\mathbf I_2,\qquad
 \mathbf T=
 \begin{pmatrix}-1.1&-0.1\\1.1&0.1\end{pmatrix}.
\]
Here $\Delta\mathbf Z=\tfrac12\mathbf Z$ has nonzero centered variation.
The conditions of Theorem~\ref{thm:latent-refinement} hold with
$q=\tfrac12$, $s_Z=\sqrt2$, $e=0.2$, and $\alpha=1/\sqrt2$.
The shared optimum $\gamma_*=0.2$ gives
\[
\begin{aligned}
 \mathcal L_{\mathrm{no}}&=\|\mathbf Z-\mathbf T\|_F^2=0.04,\\
 \mathcal L_{\mathrm{van}}&=\|\Delta\mathbf Z-\mathbf T\|_F^2=0.74,\\
 \mathcal L_{\mathrm{SRPA}}
 &=\|\mathbf Z+0.2\Delta\mathbf Z-\mathbf T\|_F^2=0.02.
\end{aligned}
\]
The nonuniform interaction corrects the first-channel discrepancy
while leaving the orthogonal second-channel error.
Thus useful residual refinement need not recover the entire target:
it reduces loss while retaining centered variation
$1.1\sqrt2$, compared with $\tfrac12\sqrt2$ after replacement.
Replicating the example across inputs and heads preserves the same
shared gate and the strict ordering.

\begin{proposition}[Output-space extension under a fixed downstream network]
\label{thm:output-refinement}
Retain \eqref{eq:utility-contraction}. For each input, let $G_s$ denote
the fixed network downstream of the selected slice update, including
deslicing, projections, residual paths, and subsequent blocks.
All other inputs to this network are held fixed; its intermediate
activations are recomputed when the slice input changes.
Let $G$ collect the outputs and $\mathbf Y$ their targets, with the
output cohort norm defined by
\[
 \|\mathbf V\|_{\mathcal D}
 =(S^{-1}\sum_{s=1}^S\|\mathbf V_s\|_F^2)^{1/2}.
\]
Assume, for some $m>0$,
\begin{equation}
\begin{aligned}
 &\|G(\Delta\mathbf Z)-G(\mathbf Z)\|_{\mathcal D}
 \ge m\|\mathbf\Pi(\Delta\mathbf Z-\mathbf Z)\|_{\mathcal D},\\
 &2\varepsilon_0<m(1-q)s_Z,\qquad
 \varepsilon_0=\|G(\mathbf Z)-\mathbf Y\|_{\mathcal D}.
\end{aligned}
\label{eq:output-observability}
\end{equation}
Define $\phi(\gamma)=
\|G(\mathbf Z+\gamma\Delta\mathbf Z)-\mathbf Y\|_{\mathcal D}^2$.
Suppose $\phi$ is differentiable on $[-r_0,r_0]$, its derivative is
$\beta$-Lipschitz there, with $r_0,\beta>0$, and $g=\phi'(0)\ne0$.
The latter implies $d=\|\Delta\mathbf Z\|_{\mathcal D}>0$.
For any $0<\xi<1$, choose
\[
 0<t<\min\{r_0,2|g|/\beta,\xi s_Z/d\},\qquad tq<1-q.
\]
Then the shared gate $\gamma=-t\,\operatorname{sign}(g)$ satisfies
\begin{equation}
\begin{aligned}
 &\phi(\gamma)<\phi(0)
 <\|G(\Delta\mathbf Z)-\mathbf Y\|_{\mathcal D}^2,\\
 &\|\gamma\mathbf\Pi\Delta\mathbf Z\|_{\mathcal D}
 \le\|\gamma\Delta\mathbf Z\|_{\mathcal D}<\xi s_Z,\\
 &\|\mathbf\Pi(\mathbf Z+\gamma\Delta\mathbf Z)\|_{\mathcal D}
 >q s_Z\ge\|\mathbf\Pi\Delta\mathbf Z\|_{\mathcal D}.
\end{aligned}
\label{eq:output-loss-order}
\end{equation}
\end{proposition}
\begin{proof}
The reverse triangle inequality and \eqref{eq:output-observability} give
\[
 \|G(\Delta\mathbf Z)-\mathbf Y\|_{\mathcal D}
 \ge m(1-q)s_Z-\varepsilon_0>\varepsilon_0.
\]
Integrating the Lipschitz derivative along the gate interval gives
$\phi(\gamma)\le\phi(0)+g\gamma+\beta\gamma^2/2$.
The chosen sign and step yield
$\phi(\gamma)\le\phi(0)-t|g|+\beta t^2/2<\phi(0)$.
The additional step restriction gives
$\|\gamma\Delta\mathbf Z\|_{\mathcal D}=td<\xi s_Z$;
orthogonal projection cannot increase this norm.
Since $s_Z\le\|\mathbf Z\|_{\mathcal D}$, it also follows that
$\|\gamma\Delta\mathbf Z\|_{\mathcal D}<\xi\|\mathbf Z\|_{\mathcal D}$,
so the total correction is small relative to the incoming state.
Finally, \eqref{eq:utility-contraction} gives
$\|\mathbf\Pi(\mathbf Z+\gamma\Delta\mathbf Z)\|_{\mathcal D}
\ge(1-tq)s_Z>q s_Z$.
The admissible interval is nonempty, including when $q=0$.
\end{proof}

\paragraph{Isolating the effect of slice differences}
A stronger sufficient condition separates centered changes from
common-mode changes. Define
\[
 \mathbf Z_c=\mathbf Z+\mathbf\Pi(\Delta\mathbf Z-\mathbf Z).
\]
This replaces the centered part of each head while retaining its
incoming slice mean. Under the remaining assumptions of
Proposition~\ref{thm:output-refinement}, its conclusions also hold if
the two inequalities in \eqref{eq:output-observability} are replaced by
\begin{equation}
\begin{aligned}
 &\|G(\mathbf Z_c)-G(\mathbf Z)\|_{\mathcal D}
 \ge m_c\|\mathbf\Pi(\Delta\mathbf Z-\mathbf Z)\|_{\mathcal D},\\
 &\|G(\Delta\mathbf Z)-G(\mathbf Z_c)\|_{\mathcal D}
 \le b_{\mathrm{cm}},\\
 &m_c(1-q)s_Z>2\varepsilon_0+b_{\mathrm{cm}},
 \qquad m_c>0,\quad b_{\mathrm{cm}}\ge0.
\end{aligned}
\label{eq:centered-output-sufficient-condition}
\end{equation}
Indeed, the triangle and reverse triangle inequalities give
\[
\begin{aligned}
 \|G(\Delta\mathbf Z)-\mathbf Y\|_{\mathcal D}
 &\ge\|G(\mathbf Z_c)-G(\mathbf Z)\|_{\mathcal D}\\
 &\quad-\|G(\Delta\mathbf Z)-G(\mathbf Z_c)\|_{\mathcal D}
     -\varepsilon_0\\
 &\ge m_c(1-q)s_Z-b_{\mathrm{cm}}-\varepsilon_0
 >\varepsilon_0.
\end{aligned}
\]
The shared-gate descent and correction bounds then follow from the
same proof. The first condition now requires an output response to
changing slice differences at fixed means; $b_{\mathrm{cm}}$ limits
how much the remaining common-mode change can cancel that response.
A readout depending only on slice means cannot satisfy this condition,
since $(1-q)s_Z>0$ but $G(\mathbf Z_c)=G(\mathbf Z)$.
For the nonuniform-attention example with identity readout,
$\mathbf Z_c=\Delta\mathbf Z$, and $m_c=1$, $b_{\mathrm{cm}}=0$
satisfy all three conditions.

\paragraph{Meaning of the conditions}
The first inequality in \eqref{eq:output-observability} requires the
complete replacement displacement to have a sufficiently large effect
on the prediction output, relative to its centered magnitude.
The output response may involve both centered and common-mode components;
the condition does not attribute it exclusively to lost slice differences.
It concerns one comparison direction rather than global invertibility. For a linear prediction map, a positive
minimum singular value on the subspace containing
$\Delta\mathbf Z-\mathbf Z$ is sufficient.
When $G$ is differentiable,
$g=2\langle G(\mathbf Z)-\mathbf Y,
DG(\mathbf Z)[\Delta\mathbf Z]\rangle_{\mathcal D}$;
a nonzero value identifies a useful signed direction for the shared gate.
For the abstract identity readout $G(\mathbf X)=\mathbf X$ and
$\mathbf Y=\mathbf T$ in the example above, one may take
$m=1$, $\varepsilon_0=0.2$, $g=-0.2$, and $\beta=1$.
Taking $r_0=1$, $\xi=0.2$, and $t=0.2$ satisfies all step restrictions,
so the output-space conditions can hold simultaneously.
Both the alignment and the derivative are aggregated across the cohort:
contributions from different inputs or heads can cancel.
The improvement is in the cohort loss and allows trade-offs between
individual inputs; it does not require each input to improve.

Theorem~\ref{thm:latent-refinement} and
Proposition~\ref{thm:output-refinement}
establish the ordering of no-attention, replacement, and residual
interaction under stated task conditions. They explain how preserving
an informative state and using the same interaction as a correction
can change attention's contribution from harmful to beneficial.
The guarantee is existence of a shared improving gate for a fixed
comparison cohort and fixed surrounding parameters; learning and
generalization determine whether it is realized after training.
At the hypothesis-class level, setting every SRPA gate to zero recovers
the corresponding no-slice-attention network. Consequently, the infimum
of empirical loss over the SRPA class is no larger than that over the
no-slice-attention class, whether or not either infimum is attained.
Strict improvement requires additional conditions such as those above.
The no-slice-attention control in
Table~\ref{tab:slice-attention-analysis} follows this update rule;
LinearNO has a different parameterization.
The theoretical comparison holds the responses and surrounding parameters
fixed, whereas the table reports independently trained models with
potentially different routing, attention, and prediction parameters.
The analysis uses squared cohort loss as an analytic surrogate, while
the table reports relative $L_2$ error. The established loss ordering
does not directly imply the same ordering under every benchmark
aggregation metric. The experiments therefore offer a complementary
comparison of trained models, rather than numerical instances of the
fixed-parameter loss inequalities.
\endgroup

\subsection{Architecture-Agnostic Finite-Horizon Error Propagation}
\label{app:error-propagation}

The following architecture-agnostic analysis separates approximation error
from propagation of errors already present in an autoregressive input.
Here $\|\cdot\|$ denotes the Euclidean norm of a vectorized state or output
block, and matrix norms are its induced operator norm.
For $T_{\mathrm{in}}=T_{\mathrm{out}}=1$, let
$\mathbf{u}_n=\operatorname{vec}(\mathbf{U}_n)
\in\mathbb{R}^{NC_{\mathrm{out}}}$ denote the vectorized exact state.
Write $\mathbf{f}_n:\mathbb{R}^{NC_{\mathrm{out}}}
\rightarrow\mathbb{R}^{NC_{\mathrm{out}}}$ for the learned state update,
with coordinates and supplied exogenous conditions held fixed.
Thus, $\widehat{\mathbf{u}}_{n+h+1}
=\mathbf{f}_{n+h}(\widehat{\mathbf{u}}_{n+h})$ defines the vectorized
autoregressive prediction at each successive step.
Define the prediction error and the local approximation error along the
exact trajectory as
\begin{equation}
  \begin{alignedat}{1}
    &\mathbf{e}_{n+h}
    =\widehat{\mathbf{u}}_{n+h}-\mathbf{u}_{n+h},\\
    &\boldsymbol{\delta}_{n+h}
    =\mathbf{f}_{n+h}(\mathbf{u}_{n+h})-\mathbf{u}_{n+h+1}.
    \label{eq:rollout-definition}
  \end{alignedat}
\end{equation}
Subtracting the true next state gives the exact decomposition
\begin{equation}
 \mathbf e_{n+h+1}
 =\boldsymbol\delta_{n+h}
 +\mathbf f_{n+h}(\mathbf u_{n+h}+\mathbf e_{n+h})
 -\mathbf f_{n+h}(\mathbf u_{n+h}).
 \label{eq:rollout-exact-identity}
\end{equation}
Assume $\mathbf f_{n+h}$ is continuously differentiable on a neighborhood
of the segment joining the true and predicted inputs. Define
\[
 \overline{\mathbf J}_{n+h}
 =\int_0^1D\mathbf f_{n+h}
   (\mathbf u_{n+h}+t\mathbf e_{n+h})\,dt.
\]
Then \eqref{eq:rollout-exact-identity} equals
$\mathbf e_{n+h+1}=\boldsymbol\delta_{n+h}
+\overline{\mathbf J}_{n+h}\mathbf e_{n+h}$.
If the Jacobian is $L_{n+h}$-Lipschitz on that neighborhood, expansion
at the true input also gives
\begin{equation}
\begin{aligned}
 \mathbf e_{n+h+1}
 &=\boldsymbol\delta_{n+h}
   +D\mathbf f_{n+h}(\mathbf u_{n+h})\mathbf e_{n+h}
   +\mathbf r_{n+h},\\
 \|\mathbf r_{n+h}\|&\le\tfrac12L_{n+h}\|\mathbf e_{n+h}\|^2.
\end{aligned}
\label{eq:error-propagation}
\end{equation}
For any bounds $a_h\ge\|\overline{\mathbf J}_{n+h}\|$, repeated
substitution and the triangle inequality yield the finite-horizon bound
\begin{equation}
 \|\mathbf e_{n+H}\|
 \le\left(\prod_{j=0}^{H-1}a_j\right)\|\mathbf e_n\|
 +\sum_{k=0}^{H-1}
   \left(\prod_{j=k+1}^{H-1}a_j\right)\|\boldsymbol\delta_{n+k}\|,
 \label{eq:rollout-finite-bound}
\end{equation}
where empty products equal one. This follows by induction from
$\|\mathbf e_{n+h+1}\|\le
\|\boldsymbol\delta_{n+h}\|+a_h\|\mathbf e_{n+h}\|$.
With an exact initial history, the first term vanishes.
Both local prediction defects and subsequent amplification therefore
determine rollout error; small one-step defects alone need not imply
small errors at a longer horizon.

For general history and output-block lengths, let $q$ denote the number of
predicted physical variables per frame, so that $C_{\mathrm{out}}=T_{\mathrm{out}}q$ under
the notation in \hyperref[sec:method-setup]{Problem Setup}.
For histories of these variables, let
$\mathbf{s}_k\in\mathbb{R}^{T_{\mathrm{in}}Nq}$ vectorize the exact
$T_{\mathrm{in}}$-frame history at model call $k$.
The complete deployed update predicts $T_{\mathrm{out}}$ frames using
$\mathcal{F}_{\theta}$, appends them to the history, and retains the most
recent $T_{\mathrm{in}}$ frames.
With supplied conditions held fixed at each call, this defines a
history-update map
$\mathbf{f}^{\mathrm{hist}}_k:\mathbb{R}^{T_{\mathrm{in}}Nq}
\rightarrow\mathbb{R}^{T_{\mathrm{in}}Nq}$, distinct from the map that
outputs only the new frames.
Equations~\ref{eq:rollout-exact-identity}--\ref{eq:rollout-finite-bound}
apply to this complete map under their respective smoothness assumptions,
using errors in the history state and its exact next history.
Here, $k$ counts model calls, each advancing $T_{\mathrm{out}}$ forecast frames, rather
than individual physical prediction leads.
Framewise errors are evaluated from the corresponding predicted block,
including frames not retained in history when $T_{\mathrm{out}}>T_{\mathrm{in}}$.
To bound the entire block, write $\mathcal B_k(\mathbf s)$ for the
$T_{\mathrm{out}}$-frame output map and $\mathbf y_k$ for its exact target.
If $\mathcal B_k$ is $L_k^{\mathrm{out}}$-Lipschitz on the segment between
$\mathbf s_k$ and the predicted history $\widehat{\mathbf s}_k$, then
\begin{equation}
 \|\mathcal B_k(\widehat{\mathbf s}_k)-\mathbf y_k\|
 \le \|\mathcal B_k(\mathbf s_k)-\mathbf y_k\|
    +L_k^{\mathrm{out}}\|\widehat{\mathbf s}_k-\mathbf s_k\|.
 \label{eq:rollout-block-output-bound}
\end{equation}
This follows by adding and subtracting $\mathcal B_k(\mathbf s_k)$.
Combined with \eqref{eq:rollout-finite-bound} for the history state,
it bounds errors in all $T_{\mathrm{out}}$ output frames, including those discarded
before the next model call.

The propagation term depends on both the model and its accumulated input
error, rather than identifying a Jacobian norm in isolation.
This architecture-independent relation motivates evaluating repeated deployment in addition to
predictions from exact inputs.
It does not assign local error or propagation sensitivity exclusively to
either branch of Transolver-$\sigma$.
Throughout the paper, rollout stability denotes empirical robustness over
the evaluated finite horizon, not a formal numerical or asymptotic guarantee.
Appendix~\ref{app:exact-error-decomposition} evaluates an exact finite-difference
decomposition on stored checkpoints, avoiding the small-error approximation
in Equation~\ref{eq:error-propagation}.

\subsection{Conditional Spectral--Physical Complementarity in Rollouts}
\label{app:spectral-physical-complementarity}

Sections~\ref{app:srpa-theory}--\ref{app:srpa-utility} analyze the slice
update; here we study why combining two representation families can help
recursive prediction. The mechanism is directional: a small one-step
error can omit a weak but persistent dynamical effect, whereas a larger
error can decay under repeated deployment. We first identify structural
properties of the two operator cores, then prove a ranking reversal in a
finite-capacity model and give conditions under which approximate learned
recomposition preserves a joint advantage. All results are conditional;
Fourier parameterization alone does not impose contraction.

\begingroup
\setlength{\abovedisplayskip}{3pt plus 1pt minus 1pt}
\setlength{\belowdisplayskip}{3pt plus 1pt minus 1pt}
\setlength{\abovedisplayshortskip}{2pt}
\setlength{\belowdisplayshortskip}{2pt}
\setlength{\jot}{3pt}
\newtheorem*{complementcorollary}{Corollary}

\paragraph{What the operator structures imply}
The norm below is Euclidean after vectorization, with induced operator
norm $\|\cdot\|_2$ and Frobenius norm $\|\cdot\|_F$ for matrices.
For the axis-factorized linear spectral core $\mathcal S$, orthonormal
Fourier transforms, truncation, and mode-wise channel matrices give
\begin{equation}
 \|D\mathcal S\|_2=\|\mathcal S\|_2
 \le \sum_{a=1}^{d}\max_{0\le k<K_a}\|\mathbf W_a(k)\|_2.
 \label{eq:complement-spectral-gain}
\end{equation}
Indeed, each axis transform is unitary in its full Fourier representation,
truncation is an orthogonal projection, and each multiplier is block
diagonal in that basis; apply the triangle inequality to the axis sum.
For real FFTs this argument uses the conjugate-symmetric full spectrum
and the real projection at self-conjugate frequencies. The bound concerns
the spectral core, before normalization, residual paths, and nonlinear
recomposition. It neither bounds the full network by itself nor implies
that its multipliers are contractive.

For the physical core, consider one head and let $\mathcal O$ denote
the linear part of its pointwise output projection. We omit the constant
output bias, which does not affect the derivative:
\[
 \Psi(\mathbf Y)
 =\mathcal O\!\left(\mathbf P(\mathbf Y)
                   \mathcal G(\mathbf Z(\mathbf Y))\right).
\]
Here $\mathcal G$ maps the $M\times d_h$ slice tensor to an equally sized
token response, including the identity and scaled attention for SRPA.
Assume differentiability at the input and that each routing row depends
only on a fixed local neighborhood of that point, as in the convolutions
in Appendix~\ref{app:srpa-implementation}. The chain rule gives
\begin{equation}
 D\Psi(\mathbf Y)[\mathbf V]
 =\mathcal O\!\left(
 \begin{aligned}
  &(D\mathbf P(\mathbf Y)[\mathbf V])\mathcal G(\mathbf Z)\\
  &{}+\mathbf P\,D\mathcal G(\mathbf Z)[D\mathbf Z(\mathbf Y)[\mathbf V]]
 \end{aligned}
 \right).
 \label{eq:complement-routing-derivative}
\end{equation}
The first term is local in its input perturbation. After vectorization,
the second factors through an $Md_h$-dimensional token space and has rank
at most $Md_h$; across heads the corresponding bound is $MC_{\mathrm{phy}}$.
This retains routing derivatives rather than freezing the adaptive
weights. Their coefficients depend on the state, so this structural
factorization supplies no automatic small-gain bound. It also does not
make the complete, multilayer physical-only model a rank-$M$ operator.
These facts motivate a controlled comparison of structured propagation
and limited interaction capacity.

\paragraph{A representation model with complementary omissions}
Work on an $n$-dimensional real resolved space. Let $\mathbf B_F$ be a fixed
orthogonal real Fourier transform matrix, and suppose the exact one-call
evolution is
\begin{equation}
\begin{alignedat}{2}
 &\mathbf u_{k+1}=\mathbf T\mathbf u_k,\qquad
 &\mathbf T&=\mathbf I-h\mathbf L,\qquad
             \mathbf L=\mathbf L_0+\mathbf R,\\
 &\mathbf L_0=\mathbf B_F^\top\operatorname{diag}(\mathbf d)\mathbf B_F,\qquad
 &\mathbf R&=\mathbf U\mathbf C\mathbf U^\top.
\end{alignedat}
 \label{eq:complement-dynamics}
\end{equation}
Assume $\mathbf d\ge0$, $\mathbf C\succeq0$,
$\operatorname{rank}(\mathbf R)\le r<n$, and
$0<\mu\mathbf I\preceq\mathbf L\preceq\Lambda\mathbf I$.
Choose $0<h\le\Lambda^{-1}$ and put $\chi=1-h\mu<1$.
Then $0\preceq\mathbf T\preceq\chi\mathbf I$.
The Fourier-diagonal term describes dissipation in the prescribed Fourier basis;
a low-rank positive interaction can describe additional damping on a few
spatial structures and generally couples Fourier modes.
With the common local path fixed to the identity, define
\begin{equation}
\begin{aligned}
 \mathcal C_{\mathrm{phy}}
   &=\{\mathbf I+\mathbf K:\operatorname{rank}(\mathbf K)\le r\},\\
 \mathcal C_{\mathrm{spec}}
   &=\{\mathbf B_F^\top\operatorname{diag}(\mathbf a)\mathbf B_F:
                                      \mathbf a\in\mathbb R^n\},\\
 \mathcal C_{\mathrm{joint}}
   &=\{\mathbf S+\mathbf K:\mathbf S\in\mathcal C_{\mathrm{spec}},
                                    \operatorname{rank}(\mathbf K)\le r\}.
\end{aligned}
\label{eq:complement-model-classes}
\end{equation}
These classes isolate the two core structures. The spectral class allows
independent gains on all resolved Fourier coordinates; it does not impose
the axis-additive symbol or retained-mode budget of the implementation.
The physical class fixes the common local path to the identity and limits
the rank of its correction. Neither is the full hypothesis class of its
deep single-branch counterpart. The joint class allocates both types of
capacity; this comparison does not assert equal parameter counts or that
two full-width baselines fit into two half-width branches.

\begin{theorem}[Complementary approximation and rollout ranking reversal]
\label{thm:complement-ranking-reversal}
Let the initial state be zero-mean with covariance $\mathbf I$, and train
each matrix class in \eqref{eq:complement-model-classes} by minimizing
$\mathbb E\|\widehat{\mathbf T}\mathbf u_0-\mathbf T\mathbf u_0\|^2$.
Write $\lambda_1\ge\cdots\ge\lambda_n>0$ and $\mathbf v_i$ for an
orthonormal eigensystem of $\mathbf L$. One physical minimizer and the
unique spectral minimizer are
\begin{equation}
\begin{aligned}
 \mathbf T_{\mathrm{phy}}
   &=\mathbf I-h\sum_{i=1}^{r}\lambda_i\mathbf v_i\mathbf v_i^\top,\\
 \mathbf T_{\mathrm{spec}}
   &=\mathbf B_F^\top\operatorname{diag}
                  (\operatorname{diag}(\mathbf B_F\mathbf T\mathbf B_F^\top))
                  \mathbf B_F.
\end{aligned}
\label{eq:complement-optimal-cores}
\end{equation}
Define $E_m(k)^2=\mathbb E\|\mathbf T_m^k\mathbf u_0-
\mathbf T^k\mathbf u_0\|^2$ for $m\in\{\mathrm{phy},\mathrm{spec}\}$.
Then
\begin{equation}
\begin{aligned}
 E_{\mathrm{phy}}(1)^2
   &=h^2\sum_{i>r}\lambda_i^2,\\
 E_{\mathrm{spec}}(1)^2
   &=h^2\|\operatorname{offdiag}(\mathbf B_F\mathbf R\mathbf B_F^\top)\|_F^2,\\[2pt]
 E_{\mathrm{phy}}(k)^2
   &=\sum_{i>r}[1-(1-h\lambda_i)^k]^2\longrightarrow n-r,\\
 E_{\mathrm{spec}}(k)^2
   &\le4n\chi^{2k}\longrightarrow0.
\end{aligned}
\label{eq:complement-risks}
\end{equation}
Consequently, if
$\sum_{i>r}\lambda_i^2<
\|\operatorname{offdiag}(\mathbf B_F\mathbf R\mathbf B_F^\top)\|_F^2$,
the physical minimizer is strictly better at one step, but the spectral
minimizer is strictly better at every sufficiently large finite $k$.
Moreover, $\mathbf T\in\mathcal C_{\mathrm{joint}}$.
\end{theorem}

\begin{proof}
Isotropic covariance turns each expected squared error into a squared
Frobenius norm. For any rank-at-most-$r$ matrix $\mathbf K$, let
$\mathbf P_K$ project onto its row space. Since $\mathbf K(\mathbf I-\mathbf P_K)=0$,
\[
 \|\mathbf K+h\mathbf L\|_F^2
 \ge h^2\operatorname{tr}(\mathbf L^2(\mathbf I-\mathbf P_K))
 \ge h^2\sum_{i>r}\lambda_i^2.
\]
For the last inequality, in the eigenbasis of $\mathbf L$ the diagonal
entries of $\mathbf P_K$ lie in $[0,1]$ and sum to at most $r$; their
weighted sum with weights $\lambda_i^2$ is at most
$\sum_{i=1}^r\lambda_i^2$. The displayed physical minimizer attains the
bound. Orthogonal projection onto the Fourier-diagonal matrices yields
the spectral minimizer and its off-diagonal residual.
Since every diagonal entry of $\mathbf B_F\mathbf L\mathbf B_F^\top$ lies
in $[\mu,\Lambda]$, $\|\mathbf T_{\mathrm{spec}}\|_2\le\chi$.
The physical minimizer commutes with $\mathbf T$ and acts as the identity
on the omitted eigenvectors, giving its exact $k$-step error.
The triangle inequality gives
$\|\mathbf T_{\mathrm{spec}}^k-\mathbf T^k\|_F\le2\sqrt n\chi^k$,
without assuming these two matrices commute. Their limits establish the
strict eventual reversal. Finally,
$\mathbf T=(\mathbf I-h\mathbf L_0)-h\mathbf R$ belongs to the joint class.
\end{proof}

The theorem also gives a computable sufficient horizon. Since
$0\le1-h\lambda_i\le\chi$, \eqref{eq:complement-risks} implies
\[
\begin{alignedat}{2}
 E_{\mathrm{phy}}(k)&\ge\sqrt{n-r}(1-\chi^k),\qquad
 &E_{\mathrm{spec}}(k)&\le2\sqrt n\,\chi^k.
\end{alignedat}
\]
Thus $E_{\mathrm{spec}}(k)<E_{\mathrm{phy}}(k)$ whenever
\[
 \chi^k<\frac{\sqrt{n-r}}{2\sqrt n+\sqrt{n-r}}.
\]
This is a sufficient bound, not the earliest crossing time.

The identity-plus-low-rank fit captures the strongest directions but
omits weaker damping: its multiplier is $1$ rather than $1-h\lambda_i$
on each omitted direction. The spectral fit retains damping but misses
cross-mode coupling; the joint class represents both. Fixing the identity
path is essential. For example, allowing a scalar local decay admits
\[
 \widetilde{\mathbf T}_{\mathrm{phy}}
 =(1-h\mu)\mathbf I
 -h\sum_{i=1}^r(\lambda_i-\mu)\mathbf v_i\mathbf v_i^\top,
\]
whose omitted-direction multiplier is $\chi<1$. The persistent identity
modes therefore follow from the stated class constraint, rather than
from physical-state modeling in general.

\paragraph{Persistent excitation}
The unforced result isolates dissipative omission. Its vanishing spectral
error is an absolute-error statement with decaying reference states,
not a claim of vanishing relative error. The same omission has a distinct
consequence under a sustained external input.

\begin{complementcorollary}[Omitted damping under a common forcing]
Keep the matrices from Theorem~\ref{thm:complement-ranking-reversal},
initialize all states at zero, and add the same known forcing
$\sigma\mathbf v_i$, $\sigma>0$, $i>r$, at every call.
Writing $E_m^{\mathrm{f}}(k)=\|\mathbf u_k^m-\mathbf u_k\|$, one has
\begin{equation}
\begin{aligned}
 E_{\mathrm{phy}}^{\mathrm{f}}(k)
   &=\sigma\left[k-\frac{1-(1-h\lambda_i)^k}{h\lambda_i}\right]
     \longrightarrow\infty,\\
 E_{\mathrm{spec}}^{\mathrm{f}}(k)&\le\frac{2\sigma}{1-\chi}.
\end{aligned}
\label{eq:complement-forced-omission}
\end{equation}
Consequently the spectral predictor has smaller error for all sufficiently
large finite $k$; the exact joint representation reproduces the forced
trajectory.
\end{complementcorollary}
\begin{proof}
On $\mathbf v_i$, the physical and exact responses are respectively
$k\sigma\mathbf v_i$ and
$\sigma[1-(1-h\lambda_i)^k](h\lambda_i)^{-1}\mathbf v_i$.
Their difference gives the first equality. Contractivity bounds the exact
and spectral state norms by $\sigma/(1-\chi)$, so their error is at most
twice this value. The diverging physical error eventually exceeds that
bound. Using the exact matrix and the same forcing gives the joint claim.
\end{proof}

This forced comparison establishes eventual superiority over the omitted
damping mode; it does not assert a strict one-call reversal from the zero
initial state, where all predictors output the same forcing.

\paragraph{Approximate joint representation}
Exact representability is not needed for controlled accumulation.
If $\|\widehat{\mathbf T}_{\mathrm{joint}}-\mathbf T\|_2\le\epsilon$
and $\epsilon<1-\chi$, then
$\|\widehat{\mathbf T}_{\mathrm{joint}}\|_2\le\chi+\epsilon<1$.
For exact states bounded by $R_0$, a common initial state, and identical
known additive forcing in both updates, subtraction gives
\begin{equation}
\begin{alignedat}{2}
 &\|\mathbf e_k^{\mathrm{joint}}\|
   \le\epsilon R_0 G_k(\chi+\epsilon),\qquad
 &G_k(a)&=\sum_{j=0}^{k-1}a^j.
\end{alignedat}
 \label{eq:complement-approximate-joint}
\end{equation}
Indeed, $\mathbf e_{k+1}=\widehat{\mathbf T}_{\mathrm{joint}}\mathbf e_k+
(\widehat{\mathbf T}_{\mathrm{joint}}-\mathbf T)\mathbf u_k$; induction
proves the bound and the uniform error envelope
\[
 \sup_{k\ge0}\|\mathbf e_k^{\mathrm{joint}}\|
 \le\frac{\epsilon R_0}{1-\chi-\epsilon}.
\]
This does not require convergence of the error sequence. Spectral
truncation, axis factorization, finite branch width, and imperfect fitting
must all be included in $\epsilon$ when using this approximation model.

A concrete nonzero-error instance is given in real Fourier coordinates by
\[
\begin{aligned}
 \mathbf L_0&=\operatorname{diag}(0.08,0.10,0.12,0.14),\\
 \mathbf U&=\tfrac12(1,1,1,1)^\top,\qquad
 \mathbf C=0.4,\quad h=1,\quad r=1.
\end{aligned}
\]
\begingroup
\predisplaypenalty=0
Use the minimizers in \eqref{eq:complement-optimal-cores} and
$\widehat{\mathbf T}_{\mathrm{joint}}
=\mathbf I-\mathbf L_0-0.39\mathbf U\mathbf U^\top$.
Here $\epsilon=0.01$ and one may choose $\mu=0.08$, $\Lambda=0.54$,
so the contraction condition holds. Direct matrix powers yield the
following expected squared errors:
\[
\begin{array}{@{}crrr@{}}
 \toprule
 k&\mathrm{physical}&\mathrm{spectral}&\mathrm{joint}\\\midrule
 1&0.0370&0.1200&1.00\times10^{-4}\\
 5&0.5853&0.2701&9.93\times10^{-6}\\
 10&1.4030&0.1611&4.51\times10^{-7}\\
 \bottomrule
\end{array}
\]
\endgroup
These are analytical example values, not benchmark measurements. Both
single-branch predictors minimize their own one-step objective, the joint
error is nonzero, and the spectral and exact matrices do not commute.
The first reversal occurs at call $k=4$: the physical and spectral squared
errors are $0.4183$ and $0.2818$, respectively. Thus the ordering changes
within a short finite horizon, well before the asymptotic limit.

\paragraph{Transfer to nonlinear learned recomposition}
Theorem~\ref{thm:complement-ranking-reversal} establishes complementary
representation in a controlled matrix model. We next quantify how much
approximation error an ideal joint advantage can tolerate. Let
$f_{\star,k}$ be a specified ideal joint map; in the linear construction,
it is $f_\star(u)=\mathbf T u$, or $\mathbf T u$ plus the common forcing.
The actual architecture approximates a joint map through latent-channel
recomposition inside each block, rather than through an output ensemble.

\begin{proposition}[Finite-horizon tolerance to recomposition error]
\label{prop:complement-realization}
Suppose $f_{\star,k}$ is $a$-Lipschitz, $a\ge0$, on a common
rollout domain containing the ideal and actual joint states through
$K$ calls. Suppose also that the actual joint map satisfies
$\|f_{\mathrm{joint},k}(u)-f_{\star,k}(u)\|\le\eta$
on that domain, for every $0\le k<K$.
All trajectories start from the same exact state and use the same
supplied conditions. Let $E_m(k)$ denote the actual state-error norm
for $m\in\{\mathrm{phy},\mathrm{spec},\star,\mathrm{joint}\}$,
and put
\[
 \Delta_k=\min\{E_{\mathrm{phy}}(k),E_{\mathrm{spec}}(k)\}
                 -E_\star(k).
\]
For any prescribed set of calls $\mathcal I\subseteq\{1,\ldots,K\}$,
if $\Delta_k>\eta G_k(a)$ for all $k\in\mathcal I$, then
\begin{equation}
 E_{\mathrm{joint}}(k)
 <\min\{E_{\mathrm{phy}}(k),E_{\mathrm{spec}}(k)\},
 \qquad k\in\mathcal I.
 \label{eq:complement-strict-joint}
\end{equation}
The same statement holds for root-mean-square errors over a common
initial-state distribution when the assumptions hold uniformly, the
errors have finite second moments, and $E_m(k)$ and $\Delta_k$ use
that root-mean-square norm.
\end{proposition}

\begin{proof}
Let $d_k$ be the distance between the joint and ideal trajectories.
Adding and subtracting $f_{\star,k}$ at the actual joint input gives
$d_{k+1}\le\eta+a d_k$, with $d_0=0$.
Hence $d_k\le\eta G_k(a)$ and
$E_{\mathrm{joint}}(k)\le E_\star(k)+\eta G_k(a)$.
The margin assumption proves the strict ordering; Minkowski's inequality
gives the distributional version. The argument uses only the ideal
map's Lipschitz constant, not a derivative bound on the approximation
residual.
\end{proof}

In the linear construction, $E_\star(k)=0$, and the representation classes
establish realizability independently of the comparison margin. On a
bounded rollout domain, sufficiently small nonzero realization error
retains any positive finite collection of ideal margins. For the actual
network, this approximation condition must hold at the allocated
half-width branch budgets and across visited states. No contraction or
derivative bound on the realization residual is needed for this transfer
result, although large $G_k(a)$ reduces its tolerance.

Another possible ideal map selects complementary state components:
\[
 f_{\star,k}(u)=\mathbf Q_{\mathrm{mix}}f_{\mathrm{spec},k}(u)
       +(\mathbf I-\mathbf Q_{\mathrm{mix}})f_{\mathrm{phy},k}(u),
\]
where $\mathbf Q_{\mathrm{mix}}$ is a fixed orthogonal projection chosen
from the dynamics or representation structure before comparing errors.
Its true-input defect selects corresponding components of the branch
defects. Subsequent propagation still determines whether the required
rollout margin is positive. This projection acts on the predicted state,
not the latent channel halves, and is an optional comparison construction;
it need not equal the diagonal-plus-low-rank map in the theorem.

\paragraph{How a physical contribution can reduce injection without excessive gain}
A complementary description uses a fixed-parameter intervention: let
$\mathcal S_k$ be the same joint network with every physical-operator
response $\Phi_{\mathrm{phy}}^{(\ell)}(\cdot)$ set to zero, retaining the
outer identity paths, recomposition layers, and all other parameters.
Define $\mathcal C_k=f_{\mathrm{joint},k}-\mathcal S_k$.
This is an exact functional decomposition through the complete nonlinear
network; $\mathcal S_k$ is not the separately trained Specsolver.
For the exact target $y_k$ at a true input $u_k$, put
$r_k=y_k-\mathcal S_k(u_k)$ and $c_k=\mathcal C_k(u_k)$.
Then
\begin{equation}
 \|f_{\mathrm{joint},k}(u_k)-y_k\|^2
 =\|r_k\|^2-2\langle r_k,c_k\rangle+\|c_k\|^2.
 \label{eq:complement-marginal-alignment}
\end{equation}
Thus $2\langle r_k,c_k\rangle>\|c_k\|^2$ is exactly the condition
for reducing that input's squared defect relative to the intervention.
If, on a common convex rollout domain containing the exact and joint
trajectories, the maps are continuously differentiable and
$\|D\mathcal S_k\|_2\le a_0$, $\|D\mathcal C_k\|_2\le a_c$,
then the joint gain is at most $a_0+a_c$.
Together with a uniform true-input defect bound $\delta$, this yields
$E_{\mathrm{joint}}(k)\le\delta G_k(a_0+a_c)$ from exact initialization.
It identifies two separate requirements: useful correction alignment and
controlled sensitivity of that correction. Comparing this upper bound
with another model's upper bound does not prove actual superiority;
\eqref{eq:complement-strict-joint} instead uses a genuine error margin.
Nor need a more accurate joint model have a smaller Jacobian norm than
the spectral model: reducing defects in persistent directions can suffice.

\paragraph{Cross-component propagation and multi-frame deployment}
The directional form of Section~\ref{app:error-propagation} is obtained
by unrolling its exact integrated-Jacobian recurrence:
\[
 e_k=\sum_{j=0}^{k-1}
       (\overline J_{k-1}\cdots\overline J_{j+1})\delta_j,
 \qquad e_0=0.
\]
Empty products are identities, and the product order is chronological
from right to left. For any fixed orthogonal state-space projection
$\mathbf Q_{\mathrm{mix}}$, put $\mathbf Q_1=\mathbf Q_{\mathrm{mix}}$ and
$\mathbf Q_2=\mathbf I-\mathbf Q_{\mathrm{mix}}$. Suppose
$\|\mathbf Q_i\overline J_k\mathbf Q_j\|_2\le A_{ij}$ and
$\|\mathbf Q_i\delta_k\|\le\epsilon_i$ uniformly along the compared
trajectory segments. Set
$x_k=(\|\mathbf Q_1e_k\|,\|\mathbf Q_2e_k\|)^\top$,
$\boldsymbol\epsilon=(\epsilon_1,\epsilon_2)^\top$, and
$\mathbf A=(A_{ij})\ge0$. Applying the triangle inequality to each
projected recurrence and then induction gives the componentwise bound
\begin{equation}
 x_k\le\sum_{j=0}^{k-1}\mathbf A^j\boldsymbol\epsilon.
 \label{eq:complement-component-bound}
\end{equation}
This finite-horizon result permits amplification and cross-component
transfer. Under the stronger conditions
\[
 A_{11}<1,\qquad A_{22}<1,\qquad
 A_{12}A_{21}<(1-A_{11})(1-A_{22}),
\]
the nonnegative $2\times2$ matrix has spectral radius below one:
its largest eigenvalue is
\[
 \tfrac12(A_{11}+A_{22}+
 \sqrt{(A_{11}-A_{22})^2+4A_{12}A_{21}})<1.
\]
Consequently $x_k\le(\mathbf I-\mathbf A)^{-1}\boldsymbol\epsilon$
for all $k$ if the bounds persist. This controls an error envelope, not a
claim that the nonlinear error trajectory converges.

For $T_{\mathrm{in}}>1$ or $T_{\mathrm{out}}>1$, use the complete
history-update maps defined in Section~\ref{app:error-propagation}.
Here $k$ counts model calls, not forecast frames. The nonlinear transfer
and component bounds then control history error; omitted output frames
require the separate block-output bound
\eqref{eq:rollout-block-output-bound}, or its corresponding joint--ideal
comparison. The single-frame dissipative matrix model is not assumed
for that full history shift.

These results separate representational complementarity from guaranteed
training outcomes. The linear model derives distinct omissions and a
strict reversal; the realization margin quantifies how much imperfect
recomposition that advantage can tolerate. The alignment and component
bounds specify sufficient conditions on the complete learned maps.
They can be examined using matched true inputs, prescribed perturbations,
and fixed-parameter branch interventions. The existing decomposition in
Appendix~\ref{app:exact-error-decomposition} motivates this analysis but
does not measure all of these conditions. All new errors here are
Euclidean or root-mean-square quantities, not an assertion of identical
ordering under every relative-$L_2$ benchmark aggregation.
\endgroup

\section{Implementation Details}
\label{app:experiments}

This appendix complements Section~\ref{sec:experimental-setup} with task
protocols, metric definitions, available training settings, and result sources.
Shared benchmark conventions are separated from the measurement protocols
for individual analyses in Appendices~\ref{app:full-ablations},
\ref{app:supplementary-results}, and~\ref{app:efficiency-measurements}.
We distinguish the experimental groups below to keep each configuration
within its evaluated setting.

\paragraph{Experimental groups and comparison scope}
Table~\ref{tab:pde-main-results} reports the main benchmark comparison,
including three-run means and standard deviations for SpecTransolver.
Table~\ref{tab:ablation-components} studies component interventions within
the proposed architecture; Table~\ref{tab:slice-attention-analysis} instead
compares slice-update rules in Transolver and an independent LinearNO baseline.
Both ablation tables report means over three independently trained runs
per configuration, using the final checkpoint from each run.
Their results do not represent repeated evaluations of the same checkpoint.

Figure~\ref{fig:rollout-comparison} evaluates finite-horizon prediction
curves, while Appendix~\ref{app:exact-error-decomposition} separates
error injection and propagation on the specified evaluation samples.
Appendix~\ref{app:efficiency-measurements} measures training costs under
dedicated workloads and configurations.
Neither diagnostic sample variation nor runtime repetitions supply the
training-run standard deviations in Table~\ref{tab:pde-main-results}.

\subsection{Benchmarks}
\label{app:experiment-protocols}

We evaluate Transolver-$\sigma$ on eleven tasks spanning canonical PDEs,
experimental measurements from RealPDEBench, and coupled reactive flows from REALM.
Table~\ref{tab:benchmark-overview} summarizes their physical fields and spatial
representations; Table~\ref{tab:forecast-protocols} lists the canonical data
splits and forecasting settings.
Darcy is a steady-state prediction task, while the remaining benchmarks
involve time-dependent fields.
For forecasting, $T_{\mathrm{in}}$ denotes the input history length,
$T_{\mathrm{out}}$ the number of frames predicted per model call,
and $H$ the total prediction horizon, all measured in frames.
Details follow for each benchmark.

\begin{table}[!b]
  \centering
  \caption{\textbf{Overview of the evaluated benchmarks.}
  Dimensions are spatial; canonical forecasting configurations are summarized
  in Table~\ref{tab:forecast-protocols}.}
  \label{tab:benchmark-overview}
  \small
  \setlength{\tabcolsep}{4pt}
  \renewcommand{\arraystretch}{1.14}
  \begin{tabular}{@{}llllll@{}}
    \toprule
    Benchmark & Data type & Spatial dim. & Grid & Input & Output \\
    \midrule
    \multicolumn{6}{@{}l}{\textit{Canonical PDE benchmarks}} \\
    Darcy & Simulation & 2D & $85^2$ & Coefficient $a$ & Pressure $p$ \\
    Navier--Stokes & Simulation & 2D & $64^2$ & $\omega$ history & Vorticity $\omega$ \\
    Kolmogorov flow & Simulation & 2D & $128^2$ & $\omega$ history & Vorticity $\omega$ \\
    Isotropic turbulence & Simulation & 3D & $60^3$ & $(\mathbf{u},p)$ history & $\mathbf{u},p$ \\
    Smoke buoyancy & Simulation & 3D & $64^3$ & $(\mathbf{u},d)$ history & $\mathbf{u},d$ \\
    \midrule
    \multicolumn{6}{@{}l}{\textit{RealPDEBench: experimental measurements}} \\
    Controlled Cylinder & Measurement & 2D & $128\!\times\!256$ & Velocity history & $u,v$ \\
    FSI & Measurement & 2D & $128^2$ & Velocity history & $u,v$ \\
    Foil & Measurement & 2D & $128\!\times\!256$ & Velocity history & $u,v$ \\
    Combustion & Measurement & 2D & $128^2$ & Intensity history & Intensity $I$ \\
    \midrule
    \multicolumn{6}{@{}l}{\textit{REALM: coupled multiphysics simulations}} \\
    IgnitHIT & Simulation & 2D & $128^2$ & Field history & Coupled fields \\
    EvolveJet & Simulation & 2D & $256^2$ & Field history & Coupled fields \\
    \bottomrule
  \end{tabular}
\end{table}

\begin{table}[!b]
  \centering
  \caption{\textbf{Canonical PDE benchmarks and forecasting configurations.}
  Train and test sizes count source sequences for time-dependent tasks and individual samples for
  Darcy, rather than extracted windows.
  $T_{\mathrm{in}}$, $T_{\mathrm{out}}$, and $H$ denote the input history length, frames predicted per model call, and total forecasting horizon, respectively, all measured in frames.
  $p$, $\omega$, $\mathbf{u}$, and $d$ denote pressure, vorticity, velocity, and smoke density.}
  \label{tab:forecast-protocols}
  \small
  \setlength{\tabcolsep}{3pt}
  \renewcommand{\arraystretch}{1.15}
  \begin{tabular*}{\linewidth}{@{\extracolsep{\fill}}llccccccc@{}}
    \toprule
    Task & Temporal type & Grid & Target fields & Train & Test
    & $T_{\mathrm{in}}$ & $T_{\mathrm{out}}$ & $H$ \\
    \midrule
    Darcy
    & Steady-state & $85^2$ & $p$
    & 1,000 & 200 & -- & -- & -- \\
    Navier--Stokes
    & Time-dependent & $64^2$ & $\omega$
    & 1,000 & 200 & 10 & 1 & 10 \\
    Kolmogorov flow
    & Time-dependent & $128^2$ & $\omega$
    & 100 & 20 & 10 & 4 & 16 \\
    Isotropic turbulence
    & Time-dependent & $60^3$ & $\mathbf{u},p$
    & 1,000 & 100 & 10 & 2 & 10 \\
    Smoke buoyancy
    & Time-dependent & $64^3$ & $\mathbf{u},d$
    & 2,000 & 200 & 4 & 4 & 16 \\
    \bottomrule
  \end{tabular*}
\end{table}

\paragraph{Darcy}
This benchmark predicts steady fluid pressure from the permeability of a
porous medium on the unit square~\citep{li2021fno}.
The underlying elliptic equation has a fixed forcing term and zero pressure
on the boundary; permeability varies between samples.
We use the data generated on a $421\times421$ grid and subsample it to
$85\times85$ for the main experiments.
The input and output are the permeability and pressure at each grid point,
with $1{,}000$ samples for training and $200$ for testing.

\paragraph{Navier--Stokes (NS2D)}
This task models two-dimensional incompressible flow on a periodic unit
square, using the vorticity formulation of the Navier--Stokes
equations~\citep{li2021fno}.
The viscosity is $10^{-5}$, and trajectories differ in their initial
vorticity fields.
Each field is represented on a $64\times64$ grid, and ten observed frames
are used to predict the following ten frames.
We use $1{,}000$ training and $200$ test trajectories, without normalizing
the input or output vorticity fields.

\paragraph{Kolmogorov Flow (KF2D)}
This benchmark describes periodically forced, two-dimensional turbulence
at Reynolds number $1000$, with periodic spatial
boundaries~\citep{li2023factformer}.
The prediction target is vorticity on a $128\times128$ grid.
We use $100$ training and $20$ test trajectories, each containing
$160$ frames over ten seconds.
Given ten observed frames, the model predicts four frames per call and
is evaluated over a sixteen-frame horizon.

\paragraph{Isotropic Turbulence (IT3D)}
This task considers three-dimensional incompressible turbulence with
periodic boundaries and Taylor Reynolds number $84$~\citep{li2023factformer}.
Each $60\times60\times60$ grid stores three velocity components and pressure.
The dataset contains $1{,}000$ training and $100$ test trajectories,
with twenty frames spanning one second.
The model receives ten frames and forecasts the next ten in two-frame blocks.

\paragraph{Smoke Buoyancy (Smoke3D)}
This benchmark couples incompressible flow with smoke transport, where
the density field produces an upward buoyancy force~\citep{li2023factformer}.
Each $64\times64\times64$ grid records velocity and smoke density.
We use $2{,}000$ training and $200$ test trajectories, with twenty frames
over fifteen seconds.
Four input frames yield sixteen forecast frames in four-frame blocks.

\paragraph{Controlled Cylinder}
This RealPDEBench task measures flow around a cylinder under periodic
external control~\citep{hu2026realpdebench}.
Reynolds number and control frequency vary across cases.
From measured velocity histories on a $128\times256$ grid, the model
predicts the subsequent two-component velocity field and its response to control.

\paragraph{Fluid--Structure Interaction (FSI)}
FSI records the coupling between fluid forces and the motion of a
vibrating cylinder~\citep{hu2026realpdebench}.
The cases vary Reynolds number, mass ratio, and structural damping.
The task forecasts both velocity components on a $128\times128$ grid
from observed histories.

\paragraph{Foil}
This benchmark provides measured cross-sections of three-dimensional foil
flows under different angles of attack and Reynolds
numbers~\citep{hu2026realpdebench}.
The planar observations contain wake structures influenced by three-dimensional
flow dynamics.
We forecast the two measured velocity components on a $128\times256$ grid.

\paragraph{Combustion}
This task uses chemiluminescence images of swirl-stabilized
ammonia/methane/air flames from RealPDEBench~\citep{hu2026realpdebench}.
The images provide a scalar intensity field describing the evolving flame.
We predict subsequent $128\times128$ intensity fields from observed images.

\paragraph{IgnitHIT}
IgnitHIT models ignition and flame growth in two-dimensional homogeneous
isotropic turbulence, using a premixed hydrogen--oxygen
mixture~\citep{mao2025benchmarking}.
Ignition geometry and turbulence vary across trajectories.
We forecast the coupled physical fields on a $128\times128$ grid.
The dataset contains thirty frames per trajectory, with
a $26/5/5$ training/validation/test split.

\paragraph{EvolveJet}
EvolveJet models premixed methane--oxygen jet flames under
strong shear that deforms and mixes reacting layers~\citep{mao2025benchmarking}.
We predict coupled fields on a $256\times256$ grid.
Each trajectory contains forty frames; the training/validation/test split is
fixed at $24/3/3$.

\paragraph{Data Preparation}
Dataset-level normalization parameters for Darcy, KF2D, IT3D, and Smoke3D are derived exclusively from their training splits and reused unchanged during testing. Predictions are mapped back to the corresponding target-field scale before error evaluation. For NS2D, the loader uses stored vorticity tensors without additional dataset-level field normalization.

RealPDEBench experiments follow its Real-world Training protocol, preserving
the published parameter-aware splits and Gaussian normalization.
The input history and output block each contain ten frames for Controlled
Cylinder and twenty for FSI, Foil, and Combustion.
Our results in Table~\ref{tab:realpdebench-results} evaluate one model call
over the complete output block, giving prediction horizons
$H=10,20,20,20$, respectively.
Predictions are not fed back for subsequent output blocks.
Sliding-window strides are $20$, $10$, $20$, and $1$ for Controlled Cylinder,
FSI, Foil, and Combustion, respectively.
Predictions are inverse-normalized before the benchmark metrics are computed.

For REALM, species mass fractions undergo a Box--Cox transform with exponent
$0.1$, followed by channelwise standardization using training statistics.

\subsection{Metrics}
\label{app:experiment-metrics}

We evaluate field prediction using relative $L^2$ error, supplemented by
the benchmark-specific error and correlation measures for RealPDEBench and REALM.
The definitions and reporting conventions are summarized below.

\paragraph{Relative $L^2$ Error for Physical Fields}
Let $\mathbf{y}_i$ and $\widehat{\mathbf{y}}_i$ denote the reference and
predicted fields for test sample $i$.
The reported relative error averages the sample-wise norm ratios:
\[
 E_{\mathrm{rel}}=\frac{1}{n}\sum_{i=1}^{n}
 \frac{\|\widehat{\mathbf{y}}_i-\mathbf{y}_i\|_2}{\|\mathbf{y}_i\|_2},
\]
where $n$ is the number of evaluated samples.
For Darcy, each vector contains a spatial pressure field.
For the dynamic canonical tasks, it contains one field component over
all spatial points and forecast frames.
Thus, NS2D is evaluated over each complete ten-frame prediction before
averaging across samples; the ``avg.'' columns likewise use the full forecast.
The ``final'' column instead evaluates only the last predicted frame.
For IT3D and Smoke3D, velocity error averages the three component-wise
relative errors; pressure and density are evaluated separately.
Canonical errors are evaluated in physical units after inverse normalization
where applicable. They are reported as percentages, with lower values
indicating more accurate predictions.

\paragraph{RMSE and Fourier-Space Error}
RealPDEBench additionally reports root mean squared error (RMSE) and
Fourier-space RMSE (fRMSE)~\citep{hu2026realpdebench}.
RMSE measures the magnitude of field differences, whereas fRMSE measures
discrepancies in their frequency representation under the benchmark convention.
Its relative $L^2$ error combines the evaluated channels and forecast
frames within each sample before averaging across samples.
All three metrics in Table~\ref{tab:realpdebench-results} are multiplied by
$100$; only relative $L^2$ thereby represents a percentage.

\paragraph{REALM Error}
We follow the official REALM evaluation protocol in the normalized field
space~\citep{mao2025benchmarking}.
Squared errors are averaged within each available physical group and
summed across groups, including species, temperature, density, velocity,
and pressure where present.
The train, validation, and test errors retain the benchmark's native scale.

\paragraph{Pearson Correlation for Physical Fields}
REALM also assesses spatial agreement using the Pearson correlation between
predicted and reference physical fields~\citep{mao2025benchmarking}:
\[
 r=\frac{\sum_j(y_j-\bar y)(\widehat y_j-\overline{\widehat y})}
 {\sqrt{\sum_j(y_j-\bar y)^2}
  \sqrt{\sum_j(\widehat y_j-\overline{\widehat y})^2}},
\]
where $j$ indexes spatial points and the bars denote spatial means.
Correlations are evaluated on decoded physical fields and averaged over
channels and rollout steps.
Table~\ref{tab:realm-results} reports $100r$, with higher values indicating
closer spatial correspondence.

\paragraph{Autoregressive Rollout Evaluation}
\label{app:rollout-comparison}
Figure~\ref{fig:rollout-comparison} evaluates physical-only, spectral-only,
and joint models on identical test inputs within each task.
It uses $200$ NS2D trajectories, $400$ KF2D windows from $20$ sequences,
and $20$ trajectories each for IT3D and Smoke3D.
Block lengths and horizons follow Table~\ref{tab:forecast-protocols}.
After the observed history, complete predicted blocks are fed back
without access to future ground truth.
The horizontal axis counts forecast frames, rather than model calls.

At each lead, we average sample-wise relative $L^2$ errors to obtain
the plotted error curve.
Reported rollout reductions compare the temporal mean of this curve
against the lower temporal mean of the two single-domain references.
This average of per-frame errors differs from the whole-sequence norm
used for the main benchmark scores.

\subsection{Implementations}
\label{app:implementations}

\paragraph{Model and Training Configurations}
\label{app:experiment-implementation}
Transolver-$\sigma$ is implemented in PyTorch, and the main benchmark
experiments are conducted on a single NVIDIA A100 GPU. 
Table~\ref{tab:training-settings} summarizes the training budgets and model
configurations for the canonical and external benchmarks. The Darcy, KF2D, IT3D, and Smoke3D epoch budgets match the released
FactFormer configurations~\citep{li2023factformer}; NS2D uses $500$ epochs.
For REALM, the matched FFNO settings use two-step training rollouts,
with gradients through the final step, and a OneCycle schedule with maximum
learning rate $10^{-3}$. Batch sizes are listed in Table~\ref{tab:training-settings}.
For the main results of our model, we train with three independent seeds,
$42$, $43$, and $44$, and evaluate the final checkpoint from each run.
The component and slice-update comparisons also use final checkpoints and
report three-run means. Appendix~\ref{app:standard-deviations} provides the
means and standard deviations for the main benchmark results.

\begin{table}[!htbp]
  \centering
  \caption{\textbf{Training and model configurations of Transolver-$\sigma$.}
  Canonical training configurations like the AdamW optimizer~\citep{loshchilov2019iclr-decoupled} are shared by the baselines trained in this study.
  $C$, $L$, $h$, and $M$ denote latent width, depth, attention heads,
  and physical slices, respectively. The spectral and physical subspaces each use
  half of $C$. $K$ denotes the retained Fourier modes.}
  \label{tab:training-settings}
  \label{app:external-configurations}
  \small
  \setlength{\tabcolsep}{3pt}
  \renewcommand{\arraystretch}{1.12}
  \begin{tabular}{@{}l|cccc|ccccc@{}}
    \toprule
    \multirow{3}{*}{Benchmark} & \multicolumn{4}{c|}{Training Configuration} & \multicolumn{5}{c}{Model Configuration} \\
    \cmidrule(lr){2-5}\cmidrule(lr){6-10}
    & \multirow{2}{*}{Budget} & \multirow{2}{*}{Batch} & \multirow{2}{*}{Peak LR} & \multirow{2}{*}{Optimizer} & Width & Depth & Heads & Slices & Modes \\
    & & & & & $C$ & $L$ & $h$ & $M$ & $K$ \\
    \midrule
    \multicolumn{10}{@{}l}{\textit{Canonical PDE benchmarks}} \\
    Darcy & 500 & 4 & $10^{-3}$ & AdamW & 128 & 8 & 8 & 32 & 4 \\
    NS2D & 500 & 2 & $10^{-3}$ & AdamW & 256 & 8 & 8 & 32 & 4 \\
    KF2D & 50 & 8 & $10^{-3}$ & AdamW & 128 & 8 & 8 & 32 & 8 \\
    IT3D & 160 & 2 & $10^{-3}$ & AdamW & 128 & 8 & 8 & 32 & 6 \\
    Smoke3D & 100 & 1 & $10^{-3}$ & AdamW & 128 & 8 & 8 & 32 & 6 \\
    \midrule
    \multicolumn{10}{@{}l}{\textit{RealPDEBench}} \\
    Controlled Cylinder & 20k iters & 4 & $10^{-3}$ & AdamW & 128 & 8 & 8 & 32 & 8 \\
    FSI & 20k iters & 4 & $10^{-3}$ & AdamW & 128 & 8 & 8 & 32 & 8 \\
    Foil & 20k iters & 4 & $10^{-3}$ & AdamW & 128 & 8 & 8 & 32 & 8 \\
    Combustion & 20k iters & 4 & $10^{-3}$ & AdamW & 128 & 8 & 8 & 32 & 8 \\
    \midrule
    \multicolumn{10}{@{}l}{\textit{REALM}} \\
    IgnitHIT & 20k iters & 26 & $10^{-3}$ & AdamW & 128 & 4 & 8 & 32 & 32 \\
    EvolveJet & 20k iters & 12 & $10^{-3}$ & AdamW & 256 & 6 & 8 & 32 & 48 \\
    \bottomrule
  \end{tabular}
  \par\smallskip
  \begin{minipage}{\linewidth}
    \footnotesize
    Canonical budgets are measured in epochs. RealPDEBench uses 20,000 optimizer updates;
    REALM budgets are reported in iterations.
    REALM follows the dataset-specific FFNO-M/FFNO-L configurations in
    \citet{mao2025benchmarking}, Tables S.7 and S.9; Peak LR denotes the maximum learning rate of the OneCycle schedule.
  \end{minipage}
\end{table}

\paragraph{Baseline Sources}
\label{app:experiment-provenance}
For Table~\ref{tab:pde-main-results}, the FNO and F-FNO results on
Darcy and NS2D follow the comparison in \citet{wu2024Transolver};
those on KF2D, IT3D, and Smoke3D are taken from \citet{li2023factformer}. We train all remaining baselines, including
FactFormer, using the same task-specific training configurations as our
model. The imported FNO and F-FNO results retain their source protocols.
For RealPDEBench (Table~\ref{tab:realpdebench-results}) and REALM
(Table~\ref{tab:realm-results}), we retain the baseline results from the
official benchmark comparisons and add our model's results.
Specifically, the RealPDEBench entries come from the Real-world Training
category of \citet{hu2026realpdebench}, excluding the aggregate
``ML Average'' row; the REALM entries for IgnitHIT and EvolveJet come from
\citet{mao2025benchmarking}.
The suffix FT denotes the imported fine-tuned DPOT variants.

\paragraph{Architecture and Operator Details}
\label{app:method-implementation}

This appendix expands the architecture in
Section~\ref{sec:method-decomposition}, the physical update in
Section~\ref{sec:method-srpa}, and the spectral operator in
Section~\ref{sec:method-spectral}.
Batch dimensions are omitted throughout the equations, and spatial
convolutions operate on grid-shaped features before returning to pointwise
notation.

\paragraph{Input Encoding and Recomposition}
\label{app:encoding-recomposition}
The input encoder concatenates field values with the configured positional
representation, then applies a pointwise multilayer perceptron.
A learned channel offset is broadcast to all grid points, and a time
embedding is added when time conditioning is enabled.
Within each block, CPE uses a depthwise convolution of kernel size three
and padding one, preserving the spatial shape.
Each channel has its own convolution kernel.
The shared CPE stage supplies spatial context before splitting and leaves
channel mixing to the subsequent operators and recomposition.
The spectral channels precede the physical channels in both the split and
concatenation operations.
SwiGLU applies two affine maps to the normalized concatenated features,
followed by elementwise gating and an output map:
\begin{equation}
  \operatorname{SwiGLU}(\mathbf{H})
  =\left[
  \operatorname{SiLU}(\mathbf{H}\mathbf{W}_g+\mathbf{b}_g)
  \odot(\mathbf{H}\mathbf{W}_u+\mathbf{b}_u)
  \right]\mathbf{W}_d+\mathbf{b}_d.
  \label{eq:swiglu-implementation}
\end{equation}
Here, biases are broadcast over points, $\odot$ is elementwise multiplication,
and the intermediate width is
$C_f=\lfloor 2\rho C/3\rfloor$ for the configured expansion ratio $\rho$.
The input matrices $\mathbf{W}_g,\mathbf{W}_u\in\mathbb{R}^{C\times C_f}$
and output matrix $\mathbf{W}_d\in\mathbb{R}^{C_f\times C}$ mix channels at
each grid point.

\paragraph{SRPA Implementation}
\label{app:srpa-implementation}
The total hidden width $C$ is divisible by $2n_h$, giving head width
$d_h=C_{\mathrm{phy}}/n_h$ with $C_{\mathrm{phy}}=C/2$.
The routing and feature projections use separate depthwise convolutions of
kernel size five and padding two, each followed by a pointwise projection
of kernel size one.
Both projections act on the full normalized physical subspace
$\mathbf{Y}_{\mathrm{phy}}$ before their outputs are partitioned into heads.
The routing affine map and the query, key, and value projections share
parameters across heads within each block.
In Equation~\ref{eq:physical-slicing},
$\mathbf{W}_s\in\mathbb{R}^{d_h\times M}$ and
$\mathbf{b}_s\in\mathbb{R}^{M}$ parameterize the routing scores.
Here, $\mathbf{1}_q$ denotes a length-$q$ vector of ones, so
$\mathbf{1}_N\mathbf{b}_s^{\top}$ broadcasts the bias across points.
The routing map includes a bias, while the query, key, and value maps are
bias-free; the output projection includes a bias.
The effective temperature for head $r$ is
$\tau_r=\operatorname{clip}(\widetilde{\tau}_r,0.1,5)$, with
$\widetilde{\tau}_r$ initialized to $0.5$ and learned independently for each
head.
We use $\varepsilon=10^{-5}$ in the slice normalization to avoid dividing
by vanishing slice mass.
Each block's residual scalar $\gamma_\ell$ is initialized to zero and
remains unconstrained during training, with the same learned value applied
to every input and all heads and channels in that block.
The unscaled attention response is $\Delta\mathbf{Z}$, and the residual
correction is $\gamma_\ell\Delta\mathbf{Z}$, as in
Equation~\ref{eq:srpa-update}.
This update adds one scalar and $\mathcal{O}(MC_{\mathrm{phy}})$ arithmetic per block,
leaving the asymptotic attention cost unchanged.
The main equations omit optional dropout after the attention softmax and
output projection.
All conditional mixing bounds concern the deterministic attention matrix
before dropout, or equivalently the attention used during inference.

\paragraph{Axis-Factorized Spectral Operator}
\label{app:spectral-operator}

The operator in Section~\ref{sec:method-spectral} applies parallel
one-dimensional transforms to a shared input.
For a $d$-dimensional domain discretized into $N=\prod_{a=1}^{d} n_a$ points,
the normalized spectral features
$\mathbf{Y}_{\mathrm{spec}}=\operatorname{LN}_{\mathrm{spec}}(\mathbf{H}_{\mathrm{spec}})\in\mathbb{R}^{N\times C_{\mathrm{spec}}}$
are reshaped into
$\mathbf{V}\in\mathbb{R}^{n_1\times\cdots\times n_d\times C_{\mathrm{spec}}}$.

Each axis uses the same normalized spectral input.
For each spatial axis $a\in\{1,\ldots,d\}$, we apply an orthonormally
normalized one-dimensional real Fourier transform:
\begin{equation}
  \widehat{\mathbf{V}}_a
  =
  \operatorname{rFFT}_a(\mathbf{V}).
  \label{eq:spectral-axis-transform}
\end{equation}
Only the lowest $K_a$ modes are retained and transformed using learnable complex channel-mixing matrices
$\mathbf{W}_a(k)\in\mathbb{C}^{C_{\mathrm{spec}}\times C_{\mathrm{spec}}}$:
\begin{equation}
  \widehat{\mathbf{R}}_a(k)
  =
  \begin{cases}
    \widehat{\mathbf{V}}_a(k)\mathbf{W}_a(k),
    & 0\leq k<K_a,\\
    \mathbf{0},
    & \text{otherwise}.
  \end{cases}
  \label{eq:spectral-mode-filter}
\end{equation}
The same channel matrix $\mathbf{W}_a(k)$ is applied at every location
along the remaining spatial coordinates.
The effective retained count is capped at $\lfloor n_a/2\rfloor+1$
when the configured mode budget exceeds the available one-sided spectrum.
The filtered spectrum is mapped back using orthonormally normalized irFFT,
with the inverse transform restoring the original length $n_a$:
\begin{equation}
  \mathbf{R}_a
  =
  \operatorname{irFFT}_a(\widehat{\mathbf{R}}_a).
  \label{eq:spectral-axis-inverse}
\end{equation}
Finally, the axis-wise responses are aggregated and flattened to recover the pointwise representation:
\begin{equation}
  \Phi_{\mathrm{spec}}(\mathbf{Y}_{\mathrm{spec}})
  =
  \operatorname{Flatten}\!\left(
  \sum_{a=1}^{d}\mathbf{R}_a
  \right)
  \in\mathbb{R}^{N\times C_{\mathrm{spec}}}.
  \label{eq:spectral-aggregation}
\end{equation}

The truncation level $K_a$ controls the number of retained frequency modes along each axis, with all
higher-frequency coefficients set to zero. Unless otherwise specified, we use the same mode budget
across axes, i.e., $K_a=K$. Compared with a full $d$-dimensional Fourier operator with parameter
complexity $\mathcal{O}(C_{\mathrm{spec}}^2K^d)$, axis factorization reduces the complexity to
$\mathcal{O}(dC_{\mathrm{spec}}^2K)$ per layer, making repeated spectral updates practical in deep Transolver-$\sigma$
backbones.
The current implementation operates on structured grids, as required by
its convolutions and axis-wise FFTs.
Mode truncation specifies the retained spectral representation; it does
not constrain learned multipliers to be contractive.
The branch residual, physical update, and nonlinear recomposition must also
be considered when interpreting complete-model rollout behavior.

\section{Full Ablations}
\label{app:full-ablations}

\paragraph{Component Ablations}
\label{app:component-ablations}
Table~\ref{tab:ablation-components} reports the component ablations.
All reported errors are relative $L^2$ errors multiplied by $100$;
the evaluated outputs are Darcy pressure, NS2D vorticity, and Smoke3D
velocity and density.
For each task, the variants are constructed by changing the designated
components of our model while keeping the data split, training budget,
optimizer settings, and checkpoint-selection rule fixed.
Each entry is the arithmetic mean of three training runs, evaluated at their
final checkpoints with the same task-specific error metric.
The interventions test different aspects of block organization, as follows.

\begin{itemize}
\item \textbf{Branch removal.} The physical or spectral update is replaced
by an identity mapping in the corresponding channel subspace.
The remaining backbone is retained; this differs from training a standalone
full-width physical-only or spectral-only solver.
\item \textbf{Sequential hybrids.} Physical-then-spectral and
spectral-then-physical updates replace the parallel organization,
testing the order and placement of the two operator types.
\item \textbf{Shared spatial context.} Removing the shared CPE tests its
contribution before subspace-specific updates.
\item \textbf{Cross-branch mixing.} Two separate grouped SwiGLU mappings
replace full-channel recomposition, retaining within-branch transformations
while preventing that FFN from mixing the two channel groups.
\item \textbf{Slice update.} Replacing SRPA with Physics-Attention tests
the slice-state update within the complete joint architecture.
\end{itemize}

The full model's Smoke3D velocity error is $17.43\%$, compared with
$20.75\%$ for the stronger sequential hybrid, a $16.0\%$ relative reduction.
For NS2D, restricting the FFN to grouped transformations changes the error
from $2.79\%$ to $4.32\%$.
These comparisons support the tested organization and information exchange.
Removing trainable operators can also change capacity; these interventions
are not, by definition, an equal-parameter sweep.

\paragraph{Update comparison}
Table~\ref{tab:slice-attention-analysis} is a separate experiment comparing vanilla Transolver,
its no-slice-attention and SRPA variants, and LinearNO.
Within each task, the four models use the same training settings and are
evaluated at the final checkpoint of each of three independent training runs.
The table reports the arithmetic mean of these three scores.
Evaluation uses the same relative $L^2$ metric, percentage scaling, and
spatial, temporal, and field aggregation as the corresponding Darcy or
NS2D entry in Table~\ref{tab:pde-main-results}.
Thus, NS2D follows the main benchmark's forecasting evaluation rather than
introducing a separate single-step metric for this comparison.
The no-attention variant retains slice aggregation and deslicing, while
removing the separate attention operation between slice tokens.
LinearNO~\citep{hu2026linearno} is an independent linear-attention
architecture, not an equivalent implementation of this Transolver variant.
Both alternatives outperform vanilla Transolver, but SRPA achieves the
lowest errors on both tasks.
Relative to the no-attention variant, SRPA reduces the reported Darcy and
NS2D errors by $9.5\%$ and $13.5\%$, respectively.
These reductions are calculated as $(e_{\mathrm{noattn}}-e_{\mathrm{SRPA}})
/e_{\mathrm{noattn}}$ using the values in the table.
Thus, the comparisons support changing the interaction rule rather than
discarding cross-slice attention altogether.
The SRPA row modifies Transolver's slice update and does not denote
the complete two-subspace SpecTransolver model.
Its results can therefore differ from
Table~\ref{tab:pde-main-results} and Table~\ref{tab:ablation-components}.
This comparison combines the identity path and learned scale, without
isolating zero initialization from an unscaled residual update.

\paragraph{Relation to the conditional analysis}
Equation~\ref{eq:appendix-identity-interaction} and
Proposition~\ref{prop:srpa-mixing} in Appendix~\ref{app:srpa-theory}
formalize replacement and residual refinement through conditional diameter bounds.
These bounds concern slice-space updates, not the complete network.
The bounds depend on the learned scale, attention weights, and value
projection; zero initialization does not ensure small corrections after training.

\paragraph{Number of Slices $M$}
\label{app:ablation-slice-count}
The reported architecture uses a fixed slice count for each configuration.
The component comparisons establish neither slice-count sensitivity nor an optimal count.

\section{Additional Visualizations}
\label{app:supplementary-results}

The following visualizations support the main model analyses in
Section~\ref{sec:experiments}. Each analysis retains its own evaluated
samples, configurations, and measurement scope.

\subsection{Learned Slices}
\label{app:slice-attention-analysis}
Learned routing and slice interactions complement the
update comparison in Appendix~\ref{app:full-ablations}.

\paragraph{Visualization setup}
Figure~\ref{fig:efficiency-representations}(c,d) compares Transolver and
Transolver-$\sigma$ on the same held-out NS2D case.
Both models receive ten consecutive vorticity frames on a $64\times64$ grid,
and we visualize their final physics-attention block during inference.
Both use eight blocks, total hidden width $256$, eight attention heads,
and $32$ slices, following Table~\ref{tab:training-settings}.
Transolver-$\sigma$ divides its hidden channels equally between physical
and spectral subspaces and retains four Fourier modes per axis.

\paragraph{Routing and attention maps}
We extract the routing and slice-attention tensors defined in
Section~\ref{sec:method-srpa} from the final block of each trained model.
For head $r$, these tensors are
$\mathbf{P}^{(r)}\in\mathbb{R}^{4096\times32}$ and
$\mathbf{A}^{(r)}\in\mathbb{R}^{32\times32}$, respectively.
The figure displays their averages over the $n_h=8$ heads,
$\overline{\mathbf{P}}=n_h^{-1}\sum_{r=1}^{n_h}\mathbf{P}^{(r)}$ and
$\overline{\mathbf{A}}=n_h^{-1}\sum_{r=1}^{n_h}\mathbf{A}^{(r)}$.
Each routing column is reshaped to the spatial grid, with S0--S31 indexing
the slices learned independently by each model.
Attention rows correspond to query slices and columns to key slices,
so vertical bands indicate repeated access to particular keys.
Each attention matrix is scaled to its own observed minimum and maximum,
emphasizing the distribution of key access within each model.
These matrices show attention weights; the realized SRPA corrections
also include the value projection and learned residual scale, as analyzed below.

\paragraph{Observed organization and interpretation}
In the visualized case, Transolver concentrates routing and key access
on a small subset of slices.
Transolver-$\sigma$ exhibits spatially structured routing across more slices,
with key access distributed more broadly across the attention matrix.
These patterns are consistent with our design of retaining distinct slice
states while allowing cross-slice interaction to refine their representations.
Together with the error reductions in Table~\ref{tab:slice-attention-analysis},
these visualizations illustrate learned slice organization and effective residual refinement.

\Needspace{6\baselineskip}
\paragraph{Layer-wise SRPA Corrections}
\label{app:srpa-layerwise-corrections}

\paragraph{Measurement protocol}
We analyze one trained Transolver-$\sigma$ checkpoint per task using the
final eight-block configurations in Table~\ref{tab:srpa-profile-settings},
consistent with Table~\ref{tab:training-settings}.
Each block's signed residual scale $\gamma_\ell$ is read directly from the
checkpoint and is independent of the evaluated input.
Correction ratios are measured on three fixed held-out cases per task,
each processed separately with batch size one.
Single-step evaluation uses the training normalization and input window,
without updating model parameters.
Read-only hooks capture slice states and attention responses before
deslicing and output projection; enabling the hooks leaves all evaluated
predictions unchanged.

\begin{table}[!htbp]
  \centering
  \caption{\textbf{Checkpoint configurations for layer-wise SRPA analysis.}
  $L$, $C$, $n_h$, $M$, and $K$ denote block count, total hidden
  width, attention heads, slices, and Fourier modes per axis, respectively.
  IT3D denotes isotropic turbulence.}
  \label{tab:srpa-profile-settings}
  \small
  \setlength{\tabcolsep}{7pt}
  \renewcommand{\arraystretch}{1.1}
  \begin{tabular*}{\linewidth}{@{\extracolsep{\fill}}lcccccc@{}}
    \toprule
    Task & Grid & $L$ & $C$ & $n_h$ & $M$ & $K$ \\
    \midrule
    Darcy & $85^2$ & 8 & 128 & 8 & 32 & 4 \\
    NS2D & $64^2$ & 8 & 256 & 8 & 32 & 4 \\
    KF2D & $128^2$ & 8 & 128 & 8 & 32 & 8 \\
    IT3D & $60^3$ & 8 & 128 & 8 & 32 & 6 \\
    Smoke3D & $64^3$ & 8 & 128 & 8 & 32 & 6 \\
    \bottomrule
  \end{tabular*}
\end{table}

\begin{figure}[!t]
  \centering
  \includegraphics[width=\textwidth]{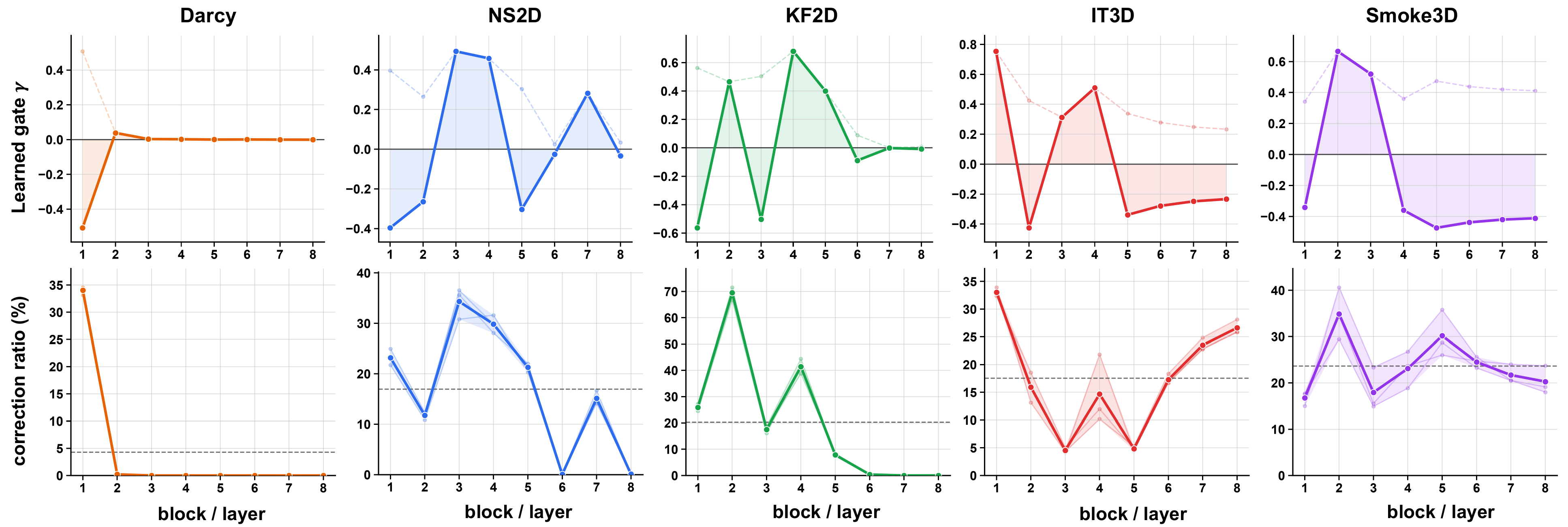}
  \caption{\textbf{Learned residual scales and realized SRPA corrections.}
  Columns show eight-layer profiles for Darcy, NS2D, KF2D, isotropic turbulence (IT3D), and Smoke3D.
  Top: solid curves show signed $\gamma_\ell$;
  pale dashed curves show $|\gamma_\ell|$, with shading between the signed
  curve and zero.
  Bottom: correction ratios in slice-token space before deslicing and
  output projection, expressed as percentages.
  Thin curves show three fixed held-out cases; thick curves show
  their means, with shading spanning their range.
  Horizontal dashed lines indicate means across all cases and layers.
  Ratios use joint norms over heads, slices, and channels before case
  averaging.
  Layers are one-based; vertical scales vary across tasks.
  Neither shaded region is a confidence interval.}
  \label{fig:srpa-layerwise-corrections}
\end{figure}

\paragraph{Joint-head correction ratio}
We retain the zero-based block index $\ell$ from
Section~\ref{sec:method-decomposition}; the figure labels this block as
Layer $\ell+1$.
For case $c$ and head $r$, let
$\mathbf{Z}_{\ell,c}^{(r)}\in\mathbb{R}^{M\times d_h}$ denote the incoming
slice states, with $d_h=C/(2n_h)$.
The unscaled attention response from Equation~\ref{eq:srpa-update} is
$\Delta\mathbf{Z}_{\ell,c}^{(r)}$;
the actual residual correction is $\gamma_\ell\Delta\mathbf{Z}_{\ell,c}^{(r)}$.
For nonzero input-state norm, the measured ratio is
\begin{subequations}
  \label{eq:srpa-correction-measurement}
  \begin{alignat}{1}
    &\Delta\mathbf{Z}_{\ell,c}^{(r)}
    =\mathbf{A}_{\ell,c}^{(r)}\mathbf{Z}_{\ell,c}^{(r)}\mathbf{W}_{v,\ell},
    \label{eq:srpa-ungated-response}\\
    &\rho_{\ell,c}=|\gamma_\ell|
    \frac{\left(\displaystyle\sum_{r=1}^{n_h}\sum_{m=1}^{M}\sum_{j=1}^{d_h}
    [\Delta\mathbf{Z}_{\ell,c}^{(r)}]_{mj}^{2}\right)^{1/2}}
    {\left(\displaystyle\sum_{r=1}^{n_h}\sum_{m=1}^{M}\sum_{j=1}^{d_h}
    [\mathbf{Z}_{\ell,c}^{(r)}]_{mj}^{2}\right)^{1/2}}.
    \label{eq:srpa-joint-ratio}
  \end{alignat}
\end{subequations}
This is a single vectorized $L_2$ norm ratio over all heads, slices, and
channels, not an arithmetic mean of per-head ratios.
Unlike the routing and attention maps in
Figure~\ref{fig:efficiency-representations}(c,d), the tensors are not
averaged over heads before measurement.
We then average per-case ratios over cases and layers:
\begin{equation}
  \begin{alignedat}{1}
    &\overline{\rho}_{\ell}=\frac{1}{3}\sum_{c=1}^{3}\rho_{\ell,c},\\
    &\overline{\rho}=\frac{1}{L}
    \sum_{\ell=0}^{L-1}\overline{\rho}_{\ell}.
  \end{alignedat}
  \label{eq:srpa-profile-means}
\end{equation}
Figure~\ref{fig:srpa-layerwise-corrections} reports percentage ratios;
ranges reflect variation across inputs rather than training seeds.

\paragraph{Observed correction profiles}
Darcy's mean correction ratio falls from $34.03\%$ in Layer $1$ to $0.20\%$
in Layer $2$, with all later layers at approximately $0.004\%$ or less.
By contrast, Smoke3D maintains $16.75\%$--$34.86\%$ across all eight layers.
NS2D and KF2D exhibit selective multi-layer corrections with near-zero
responses in some later layers, while IT3D retains measurable corrections
throughout all eight layers, with stronger corrections in its final two layers.
Scale magnitude alone does not determine correction strength.
In KF2D, Layer $4$ has larger $|\gamma|$ than Layer $2$ ($0.679$ vs.\ $0.465$),
but a smaller correction ratio ($41.44\%$ vs.\ $69.51\%$).
Together, the profiles reveal how the trained model distributes residual
interactions across tasks and layers.
Darcy concentrates slice refinement near the network input, whereas Smoke3D
sustains substantial corrections throughout all eight layers.
The correction ratio captures the joint contribution of the learned scale
and attention response, while the sign of $\gamma_\ell$ determines the correction's orientation.
These measurements complement Table~\ref{tab:slice-attention-analysis} by
locating and quantifying residual slice interactions.

\makeatletter
\setlength{\@fptop}{0pt}
\setlength{\@fpsep}{8pt}
\setlength{\@fpbot}{0pt plus 1fil}
\makeatother

\subsection{Showcases}
\label{app:rollout-case-studies}

These cases extend the main-text comparison in
Figure~\ref{fig:case-study-baselines} with successive rollout leads for
physical-only, spectral-only, and joint models.
They complement the aggregate curves in Figure~\ref{fig:rollout-comparison};
each visualization concerns one case, not a dataset-wide statistic.

\paragraph{Error-map conventions}
In Figure~\ref{fig:case-study-baselines}, each task uses the same test
sample, prediction time, and physical variable or channel set across models.
Scalar error maps show absolute prediction errors; maps combining multiple
channels use the channel-wise $L^2$ norm at each spatial location.
Spatial views, colormaps, and error limits are shared across models
within each task.
Reference fields use their own value scales, while brighter error-map
colors indicate larger deviations from the reference.

\paragraph{Autoregressive cases}
Figures~\ref{fig:case-study-ns2d}--\ref{fig:case-study-smoke3d} show one
case each for NS2D, KF2D, isotropic turbulence (IT3D), and Smoke3D.
Rows correspond to Physolver, Specsolver, and Transolver-$\sigma$ from top
to bottom; columns indicate successive physical prediction leads.
Predictions are fed back without future observations, following the
rollout protocol in Appendix~\ref{app:rollout-comparison}.
The displayed fields are vorticity $\omega$ in NS2D and KF2D,
$x$-velocity $u_x$ in IT3D, and smoke density $d$ in Smoke3D.
Each panel shows the pointwise absolute error in the displayed field.
Within each dataset, all models and leads share one error scale,
with its upper limit set by the $99.5$th error percentile.

\Needspace{6\baselineskip}
\paragraph{Three-dimensional visualization}
For IT3D and Smoke3D, we apply light three-dimensional Gaussian smoothing
to the error volume, with a standard deviation of $0.75$ grid cells.
The smoothed volume is projected along each spatial direction onto
three outer cube faces, with selected high-error interior voxels retained.
Smoothing and color saturation affect only the visualization;
the quantitative metrics are computed from unsmoothed arrays.

\begin{figure}[H]
  \centering
  \includegraphics[width=\textwidth]{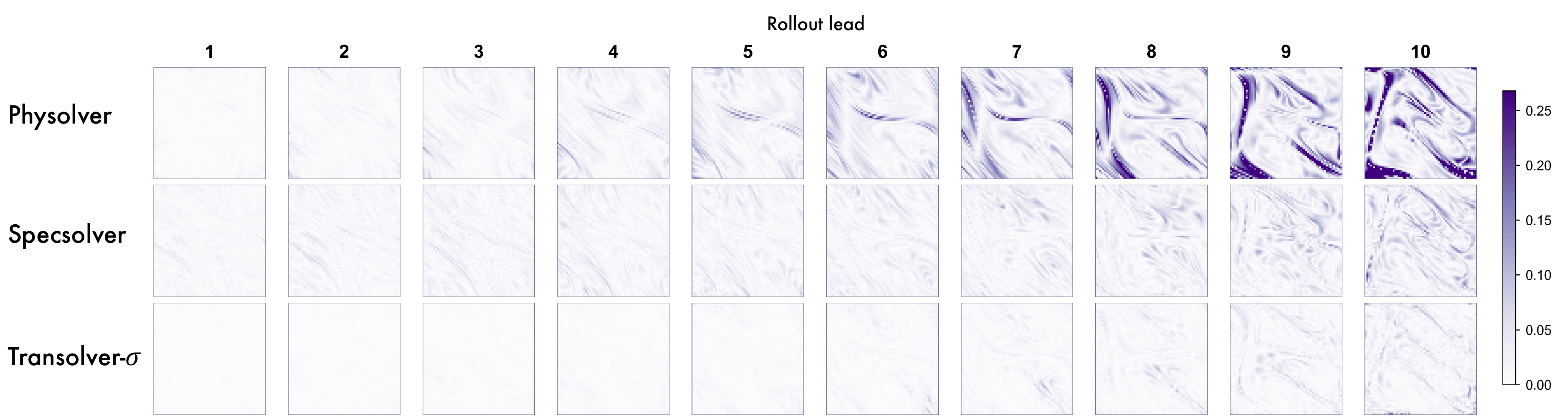}
  \caption{\textbf{Autoregressive error maps on NS2D.}
  Columns show leads $1$--$10$ for the same test case.
  Rows show Physolver, Specsolver, and Transolver-$\sigma$ from top to bottom,
  using a shared error scale.
  Physolver develops pronounced error bands at later leads;
  Transolver-$\sigma$ retains weaker spatial errors across the displayed sequence.}
  \label{fig:case-study-ns2d}
\end{figure}

\begin{figure}[H]
  \centering
  \includegraphics[width=\textwidth]{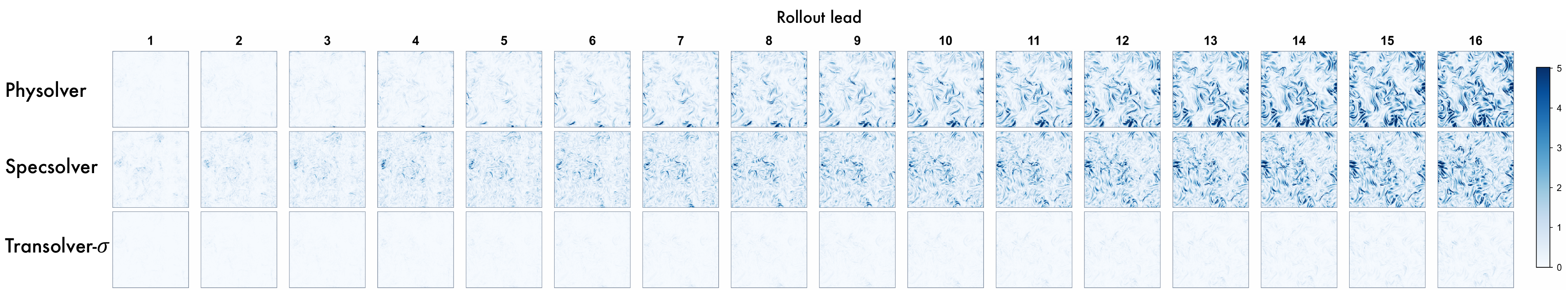}
  \caption{\textbf{Autoregressive error maps on KF2D.}
  Columns show leads $1$--$16$ for the same test case.
  Rows show Physolver, Specsolver, and Transolver-$\sigma$ from top to bottom,
  using a shared error scale.
  Errors are more prominent in both single-domain models
  than in Transolver-$\sigma$.}
  \label{fig:case-study-kf2d}
\end{figure}

\begin{figure}[H]
  \centering
  \includegraphics[width=\textwidth]{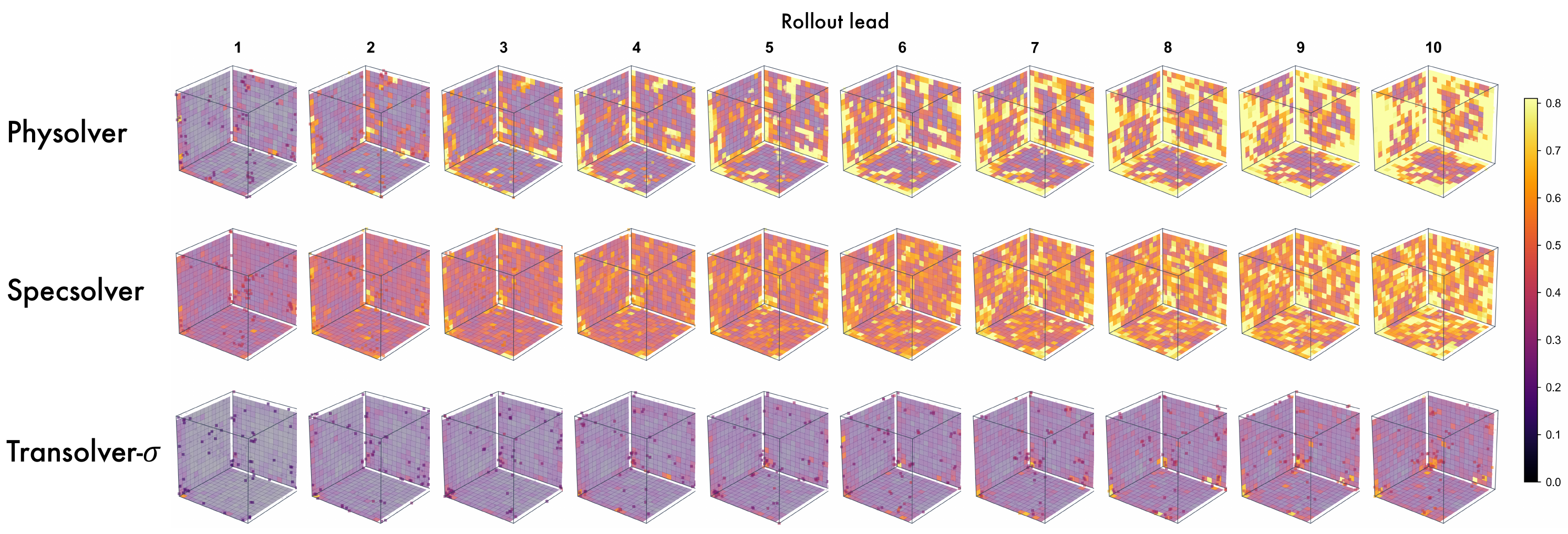}
  \caption{\textbf{Autoregressive error maps on isotropic turbulence (IT3D).}
  Columns show leads $1$--$10$ for the same test case.
  Rows show Physolver, Specsolver, and Transolver-$\sigma$ from top to bottom,
  using a shared error scale.
  The single-domain models develop more extensive high-error regions
  at later leads than Transolver-$\sigma$.}
  \label{fig:case-study-it3d}
\end{figure}

\begin{figure}[H]
  \centering
  \includegraphics[width=\textwidth]{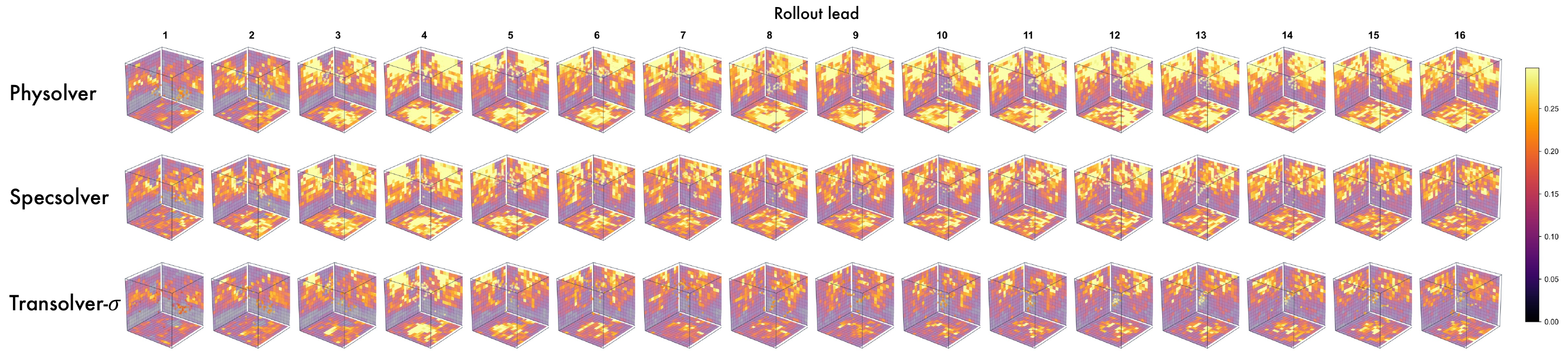}
  \caption{\textbf{Autoregressive error maps on Smoke3D.}
  Columns show leads $1$--$16$ for the same test case.
  Rows (top to bottom): Physolver, Specsolver, and Transolver-$\sigma$,
  all sharing one error scale.
  Transolver-$\sigma$ has weaker localized errors at later leads;
  spatial patterns evolve non-monotonically.}
  \label{fig:case-study-smoke3d}
\end{figure}

\section{Additional Experiments}
\label{app:additional-experiments}

\paragraph{Additional REALM Results}
\label{app:realm-training}

Table~\ref{tab:realm-results} reports the complete results on
IgnitHIT and EvolveJet, including correlation and training,
validation, and test errors.
Metric scales and baseline sources are specified in
Appendices~\ref{app:experiment-metrics} and~\ref{app:experiment-provenance};
Table~\ref{tab:training-settings} summarizes the matched experimental configurations.

\begin{table*}[!htbp]
  \centering
  \caption{\textbf{Results on REALM.}
  Correlation ($\times100$, higher is better) and training, validation,
  and test errors (lower is better) on IgnitHIT and EvolveJet.
  Errors follow the original benchmark scale, and baseline results are
  taken from the official comparison.
  Our results are averaged over three training runs evaluated at their
  final checkpoints; standard deviations are provided in
  Appendix~\ref{app:standard-deviations}.
  Best and second-best results are bold and underlined, respectively.}
  \label{tab:realm-results}
  \small
  \setlength{\tabcolsep}{5pt}
  \renewcommand{\arraystretch}{1.08}
  \begin{tabular}{l|cccc|cccc}
    \toprule
    \multirow{2}{*}{\textsc{Models}} & \multicolumn{4}{c}{\textsc{IgnitHIT}} & \multicolumn{4}{c}{\textsc{EvolveJet}} \\
    \cmidrule(lr){2-5}\cmidrule(lr){6-9}
    & Corr & train & val & test & Corr & train & val & test \\
    \midrule
    CNext & 96.09 & 2.47 & 3.17 & 2.71 & 92.56 & 1.06 & 4.25  & 2.59  \\
    CROP & 94.91 & \textbf{0.20} & 5.24 & 4.37 & 84.95 & 1.66 & 3.70  & 3.41  \\
    DPOT & 94.19 & 1.47 & 6.15 & 5.90 & 87.71 & 2.10 & 5.32  & 3.36  \\
    DeepONet & 68.60 & 39.63 & 48.54 & 48.25 & 62.20 & 7.85 & 28.79  & 18.33  \\
    F-FNO & \underline{97.36} & 0.52 & \underline{1.86} & \underline{1.87} & \underline{95.08} & \underline{0.54} & 2.50  & \underline{0.98}  \\
    FNO & 92.49 & 5.50 & 8.45 & 6.91 & 90.40 & 0.96 & 7.21  & 2.97  \\
    FactFormer & 95.70 & 0.62 & 2.74 & 2.91 & 92.76 & 0.57 & \underline{2.15}  & 1.54  \\
    GNOT & 92.71 & 8.18 & 12.94 & 13.38 & 85.81 & 5.03 & 12.57  & 7.47  \\
    ONO & 78.76 & 38.21 & 46.32 & 39.60 & 37.45 & 75.29 & 88.59  & 70.96  \\
    Transolver & 86.85 & 14.86 & 23.25 & 13.93 & 50.62 & 71.50 & 77.78  & 56.10  \\
    U-NO & 93.94 & 1.03 & 4.88 & 5.59 & 86.09 & 4.16 & 5.31  & 4.09  \\
    \midrule
    \textbf{Transolver-$\sigma$} & \textbf{98.49} & \underline{0.47} & \textbf{1.50} & \textbf{1.49} & \textbf{95.93} & \textbf{0.03} & \textbf{1.31} & \textbf{0.63}  \\
    \bottomrule
  \end{tabular}
\end{table*}

\subsection{Model Scalability}
\label{app:model-scalability}
We study scalability on Darcy flow along three axes: training-set size,
spatial resolution, and model capacity controlled by hidden width.
Figure~\ref{fig:darcy-scaling} shows decreasing prediction errors along
all three axes, with progressively smaller gains toward the largest settings.

\begin{figure}[!htbp]
  \centering
  \includegraphics[width=\linewidth]{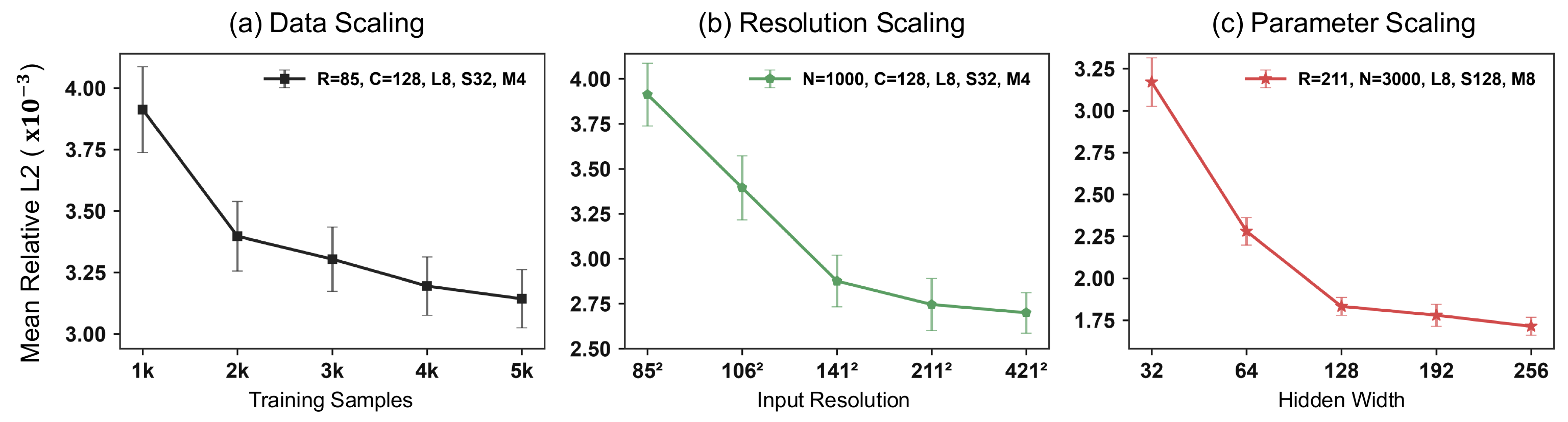}
  \caption{\textbf{Scalability on Darcy flow.}
  Relative $L^2$ error as a function of \textbf{(a)} training-set size,
  \textbf{(b)} grid resolution, and \textbf{(c)} hidden width.
  Each point is the mean over the same $200$ test cases;
  error bars show $1.96s/\sqrt{200}$, where $s$ is the standard deviation
  of the per-case errors.
  The legends use $S$ for slices and $M$ for Fourier modes,
  corresponding to $M$ and $K$ in the manuscript.}
  \label{fig:darcy-scaling}
\end{figure}

\paragraph{Experimental settings}
All configurations use eight blocks, eight attention heads, MLP ratio two,
and no dropout, with a global batch size of four.
We train with AdamW and a OneCycle schedule, using peak learning rate
$10^{-3}$, weight decay $10^{-5}$, and warmup fraction $0.3$.
The objective combines relative $L^2$ loss with derivative loss weighted by
$0.1$, and evaluation uses the final checkpoint.
Table~\ref{tab:darcy-data-scaling} summarizes the three sweeps, each with
its own fixed training settings.

\paragraph{Data scaling}
We increase the training set from $1{,}000$ to $5{,}000$ cases while fixing
resolution $85^2$, width $128$, $32$ slices, and four Fourier modes.
Each configuration uses $200$ epochs and $1{,}245{,}289$ trainable parameters.
The initial training set is extended using the same Darcy generation
process, while the $200$ test cases remain fixed.
Mean relative error decreases from $0.003912$ to $0.003143$, a $19.7\%$
reduction, with the largest gain between $1{,}000$ and $2{,}000$ training cases.

\paragraph{Resolution scaling}
We train at resolutions $85^2$, $106^2$, $141^2$, $211^2$, and $421^2$,
obtained from the original $421^2$ fields with strides $5$, $4$, $3$, $2$, and $1$.
All resolutions use $1{,}000$ training cases, width $128$, $32$ slices,
four Fourier modes, and $200$ epochs.
The parameter count remains fixed at $1{,}245{,}289$, and each model is
evaluated on its corresponding grid.
Increasing resolution reduces mean error from $0.003912$ to $0.002699$,
a $31.0\%$ improvement, showing the benefit of finer spatial discretization.

\paragraph{Parameter scaling}
We increase hidden width from $32$ to $256$, using $3{,}000$ training cases,
resolution $211^2$, $128$ slices, eight Fourier modes, and $100$ epochs.
The parameter count increases from $125{,}161$ to $6{,}964{,}889$, while
mean error decreases from $0.003170$ to $0.001713$, a $46.0\%$ reduction.
The strongest gains occur up to width $128$, followed by smaller improvements
at widths $192$ and $256$.
Together, these results show that the model benefits from additional data,
finer grids, and increased capacity within the evaluated Darcy settings.

\begin{table}[!htbp]
 \centering
 \caption{\textbf{Darcy scaling results.}
 Relative $L^2$ errors ($\times10^{-3}$) are mean $\pm$
 standard deviation across $200$ test cases.
 Data and resolution sweeps use $200$ epochs; width scaling uses $100$.}
 \label{tab:darcy-data-scaling}
 \label{tab:darcy-resolution-scaling}
 \label{tab:darcy-width-scaling}
 \small
 \setlength{\tabcolsep}{4pt}
 \renewcommand{\arraystretch}{1.15}
 \begin{tabular*}{\linewidth}{@{\extracolsep{\fill}}lccccc@{}}
 \toprule
 \multicolumn{6}{@{}l}{\textit{Data scaling: $85^2$ grid, $C=128$, $M=32$, $K=4$}}\\
 Training cases & 1,000 & 2,000 & 3,000 & 4,000 & 5,000 \\
 Error & $3.912\pm1.259$ & $3.397\pm1.022$ & $3.304\pm0.943$ & $3.195\pm0.859$ & $3.143\pm0.853$ \\
 \midrule
 \multicolumn{6}{@{}l}{\textit{Resolution scaling: 1,000 training cases, $C=128$, $M=32$, $K=4$}}\\
 Grid & $85^2$ & $106^2$ & $141^2$ & $211^2$ & $421^2$ \\
 Error & $3.912\pm1.259$ & $3.394\pm1.281$ & $2.876\pm1.042$ & $2.745\pm1.043$ & $2.699\pm0.812$ \\
 \midrule
 \multicolumn{6}{@{}l}{\textit{Parameter scaling: 3,000 training cases, $211^2$ grid, $M=128$, $K=8$}}\\
 Width $C$ & 32 & 64 & 128 & 192 & 256 \\
 Parameters & 125,161 & 462,937 & 1,776,489 & 3,946,377 & 6,964,889 \\
 Error & $3.170\pm1.041$ & $2.280\pm0.594$ & $1.832\pm0.388$ & $1.779\pm0.473$ & $1.713\pm0.383$ \\
 \bottomrule
 \end{tabular*}
\end{table}

\subsection{Standard Deviations}
\label{app:standard-deviations}
We report means and sample standard deviations over three independent
training runs, using seeds $42$, $43$, and $44$ and each run's final checkpoint.
Tables~\ref{tab:canonical-seed-results}--\ref{tab:realm-standard-deviations}
cover the canonical benchmarks, RealPDEBench, and REALM, respectively.
Following the main comparisons, we include the strongest baseline for each
metric to place the variation across runs in context.
Transolver-$\sigma$ reduces mean errors on canonical and
RealPDEBench tasks, and improves correlations and test errors on REALM.

\begin{table}[H]
 \centering
 \caption{\textbf{Standard deviations on the canonical benchmarks.}
 Relative $L^2$ errors (\%) are reported as mean $\pm$ standard deviation
 over three training runs. Strongest baseline values are taken from
 Table~\ref{tab:pde-main-results}: DRIFT-Net for Darcy, FactFormer for NS2D and Smoke3D, and EddyFormer for KF2D and IT3D.}
 \label{tab:canonical-seed-results}
 \small
 \setlength{\tabcolsep}{1.8pt}
 \renewcommand{\arraystretch}{1.15}
 \begin{tabular*}{\linewidth}{@{\extracolsep{\fill}}l|cccccccc@{}}
 \toprule
 \multirow{2}{*}{Model} & Darcy & NS2D & \multicolumn{2}{c}{KF2D} & \multicolumn{2}{c}{IT3D} & \multicolumn{2}{c}{Smoke3D}\\
 \cmidrule(lr){2-2}\cmidrule(lr){3-3}\cmidrule(lr){4-5}\cmidrule(lr){6-7}\cmidrule(lr){8-9}
 & $p$ & $\omega$ & Avg. & Final & $\mathbf{u}$ & $p$ & $\mathbf{u}$ & $d$\\
 \midrule
 Strongest baseline & 0.45 & 4.23 & 12.69 & 21.99 & 12.70 & 19.43 & 23.37 & 9.19\\
 Transolver-$\sigma$ & $0.40_{\pm 0.01}$ & $2.79_{\pm 0.05}$ & $2.52_{\pm 0.05}$ & $4.14_{\pm 0.07}$ & $10.80_{\pm 0.16}$ & $15.65_{\pm 0.23}$ & $17.43_{\pm 0.25}$ & $7.07_{\pm 0.09}$\\
 \bottomrule
 \end{tabular*}
\end{table}

\begin{table}[H]
 \centering
 \caption{\textbf{Standard deviations on RealPDEBench.}
 Our entries show mean $\pm$ standard deviation across three training runs;
 all metrics retain the $100\times$ display scale of the main comparison.
 Strongest baseline values are taken from Table~\ref{tab:realpdebench-results}.}
 \label{tab:external-seed-results}
 \small
 \setlength{\tabcolsep}{4pt}
 \renewcommand{\arraystretch}{1.15}
 \begin{tabular*}{\linewidth}{@{\extracolsep{\fill}}ll|cccc@{}}
 \toprule
 Metric & Model & Controlled Cylinder & FSI & Foil & Combustion\\
 \midrule
 \multirow{2}{*}{RMSE} & Strongest baseline & 0.80 & 0.85 & 1.00 & 2.08\\
 & Transolver-$\sigma$ & $0.650_{\pm 0.010}$ & $0.743_{\pm 0.015}$ & $0.470_{\pm 0.010}$ & $1.860_{\pm 0.040}$\\
 \midrule
 \multirow{2}{*}{Relative $L^2$} & Strongest baseline & 5.55 & 5.83 & 1.59 & 53.31\\
 & Transolver-$\sigma$ & $5.200_{\pm 0.080}$ & $5.550_{\pm 0.080}$ & $1.403_{\pm 0.025}$ & $52.857_{\pm 0.550}$\\
 \midrule
 \multirow{2}{*}{fRMSE} & Strongest baseline & 0.09 & 0.07 & 0.10 & 0.24\\
 & Transolver-$\sigma$ & $0.060_{\pm 0.002}$ & $0.060_{\pm 0.002}$ & $0.080_{\pm 0.002}$ & $0.180_{\pm 0.004}$\\
 \bottomrule
 \end{tabular*}
\end{table}

\begin{table}[H]
 \centering
 \caption{\textbf{Standard deviations on REALM.}
 Our entries show mean $\pm$ standard deviation over three training runs.
 Correlations are multiplied by $100$; errors retain the benchmark scale.
 Strongest baseline values are taken from Table~\ref{tab:realm-results}.}
 \label{tab:realm-standard-deviations}
 \small
 \setlength{\tabcolsep}{3pt}
 \renewcommand{\arraystretch}{1.15}
 \begin{tabular*}{\linewidth}{@{\extracolsep{\fill}}ll|cccc@{}}
 \toprule
 Dataset & Model & Correlation & Train error & Val. error & Test error\\
 \midrule
 \multirow{2}{*}{IgnitHIT} & Strongest baseline & 97.36 & 0.20 & 1.86 & 1.87\\
 & Transolver-$\sigma$ & $98.493_{\pm 0.267}$ & $0.470_{\pm 0.010}$ & $1.503_{\pm 0.035}$ & $1.493_{\pm 0.035}$\\
 \midrule
 \multirow{2}{*}{EvolveJet} & Strongest baseline & 95.08 & 0.54 & 2.15 & 0.98\\
 & Transolver-$\sigma$ & $95.930_{\pm 0.255}$ & $0.030_{\pm 0.001}$ & $1.310_{\pm 0.030}$ & $0.630_{\pm 0.020}$\\
 \bottomrule
 \end{tabular*}
\end{table}

\subsection{Exact Rollout Error Decomposition}
\label{app:exact-error-decomposition}
The local expansion in Appendix~\ref{app:error-propagation} separates new
prediction error from the response to an imperfect history.
We measure these contributions along free rollouts using an exact decomposition
of the prediction error into injection, propagation, and their interaction.

\begin{figure}[!b]
 \centering
 \includegraphics[width=0.98\linewidth]{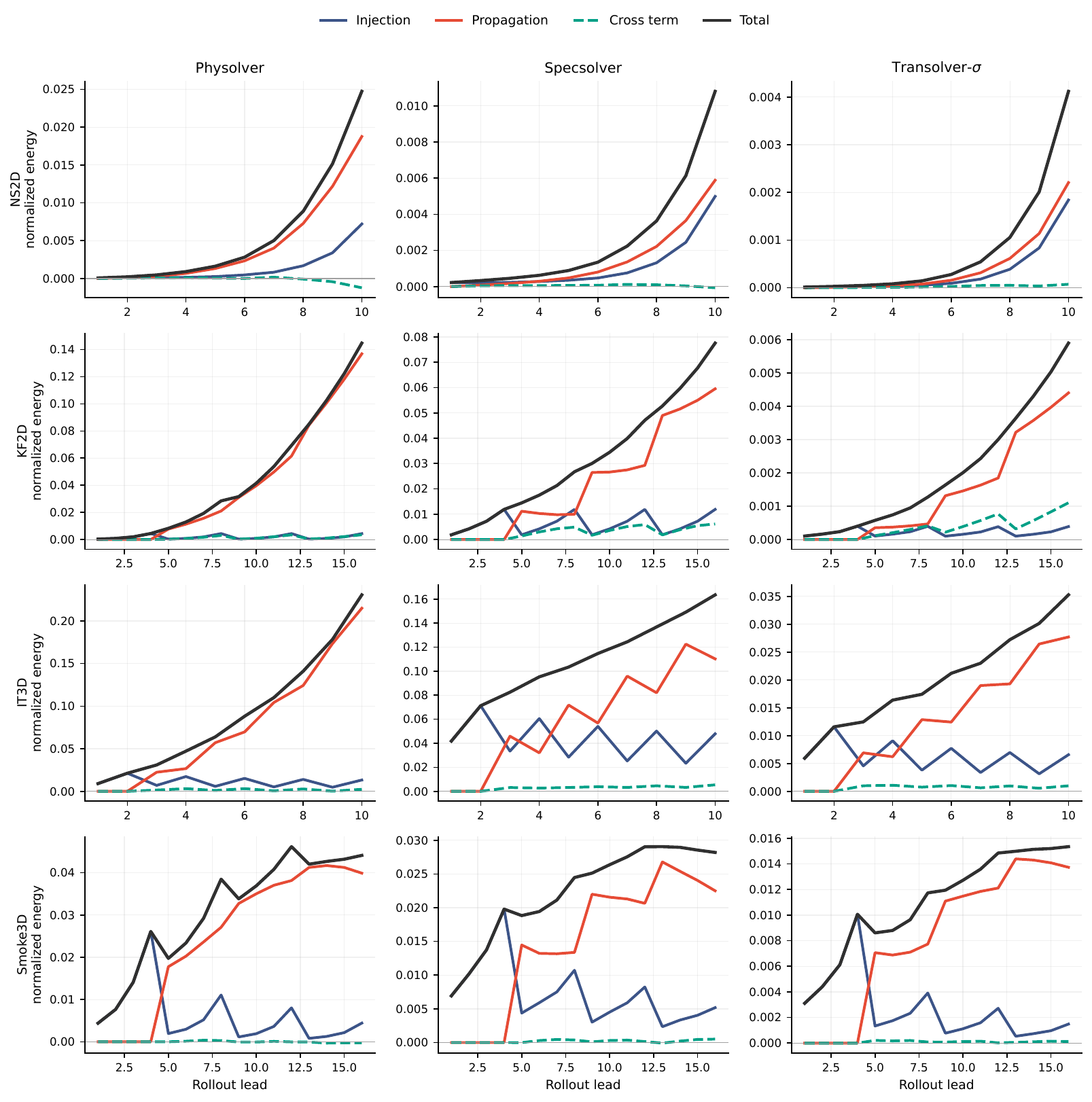}
 \caption{\textbf{Exact error decomposition during free rollout.}
 Rows correspond to NS2D, KF2D, IT3D, and Smoke3D;
 columns compare Physolver, Specsolver, and Transolver-$\sigma$.
 Injection, propagation, and their signed cross term sum to the total
 normalized error energy in Equation~\ref{eq:exact-error-energy}.
 Panel-specific vertical scales must be read when comparing magnitudes.
 The KF2D spectral-only model retains eight modes per axis.
 Leads count predicted frames; the associated block lengths are $1,4,2,4$.
 The decomposition tracks how local prediction errors and feedback-induced
 deviations jointly contribute to rollout error.}
 \label{fig:exact-error-decomposition}
\end{figure}

\paragraph{Evaluation setup}
We compare physical-only, spectral-only, and joint models on the same
cases, observed histories, output-block lengths, and rollout horizons within
each task. The joint model uses the parallel spectral--physical architecture
in Section~\ref{sec:method-decomposition}, with the canonical configurations
in Table~\ref{tab:training-settings}.
Table~\ref{tab:decomposition-cohort} lists evaluation samples,
output-block lengths, and prediction horizons.

\begin{table}[!htbp]
 \centering
 \caption{\textbf{Evaluation settings for the exact error decomposition.}
 $n$ counts trajectories (windows for KF2D);
 $T_{\mathrm{out}}$ denotes output frames per call and $H$ the total horizon.}
 \label{tab:decomposition-cohort}
 \small
 \setlength{\tabcolsep}{12pt}
 \begin{tabular}{@{}lrrr@{}}
 \toprule
 Task & $n$ & $T_{\mathrm{out}}$ & $H$ \\
 \midrule
 NS2D & 200 & 1 & 10 \\
 KF2D & 400 & 4 & 16 \\
 IT3D & 100 & 2 & 10 \\
 Smoke3D & 200 & 4 & 16 \\
 \bottomrule
 \end{tabular}
\end{table}

\paragraph{Exact decomposition and aggregation}
For a fixed model $f$, let $\mathbf{x}_k$ and $\widehat{\mathbf{x}}_k$
denote the true and freely rolled-out histories at block boundary $k$.
All outputs below are expressed in the physical units used for evaluation.
For the reference output block $\mathbf{y}_k$, define
\begin{equation}
\begin{aligned}
 \boldsymbol\delta_k &= f(\mathbf{x}_k)-\mathbf{y}_k,\\
 \boldsymbol\pi_k &= f(\widehat{\mathbf{x}}_k)-f(\mathbf{x}_k),\\
 \mathbf{e}_k &= f(\widehat{\mathbf{x}}_k)-\mathbf{y}_k
 =\boldsymbol\delta_k+\boldsymbol\pi_k.
\end{aligned}
\label{eq:exact-error-vectors}
\end{equation}
True-history predictions measure error injection; their difference from
free-history predictions measures accumulated history effects.
At each lead, we extract the corresponding frame from these blocks.
For a nonzero target frame $\mathbf{y}$, normalized error energy satisfies
\begin{equation}
 \frac{\|\mathbf{e}\|_2^2}{\|\mathbf{y}\|_2^2}
 = \underbrace{\frac{\|\boldsymbol\delta\|_2^2}{\|\mathbf{y}\|_2^2}}_{\text{injection}}
 + \underbrace{\frac{\|\boldsymbol\pi\|_2^2}{\|\mathbf{y}\|_2^2}}_{\text{propagation}}
 + \underbrace{\frac{2\langle\boldsymbol\delta,\boldsymbol\pi\rangle}
 {\|\mathbf{y}\|_2^2}}_{\text{cross term}}.
 \label{eq:exact-error-energy}
\end{equation}
We compute each normalized term per case and then average over cases at each lead.
The cross term is signed and records alignment or cancellation between the two error components.
The total is the mean squared relative error, providing an additive energy
view of the rollout errors in Figure~\ref{fig:rollout-comparison}.
For multi-frame outputs, the true history is supplied at each block boundary,
and all frames in that block share this conditioning.
Norms jointly cover spatial positions and all output fields in each evaluated frame.
Predictions and targets are denormalized before measurement.
For Smoke3D, we apply the prescribed Dirichlet boundary
operator to predictions before forming both diagnostic components.

\paragraph{Error accumulation during rollout}
Figure~\ref{fig:exact-error-decomposition} shows that propagation energy
exceeds injection energy at the final lead in all twelve panels.
This highlights the contribution of accumulated history error to finite-horizon prediction.
For NS2D, the physical-only model's final injection and propagation energies
are approximately $0.00720$ and $0.01876$, compared with $0.00184$ and
$0.00221$ for the joint model.
Thus, the joint model reduces both fresh prediction error and the
feedback-induced component along its rollout trajectory in this comparison.
The signed cross term explains how these components combine, including
partial cancellation in the physical-only NS2D result.
For multi-frame outputs, the within-block variation reflects changing
prediction lead under a shared input history.
Together, these measurements complement the prediction-error curves by
separating fresh prediction error, feedback-induced deviations, and their interaction.

\section{Full Efficiency Analysis}
\label{app:efficiency-measurements}
We complement Figure~\ref{fig:efficiency-representations}(a,b) with the
configurations and measurement protocol used to compare training costs
on NS2D and Kolmogorov flow.
We compare HPM, DRIFT-Net, FactFormer, Transolver, and
Transolver-$\sigma$, covering hybrid and attention-based PDE solvers.

\paragraph{Model configurations}
Table~\ref{tab:efficiency-settings} lists the dedicated efficiency configurations.
For HPM, FactFormer, Transolver, and Transolver-$\sigma$, we match backbone
depth, hidden width, and head count within each task.
DRIFT-Net uses its official T hierarchy, and each model retains its native
operator structure.
For Transolver-$\sigma$, both the accuracy and efficiency experiments
on Kolmogorov flow use eight blocks and eight Fourier modes per axis.

\begin{table}[!htbp]
  \centering
  \caption{\textbf{Model configurations for training-efficiency measurements.}
  $L$, $C$, and $h$ denote backbone depth, width, and head count;
  $M$ and $K$ denote slice count and retained Fourier modes per axis.
  Model-specific settings preserve the respective native operator structures.}
  \label{tab:efficiency-settings}
  \small
  \setlength{\tabcolsep}{5pt}
  \renewcommand{\arraystretch}{1.15}
  \begin{tabular*}{\linewidth}{@{\extracolsep{\fill}}lccc@{}}
    \toprule
    Model & NS2D $(L,C,h)$ & Kolmogorov $(L,C,h)$ & Operator settings \\
    \midrule
    HPM & $(8,256,8)$ & $(8,128,8)$ & 128 fixed frequencies \\
    DRIFT-Net & \multicolumn{2}{c}{Official T hierarchy} & Patch 4, window 16 \\
    FactFormer & $(8,256,8)$ & $(8,128,8)$ & Head dimension 64 \\
    Transolver & $(8,256,8)$ & $(8,128,8)$ & $M=32$ \\
    Transolver-$\sigma$
    & $(8,256,8)$
    & $(8,128,8)$
    & \begin{tabular}[c]{@{}c@{}}
        $M=32$ \\
        $K=4$ (NS2D), $K=8$ (KF2D)
      \end{tabular} \\
    \bottomrule
  \end{tabular*}
\end{table}

HPM and Transolver use their structured-mesh implementations with MLP ratio
two, reference size eight, and unified positional encoding.
FactFormer uses rotary positions, kernel and latent multipliers of two,
and head dimension $64$.
DRIFT-Net uses stage depths $(4,4,4,4)$, widths $(48,96,192,384)$,
head counts $(3,6,12,24)$, and skip settings $(2,2,2,0)$, with MLP ratio four.
Its ConvNeXt residual and spectral branch are retained.
Transolver-$\sigma$ uses shared CPE, equal-width physical and spectral
subspaces, and full-channel SwiGLU recomposition with MLP ratio two.
Its NS2D implementation uses fused scaled dot-product attention for the
slice-residual update.

\paragraph{Training workloads}
NS2D uses $1{,}000$ unnormalized vorticity trajectories on a $64^2$ grid,
with ten history frames and ten supervised frames.
Each update accumulates gradients from ten single-frame predictions,
advances the history with the reference frames, and performs one AdamW step.
The loss sums per-sample, per-frame relative $L^2$ errors.
With batch size two, an epoch contains $500$ optimizer updates and
$5{,}000$ batched forward passes with their associated backward passes.

Kolmogorov uses $100$ training trajectories, spatially downsampled to
$128^2$ and temporally subsampled by a factor of two.
Each trajectory contains $160$ sampled frames, normalized using training
statistics, and provides $134$ windows with ten history frames and sixteen available forecast frames.
The timed workload supervises four frames per window, giving $13{,}400$
windows and $6{,}700$ updates per epoch at batch size two.
Each update evaluates relative $L^2$ loss over the four-frame outputs and
performs one backward pass and one AdamW step.
FactFormer retains four internal latent-propagation steps.
Learning rates are $10^{-3}$ for NS2D and $5\times10^{-4}$ for Kolmogorov,
with weight decays $10^{-5}$ and $10^{-4}$, respectively.

\paragraph{Time and memory measurements}
All models are measured sequentially on the same GPU in FP32, with
TF32 and cuDNN benchmarking disabled.
We use Python 3.12.3, CUDA 12.8, and PyTorch 2.7, with identical data-loading
settings across models within each task.
After twenty warmup updates, we time complete training epochs with CUDA
synchronization at the start and end.
Timing includes data loading, device transfer, gradient clearing, forward
and backward computation, and optimizer updates.
Validation, visualization, checkpoint saving, and logging are outside the timed region.
We report mean epoch time and sample standard deviation over three
independent timing processes, varying model order across rounds.

Peak training memory is the maximum allocated GPU tensor memory, including
parameters, gradients, optimizer states, and activations, in decimal GB.
Parameter storage is the total byte size of parameter tensors, including
frozen and complex-valued parameters, in decimal MB.
Ordinary buffers and optimizer states contribute to training memory but
are excluded from parameter storage.

\paragraph{Efficiency comparison}
Transolver-$\sigma$ has the lowest parameter storage among the compared
models on both tasks, together with competitive epoch times.
On NS2D, it uses $19.44$ MB compared with FactFormer's $66.92$ MB,
a reduction of approximately $71\%$.
Their mean epoch times are $200.132$ and $210.461$ seconds, respectively,
corresponding to approximately $1.05\times$ training throughput under the matched workload.
Parameter storage and peak training memory capture different costs:
HPM's Kolmogorov configuration includes $33.55$ MB of frozen frequency
bases within its $39.92$ MB parameter storage.
Transolver-$\sigma$ also uses a $4.19$ MB positional buffer, which is
included in peak training memory.
These measurements characterize the storage and computational costs of
the operator configurations under each task's training workload.

\endgroup 
\end{document}